\documentclass[Afour,sageh,times]{sagej}

\usepackage{moreverb,url}
\usepackage[colorlinks,bookmarksopen,bookmarksnumbered,citecolor=red,urlcolor=red]{hyperref}

\usepackage{irom}

\usepackage{amsthm}
\usepackage{amsmath}
\usepackage{mathrsfs}
\usepackage{algorithmic}
\usepackage{algorithm}
\usepackage{booktabs}
\usepackage{breakcites}
\usepackage{thm-restate}
\usepackage{balance}
\usepackage{bbm}
\usepackage{verbatim}
\usepackage{comment}
\usepackage{url}
\usepackage{natbib}
\usepackage{graphicx, float, url, wrapfig, stfloats}
\usepackage[dvipsnames]{xcolor}
\usepackage{cleveref}
\usepackage[title]{appendix}
\setcitestyle{aysep={,}}

\renewcommand{\O}{\mathcal{O}}
\newcommand{\M}{\mathcal{M}}
\newcommand{\D}{\mathcal{D}}
\newcommand{\E}{\mathcal{E}}
\renewcommand{\S}{\mathcal{S}}
\newcommand{\A}{\mathcal{A}}

\newcommand{\Z}{\mathcal{Z}}

\usepackage[font=footnotesize]{caption}

\usepackage{makecell, tabularx, multirow,  adjustbox}
\usepackage{subcaption}
\newcommand{\addendum}[1]{{\color{black} #1}}

\usepackage{xspace}
\newcommand{\pwc}{\textsc{PwC}\xspace}

\usepackage[protrusion=true,
    activate={true,nocompatibility},
    final,
    tracking=true,
    kerning=true,
    spacing=false,
    factor=1100]{microtype}
    
\SetTracking{encoding={*}, shape=sc}{10}

\theoremstyle{nospace} 
\theoremstyle{nospace} \newtheorem{proposition}{Proposition}
\theoremstyle{nospace} 
\theoremstyle{nospace} 
\theoremstyle{nospace} 
\theoremstyle{nospace} 
\theoremstyle{nospace} 
\theoremstyle{nospace} 

\newcommand\BibTeX{{\rmfamily B\kern-.05em \textsc{i\kern-.025em b}\kern-.08em
T\kern-.1667em\lower.7ex\hbox{E}\kern-.125emX}}

\def\volumeyear{2023}

\begin{document}

\runninghead{Learning-Based Perception with Safety Guarantees}

\title{Perceive With Confidence: Statistical Safety Assurances for Navigation with Learning-Based Perception}

\author{Zhiting Mei\affilnum{1}, Anushri Dixit\affilnum{2}, Meghan Booker\affilnum{3}\textsuperscript{*}, Emily Zhou\affilnum{1}, Mariko Storey-Matsutani\affilnum{1}, \\ Allen Z. Ren\affilnum{1}, Ola Shorinwa\affilnum{1}, Anirudha Majumdar\affilnum{1}}

\affiliation{\affilnum{1} Department of Mechanical and Aerospace Engineering, Princeton University, Princeton, NJ, USA\\
\affilnum{2} Department of Mechanical and Aerospace Engineering, University of California, Los Angels, CA, USA\\
\affilnum{3} Johns Hopkins University Applied Physics Laboratory, Laurel, MD, USA}

\corrauth{Anirudha Majumdar, Department of Mechanical and Aerospace Engineering, Princeton University, Princeton, NJ, 08540, USA}

\email{ani.majumdar@princeton.edu}

\begin{abstract}
Rapid advances in perception have enabled large pre-trained models to be used out of the box for transforming high-dimensional, noisy, and partial observations of the world into rich occupancy representations. However, the reliability of these models and consequently their safe integration onto robots remains unknown, particularly when deployed in environments unseen during training. To provide safety guarantees, we rigorously quantify the uncertainty of pre-trained perception systems for object detection and scene completion via a novel calibration technique based on conformal prediction. Crucially, this procedure guarantees robustness to distribution shifts in states when perception outputs are used in conjunction with a planner. As a result, the calibrated perception system can be used in combination with \emph{any} safe planner to provide an end-to-end statistical assurance on safety in unseen environments. We evaluate the resulting approach, \emph{Perceive with Confidence} (\pwc), in simulation and on hardware where a quadruped robot navigates through previously unseen static indoor environments. These experiments validate the safety assurances for obstacle avoidance provided by \pwc. 
In simulation, our method reduces obstacle misdetection by $70\%$ compared to uncalibrated perception models. While misdetections lead to collisions for baseline methods, our approach consistently achieves $100\%$ safety. We further demonstrate reducing the conservatism of our method without sacrificing safety, achieving a $46\%$ increase in success rates in challenging environments while maintaining $100\%$ safety. In hardware experiments, our method improves empirical safety by $40\%$ over baselines and reduces obstacle misdetection by $93.3\%$. The safety gap widens to $46.7\%$ when navigation speed increases, highlighting our approach’s robustness under more demanding conditions.
\end{abstract}


\maketitle
\footnotetext{\textsuperscript{*} Work conducted while affiliated with Princeton University.}

\section{Introduction}
\label{sec:intro}

How can we decide if the outputs of a given perception system are sufficiently reliable for safety-critical robotic tasks such as autonomous navigation? 
Significant strides in perception over the past few years have enabled large pre-trained models to be used out of the box~\citep{firoozi2023foundation} for tasks such as \emph{object detection} and \emph{occupancy prediction}, which serve as a fundamental building block for navigation. 
However, current pre-trained models are still not reliable enough for safe integration into many real-world robotic systems. Despite being trained on vast amounts of data, these systems can often fail to generalize to novel environments~\citep{eduardo2019survey, wang2023sam, sunderhauf2018limits}. 
In this paper, we ask: \emph{how can we leverage the power of large pre-trained perception models while providing safety assurances for robot navigation?}
 
Consider a legged robot tasked with navigating a cluttered environment such as a home, office, or warehouse (Figure~\ref{fig:processDiagram}). A typical navigation pipeline for such a system consists of two modules: (i) a perception module that detects obstacles, and (ii) a planner that produces collision-free trajectories assuming accurate perception. However, there are two challenges associated with obtaining reliable outputs from the perception module. First, the environments in which we deploy our robots will be \emph{unseen} during training, and thus require \emph{generalization} to new obstacle geometries, appearances, and other environmental factors. Second, \emph{closed-loop deployment} of the perception system in conjunction with a planner causes a shift in the distribution of \emph{states} (e.g., relative locations to obstacles) that are visited by the robot. Since the robot's planner influences future states, the robot may view obstacles from unfamiliar relative poses (Figure~\ref{fig:processDiagram}), which can cause the perception system to fail.
\begin{figure*}[h]
    \centering
    \vspace{-2pt}
    \includegraphics[width = 0.95\linewidth]{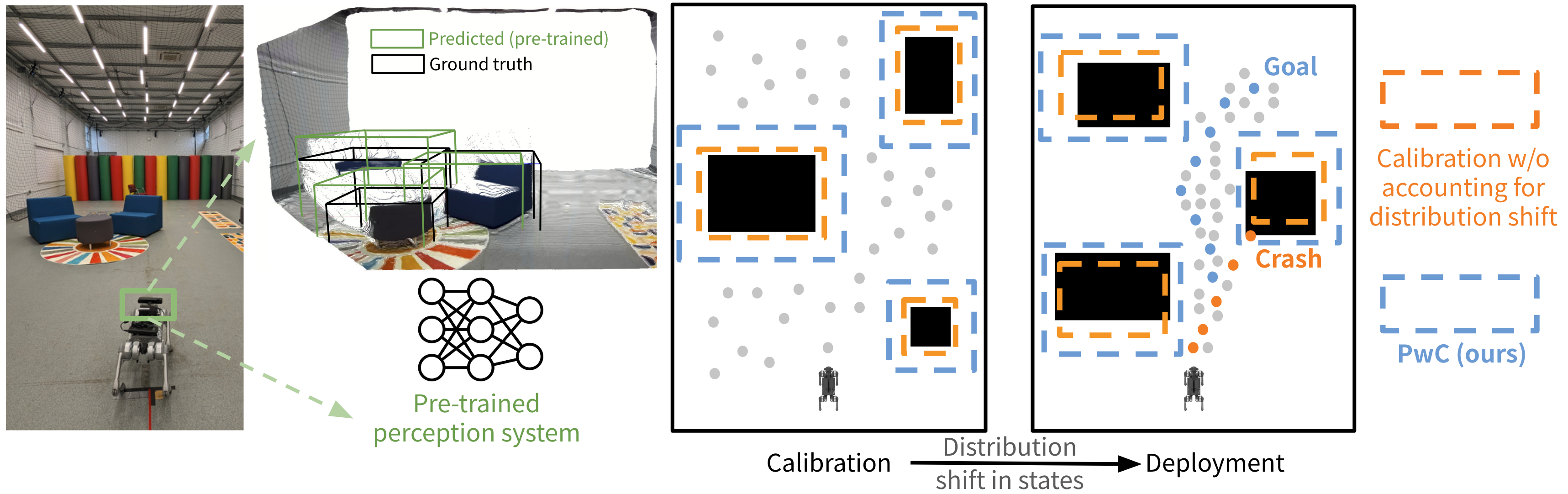}
    \caption{\pwc lightly processes the outputs of a pre-trained perception system (green bounding boxes) using conformal prediction in order to ensure a bounded misdetection rate despite \emph{any} distribution shift in states (gray dots).
    The calibrated perception system (blue boxes) paired with a non-deterministic filter and a safe planner provide an end-to-end statistical assurance on safety in new test environments.}
    \label{fig:processDiagram}
    \vspace{-15pt}
\end{figure*}

In this paper, we address these challenges by performing rigorous \emph{uncertainty quantification} for the outputs of a pre-trained perception system in order to achieve reliably safe (i.e., collision-free) navigation. We utilize techniques from \emph{conformal prediction}~\citep{vovk_algorithmic_2005} in order to lightly process the outputs of a pre-trained perception system in a way that provides a \emph{formal assurance} on correctness, i.e., with a user-specified probability $1-\epsilon$, the processed perception outputs will correctly detect obstacles in a \emph{new} environment. To enable this, we assume access to a relatively small-sized (e.g., $|\cdot| = 400$) dataset of environments that are representative of deployment environments with ground-truth obstacle annotations, and use these for \emph{calibrating} the outputs of the perception system. Crucially, we propose a novel calibration technique that ensures robustness of the perception system to \emph{any closed-loop distribution shift in states}. Hence, the calibrated outputs can be used in conjunction with \emph{any} safe planner to provide an end-to-end statistical assurance on safety in new static environments with a user-specified threshold $1-\epsilon$. To the best of our knowledge, this is the first work to calibrate a black-box perception system in a way that ensures robustness to closed-loop distribution shifts in order to provide end-to-end statistical assurances on safe navigation. 

Our proposed framework, \emph{Perceive with Confidence} (\pwc), is amenable to different types of perception systems, e.g., bounding-box prediction perception systems and scene-completion perception systems. A preliminary version of this work, which was presented at the Conference on Robot Learning~\citep{dixit2024perceive}, focused on the calibration of bounding-box predictions from onboard robot perception systems. 
In the preliminary version, we demonstrate the end-to-end safety guarantees provided by \pwc in simulated and hardware experiments on the Unitree Go1 quadruped navigating in indoor environments with objects that are unseen during calibration (Figure~\ref{fig:processDiagram}), where \pwc achieves up to $40\%$ increase in safety with only modest reductions in task completion rates compared to baselines that use the pre-trained perception model directly, fine-tune it on the calibration dataset, or utilize conformal prediction for uncertainty quantification but do not account for closed-loop distribution shift.

However, bounding boxes lack the high level of expressiveness required for high-accuracy occupancy predictions in environments with complex geometry. In this work, we extend \pwc to perception systems that predict 3D occupancy maps, providing higher-fidelity scene representations for robot planning and navigation. Specifically, in this paper, we derive procedures for calibrating unsigned distance functions predicted by scene-completion models to safely decompose a robot's environment into free and occupied space. We call the resulting method \pwc-NU-MCC, based on the scene-completion model used in this work, NU-MCC~\citep{lionar2023nu}. In addition, we present simulated experiments on a quadruped robot, demonstrating that \pwc-NU-MCC outperforms the calibrated bounding-box predictor, improving the goal-reaching rate by $46\%$, particularly in scenes with intricate geometry (Section~\ref{sec:numcc}). Further, we present additional quadruped robot experiments, achieving faster navigation with the robot moving more than three times faster compared to our original experiments (Section~\ref{sec:fast-robot}).

The rest of the paper is organized as follows: we first review relevant literature in~\Cref{sec:related work}. In \Cref{sec:problem formulation}, we formulate the problem of rigorous calibration of perception systems with safety guarantees. In \Cref{sec:CP}, we provide a brief introduction to conformal prediction. Next, we discuss the offline calibration of perception systems in \Cref{sec:approach-calibration} followed by a discussion of the online perception and planning procedure \Cref{sec:approach-planning}. We present simulated and hardware experiments on a quadruped robot in \Cref{sec:simulations,sec:hardware}, respectively. We conclude in \Cref{sec:conclusion} and provide additional discussion, including simulations and hardware results, in the Appendix.
\section{Related Work}
\label{sec:related work}

\textbf{Safe planning.} Collision avoidance is a crucial goal in autonomous navigation. Safe planning methods typically rely on the assumption that the robot has perfect knowledge of its state and environment~\citep{hsu_safety_2023}. Recent approaches have allowed for occlusion~\citep{janson_safe_2018, zhang2021safe, packer2023anyone, koschi2020set} or accounted for losing sight of a previously tracked object~\citep{laine2020eyes}, but still require either perfect detection of seen objects or bounded sensor noise. Such assumptions are impractical for learning-based perception modules that can fail catastrophically in new environments. 

\noindent\textbf{Formal assurances for perception-based control.}
Proposed methods include control barrier functions (CBFs) \citep{dean_guaranteeing_2021,dawson_learning_2022}, verification methods on neural networks (NNs) \citep{hsieh_verifying_2022,katz_verification_2021}, and other learning-based methods \citep{katz_verification_2021,ghosh_counterexample-guided_2021,dean_certainty_2021,sun_learning-based_2023,liu_autoregressive_2022,majumdar2021pac,farid2022task,farid2022failure}. However, these works either do not guarantee generalization to novel environments \citep{hsieh_verifying_2022,katz_verification_2021}, ignore closed-loop distribution shifts \citep{liu_autoregressive_2022,sun_learning-based_2023}, require end-to-end training and a good prior \citep{majumdar2021pac,farid2022task,farid2022failure}, or demand usage/design of specific controllers \citep{dean_guaranteeing_2021,dawson_learning_2022,ghosh_counterexample-guided_2021, sinha2023closing}. Some make strong assumptions on the perception system \citep{dean_robust_2020,chou_safe_2022} that are unrealistic for deployment. In contrast, our method doesn't need any of the above, and is lightweight and modular, allowing for the use of any downstream safe planners to ensure end-to-end safety.

\noindent\textbf{Conformal prediction.}
Conformal prediction (CP)~\citep{vovk_algorithmic_2005,pmlr-v25-vovk12,angelopoulos_gentle_2022} is an uncertainty quantification framework particularly suitable for robotics applications~\citep{ren2023robots,lindemann_safe_2023,dixit_adaptive_2022,luo_sample-efficient_2023} where learned modules are deployed in environments drawn from unknown distributions. In this work, we focus on providing uncertainty quantification for the perception system, which usually involves high-dimensional inputs and closed-loop distribution shifts. Prior works~\citep{dixit_adaptive_2022,sun_learning-based_2023,yang2023safe,park_pac_2020} either provide guarantees for a single environment, assume known environments, or do not account for closed-loop distribution shifts. To the best of our knowledge, this is the first work to obtain end-to-end safety assurances for the perception and planning system in new environments while being robust to closed-loop distribution shifts and amenable to changes in the planner parameters.

\section{Problem Formulation}
\label{sec:problem formulation}

\noindent{\bf Dynamics and environments.} Suppose that the dynamics of the robot are described by $s_{t+1} = f_E(s_t, a_t)$, 
where $s_t \in \S$ is the robot's state at time-step $t$, $a_t \in \A$ is the action, and $E \in \E$ is the \emph{environment} that the robot operates in during a given episode. We primarily focus on navigation with static obstacles; in this context, the environment $E$ specifies the locations and geometries of objects.  
We assume that environments that the robot will be deployed in are drawn from an \emph{unknown} distribution $\D_\E$, e.g., a distribution over possible rooms that the robot may be deployed in. We will make no assumptions on this distribution besides the ability to sample a finite dataset $D = \{E_1, \dots, E_N\}$ of independent identically distributed (i.i.d.) environments from $\D_\E$. 

\smallskip
\noindent{\bf Sensor and perception system.} We consider a robot equipped with a sensor ${\sigma: \S \times \E \rightarrow \O}$ that provides observations ${o_t = \sigma(s_t, E)}$ (e.g., depth images) based on the robot's state and environment. 
We assume access to a pre-trained perception model $\phi: \O \rightarrow \Z$, which processes raw sensor observations $o_{t}$ into an occupancy representation of the environment $z_{t}$. 
For example, models for 3D object detection can produce bounding boxes for obstacles~\citep{supergradients}, perform shape completion~\citep{wu2023multiview}, or predict free space in the environment \citep{chatzipantazis2023mathrmseequivariant}. 
We demonstrate our framework with two types of perception models: (i) models for obstacle detection that output 3D bounding boxes, and (ii) models for scene reconstruction that output occupancy grid maps.
The representations $(z_0, \dots, z_t)$ up to the current time-step are aggregated into an overall representation $m_t \in \M$ (e.g., a map). 
We denote predicted occupied space in green and predicted free space in blue, except where noted otherwise.

\smallskip
\noindent{\bf Planner and Policy.} The planner utilizes the environment representation $m_t$ to compute actions for the given task. We denote the resulting end-to-end policy that utilizes a perception model $\phi$ by ${\pi^\phi: \O^{t+1} \rightarrow \Z^{t+1} \rightarrow \M \rightarrow \A}$, which maps histories of sensor observations to actions.

\smallskip
\noindent{\bf Safety and task performance.} Let $C_E^\text{safe}$ be a cost function that captures safety (e.g., obstacle avoidance). In addition, let $\S_{0,E}$ denote the allowable set of initial conditions in environment $E$. Then, $C_E^\text{safe}(\pi^\phi) \in \{0,1\}$ assigns a cost of $0$ if policy $\pi^\phi$ maintains safety from any initial state $s_0 \in \S_{0,E}$ when deployed over a given time horizon in environment $E$, and a cost of $1$ otherwise. Although we only present safety-oriented cost functions here, additional cost functions can be used to capture task performance, e.g., $C_E^\text{task}$ which minimizes the time to reach the goal.

\smallskip
\noindent{\bf Goal (statistical safety assurance).} Our goal is to provide a statistical assurance on safety for the end-to-end policy $\pi^\phi$. We propose a procedure that uses a finite dataset $D$ of environments in order to produce a \emph{calibrated} perception system $\tilde{\phi}: \O \xrightarrow{\phi} \Z \xrightarrow{\rho} \Z$. Our approach is modular: outputs of the calibrated perception system may be used with \emph{any} safe planner (cf. Section~\ref{sec:approach-planning}) to provide probabilistic guarantees on safety, with:
\begin{equation}
     C_{\D_\E}^\text{safe}(\pi^{\tilde{\phi}}) := \underset{E \sim \D_\E}{\EE} \ \Big{[} C_E^\text{safe}(\pi^{\tilde{\phi}}) \Big{]} \ \leq  \ \epsilon ,
\end{equation}
for a user-specified safety tolerance $\epsilon$, while also post-processing outputs from $\phi$ as lightly (i.e., non-conservatively) as possible in order to allow the robot to optimize task performance. 
\section{Background: Conformal Prediction}
\label{sec:CP}

We leverage the theory of conformal prediction (CP) to perform rigorous uncertainty quantification for perception. Here, we provide a brief overview of conformal prediction and refer interested readers to \citep{vovk_algorithmic_2005, angelopoulos_conformal_2023} for a more detailed discussion.

Given $N$ i.i.d. (or \emph{exchangeable}) samples $U_1, \dotsc, U_N$ of a scalar random variable $U$, we compute the threshold, $\hat{q}_{1-\epsilon}$, such that the next sample, $U_{\text{test}}$, satisfies,
\begin{align}\label{eq:marginal}
    &\mathbb{P}[U_{\text{test}} \leq \hat{q}_{1-\epsilon}] \geq 1-\epsilon, \\
    &\hat{q}_{1-\epsilon} = \begin{cases}
U_{(\lceil (N+1)(1-\epsilon) \rceil)} \ \ \text{if } \lceil (N+1)(1-\epsilon) \rceil \leq N,\\
\infty \ \ \text{otherwise},
\end{cases}\nonumber
\end{align}
where  $U_{(1)}\leq U_{(2)} \leq \dotsc \leq U_{(N)}$ are the order statistics (sorted values) of the N samples  $U_1, \dotsc, U_N$. 
In the CP literature, the non-conformity score $U$ represents a measure of the (in)correctness of a model.

The guarantee in \eqref{eq:marginal} is \textit{marginal}, i.e.,~\eqref{eq:marginal} holds over the sampling of both the calibration dataset $U_1, \dotsc, U_N$ and the test variable $U_{\text{test}}$.  Hence, we will need to generate a fresh set of i.i.d. calibration data $\tilde{U}_1, \dotsc, \tilde{U}_N$ for the guarantee to hold for a new sample $\tilde{U}_{\text{test}}$. However, in practice, one typically only has access to a single dataset of examples; inferences from this dataset must be used for all future predictions on test examples. In this work, we use the following dataset-conditional guarantee\addendum{~\citep{pmlr-v25-vovk12, angelopoulos_gentle_2022}} that doesn't require us to generate of $N$ new samples for every test prediction and holds with probability $1-\delta$ over the sampling of the calibration dataset:
\begin{align}\label{eq:dataset_conditional_cp}
    &\mathbb{P}[U_{\text{test}} \leq \hat{q}_{1-\epsilon} | U_1,
    \dotsc, U_N ] \geq \text{Beta}_{N+1-v, v}^{-1}(\delta), \\
    &v := \lfloor (N+1)\hat{\epsilon} \rfloor,\nonumber
\end{align}
where, $\text{Beta}_{N+1-v, v}^{-1}(\delta)$ is the $\delta-$quantile of the $\text{Beta}$ distribution with parameters $N+1-v$ and $v$, and we can choose $\hat{\epsilon}$ to achieve the desired $1-\epsilon$ coverage.
\section{Calibrating the Perception System}
\label{sec:approach-calibration}

In this section, we describe our approach to the uncertainty quantification of a pre-trained perception system. We focus on the challenges highlighted in Section~\ref{sec:intro}: providing statistical assurances on safe generalization to novel environments and ensuring that the offline calibration procedure is robust to shifts in the distribution of states induced by the online implementation of the planner.

We consider two types of perception systems: (i) perception systems that output bounding boxes predicting the locations of objects in the environment, and (ii) perception systems that perform scene reconstruction and output occupancy grid maps representing the occupancy of the environment. For example, Figure~\ref{fig:TwoPredictors} illustrates the maps produced by (i) the bounding-box predictor and (ii) the occupancy prediction.

\begin{figure}
    \centering
    \includegraphics[width=\columnwidth]{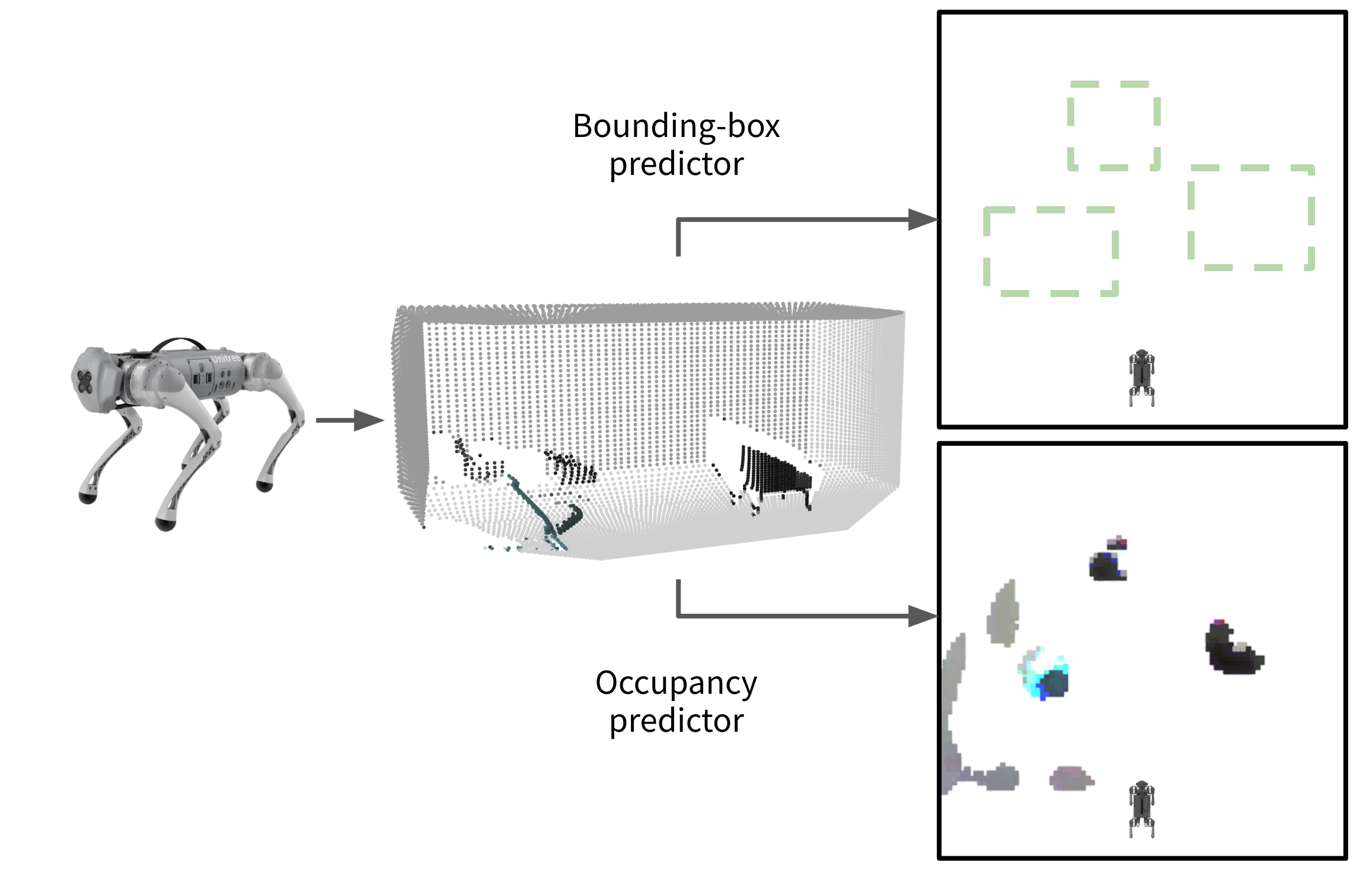}
    \caption{Our proposed method is amenable to a range of perception systems, e.g., bounding-box predictors (top), which output a map with predicted bounding boxes (green dashed boxes), and occupancy predictors (bottom), which output a map with predicted occupied regions.}
    \label{fig:TwoPredictors}
\end{figure}

\subsection{Misdetection Rate}

Our key idea for ensuring \textit{generalization} to new, unseen environments and tackling \textit{the distribution shift arising from the closed-loop deployment} of the calibrated perception system with a plannner is to use a \emph{policy-independent} misdetection cost, $\tilde{C}_E$, which considers worst-case errors across \textit{all} states in an environment\footnote{It would be infeasible to consider \textit{all} possible states in an environment. In practice, we use a sampling-based motion planner and consider a fixed set of samples for our calibration that could be used by any planner.},
${
    \tilde{C}_E(\phi)  :=  \max_{s\in\S} \; \mathbbm{1}_{\X^\text{occ}\not\subseteq \overline\X^\text{occ}_s }.
}$
We present a calibration procedure that bounds this misdetection cost with high probability in a new environment, and thus guarantee the correctness of the calibrated perception system independent of the robot policy using CP. Moreover, we note that the perception system can be fine-tuned to the target deployment environments to reduce the nominal misdetection rate, which we discuss in Appendix~\ref{appendix:extensions}.

\subsection{Calibration Procedure}

\noindent\textbf{Dataset.} We assume access to a dataset of $N$ i.i.d. environments $D = \{E_1, \dotsc E_N\}\sim{\D_{\E}}$ (cf. Section~\ref{sec:problem formulation}).  Let $\X_i$ denote the configuration space of environment $E_i$ (e.g., $x$-$y$ location). 
In each environment, $E_i$, we have access to the ground-truth occupied space $\mathcal X^\text{occ}_i$ and the predicted occupied space $\overline{\mathcal X}^\text{occ}_{s,i}$, generated by the pre-trained perception system $\phi$ for each state $s\in\mathcal S$.
\addendum{Care is required to ensure that the calibration environments are representative of deployment environments. As such, we construct the calibration dataset using real-world environments or create simulation environments using real-world data~\citep{uy-scanobjectnn-iccv19, berk2017yale, fu20213d} to ensure sufficient variation in environmental factors (e.g., geometry and locations of obstacles, lighting, etc.).} 

\smallskip
\noindent\textbf{Calibration.} In each calibration environment $E_i$, we define a parameter $q_i$ that monotonically scales the predicted occupied space to be  $\overline{\mathcal X}^\text{occ}_{s,i} (q_i)$. In other words, as we increase $q_i$, $\overline{\mathcal X}^\text{occ}_{s,i} (q_i)$ expands monotonically. We find the $q_i$ such that the ground truth occupied space is fully enclosed by the scaled prediction, i.e., $ {\mathcal X}^\text{occ}_i\subseteq  \overline{\mathcal X}^\text{occ}_{s,i}(q_i), \forall s \in \S$, where, $\S$ is assumed to be a finite, discrete set.
In~\Cref{sec:simulations}, we provide concrete examples on how to choose the parameter $q$ for the two types of perception models considered.
We define the \emph{non-conformity score} for environment $E_i$ to be the minimum required scaling parameter $q_i$ in that environment:
\begin{align}
\label{eq:nonconformity}
   U_i \ = \  \min_{q_i} \; q_i \;\;\;\text{s.t} \quad \mathcal X^\text{occ}_i \subseteq \overline{\mathcal X}^\text{occ}_{s,i}(q_i), \forall s\in\S.
\end{align}
Observe that $U_i \leq 0 \implies {\mathcal X}^\text{occ}_{i} \subseteq \overline{\mathcal X}^\text{occ}_{s,i}, \, \forall s \in \S$ and a growing $U_i$ signals a worse performance of the pre-trained perception system. We can compute the nonconformity scores for the i.i.d. sampled environments $\{E_1, \dotsc, E_N\}$ and the quantile $\hat{q}_{1-\epsilon} = \text{Quantile}\bigg(U_{(1)},\dotsc,U_{(N)}; \frac{\lceil(N+1)(1-\hat{\epsilon})\rceil}{N} \bigg)$. 
Here, $\hat{\epsilon}$ is the calibration threshold such that the dataset conditional guarantee~\eqref{eq:dataset_conditional_cp} achieves the desired $(1-\epsilon)-$coverage with probability $1-\delta = 0.99$ over the sampling of the calibration dataset.

\begin{proposition}
\label{prop:calibration}
	Consider the \textit{calibrated} perception system $\tilde{\phi}$ that modifies every output of the perception system $\phi$ by scaling the predicted occupied space as $ \overline{\mathcal X}^\text{occ}_{s,i}(q_i)$. With probability $1-\delta$ over the sampling of the dataset used for calibration, the calibrated perception system, $\tilde{\phi}$, is guaranteed to have an $\epsilon$-bounded misdetection rate on \emph{new} test environments:
{\begin{equation}\label{eq:cp_perception_guarantee}
    	\underset{E_{\text{test}} \sim \D_\E}{\EE} \Big[\tilde{C}_{E_{\text{test}}}(\tilde{\phi})| U_1, \dotsc, U_N \Big]\leq \epsilon.
	\end{equation}}
\end{proposition}

\begin{proof}
    As seen in Section~\ref{sec:CP}, conformal prediction gives us the following \textit{dataset-conditional} guarantee on a new sample of the nonconformity score $U_{\text{test}}$ corresponding to a test environment $E_{\text{test}}$. With probability $1-\delta$ over the sampling of $U_1, \dotsc, U_N$,
\begin{equation*}
	\mathbb{P}[U_{\text{test}} \leq \hat{q}_{1-\epsilon} | U_1, \dotsc, U_N ] \geq \text{Beta}_{N+1-v, v}^{-1}(\delta).
\end{equation*}

We can rewrite the event $U_{\text{test}} \leq \hat{q}_{1-\epsilon}$ as:
\begin{align*}
	&\{U_{\text{test}} \leq \hat{q}_{1-\epsilon}\}  \\ 
    =  &\Big\{ \hat{q}_{1-\epsilon} \geq \min_{q_{\text{test}}} \ q_{\text{test}} \mid  \mathcal X^\text{occ}_{\text{test}} \subseteq \overline{\mathcal X}^\text{occ}_{s,\text{test}}(q_{\text{test}}), \forall s\in\S\ \Big\} \\
 	= &\Big\{ \mathcal X^\text{occ}_{\text{test}} \subseteq \overline{\mathcal X}^\text{occ}_{s,{\text{test}}}(\hat{q}_{1-\epsilon}), \forall s\in\S \Big\} \\
	= &\Big\{ \tilde{C}_{E_{\text{test}}}(\tilde{\phi}) = 0\Big\},
\end{align*}
which gives us the desired result~\eqref{eq:cp_perception_guarantee}.
\end{proof}

\Cref{prop:calibration}
gives us a formal assurance on the correctness of the perception system \emph{independent of the robot's policy}. As we describe below, the calibrated perception can thus be combined with \emph{any} safe planner to bound the collision rate to $\epsilon$. The calibrated perception outputs are guaranteed to be correct with probability $1-\epsilon$ over environments. Since we accounted for the perception error from every state in each environment, the resulting calibrated outputs are also guaranteed to be correct from every state in new test environments. Given that we have addressed the challenge of closed-loop distribution shift, we can now utilize this calibrated perception system with any safe planner to obtain a statistical assurance on robot safety. 

\subsection{Implementation with a limited field-of-view}~\label{appendix:fov}
A natural question that arises after following the calibration procedure described above is: what happens if the robot is not able to observe all objects in the environment from all states? This may happen due to a limited sensing capability or because some parts of the environment are occluded from view. We address this issue in our calibration procedure implementation by only taking into account perception errors for objects that are within the field-of-view of the robot in a given state, and masking any region of the ground-truth occupied space that is not visible to the robot, i.e., $\X^\text{occ}$ (which now depends on state $s$) is the ground-truth visible occupied region. Hence, the perception system correctness assurance stated above holds for all objects within the field-of-view of the robot at any given state. The presence of possibly occluded obstacles is dealt with by a safe planner, which we describe next. 

\section{Perception and Planning}
\label{sec:approach-planning}
\begin{figure}[h]
    \begin{minipage}{0.47\linewidth}
    \centering\captionsetup[subfigure]{justification=justified}
    \includegraphics[width = \linewidth]{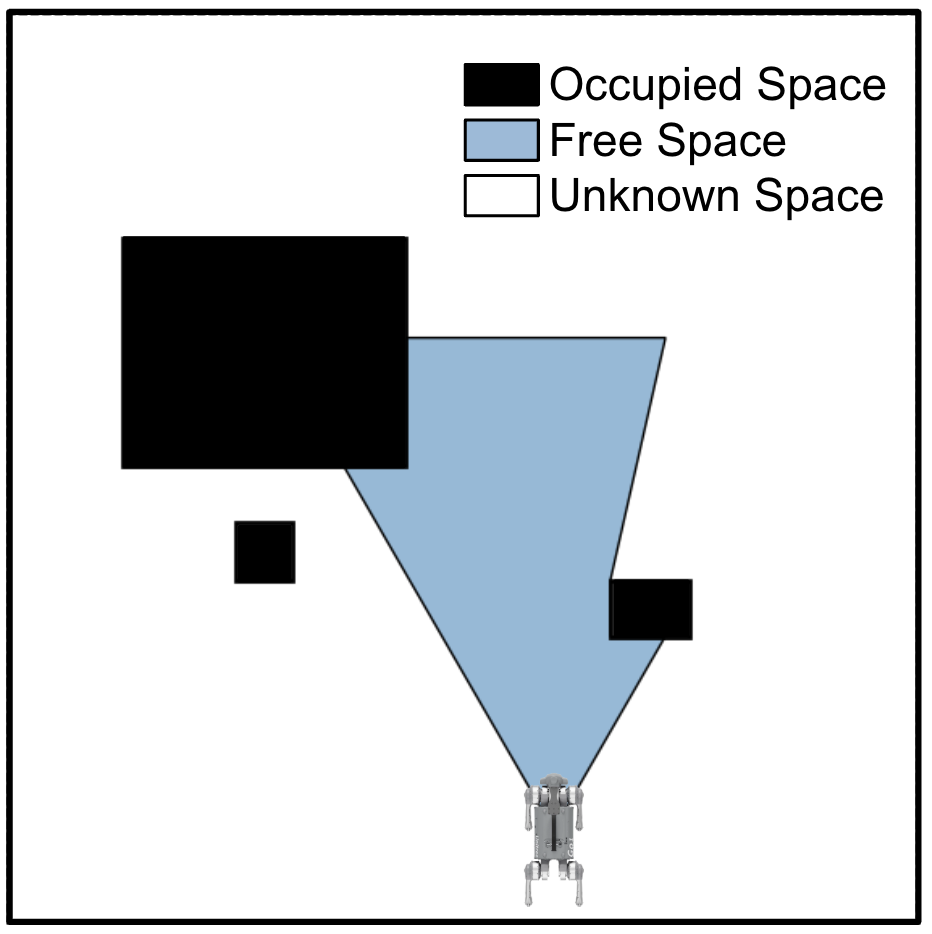}
    \caption{The configuration space is partitioned into three.}
    \label{fig:configSpace}
    \end{minipage}
    \hfill
    \begin{minipage}{0.47\linewidth}
    \centering\captionsetup[subfigure]{justification=justified}
    \includegraphics[width=\linewidth]{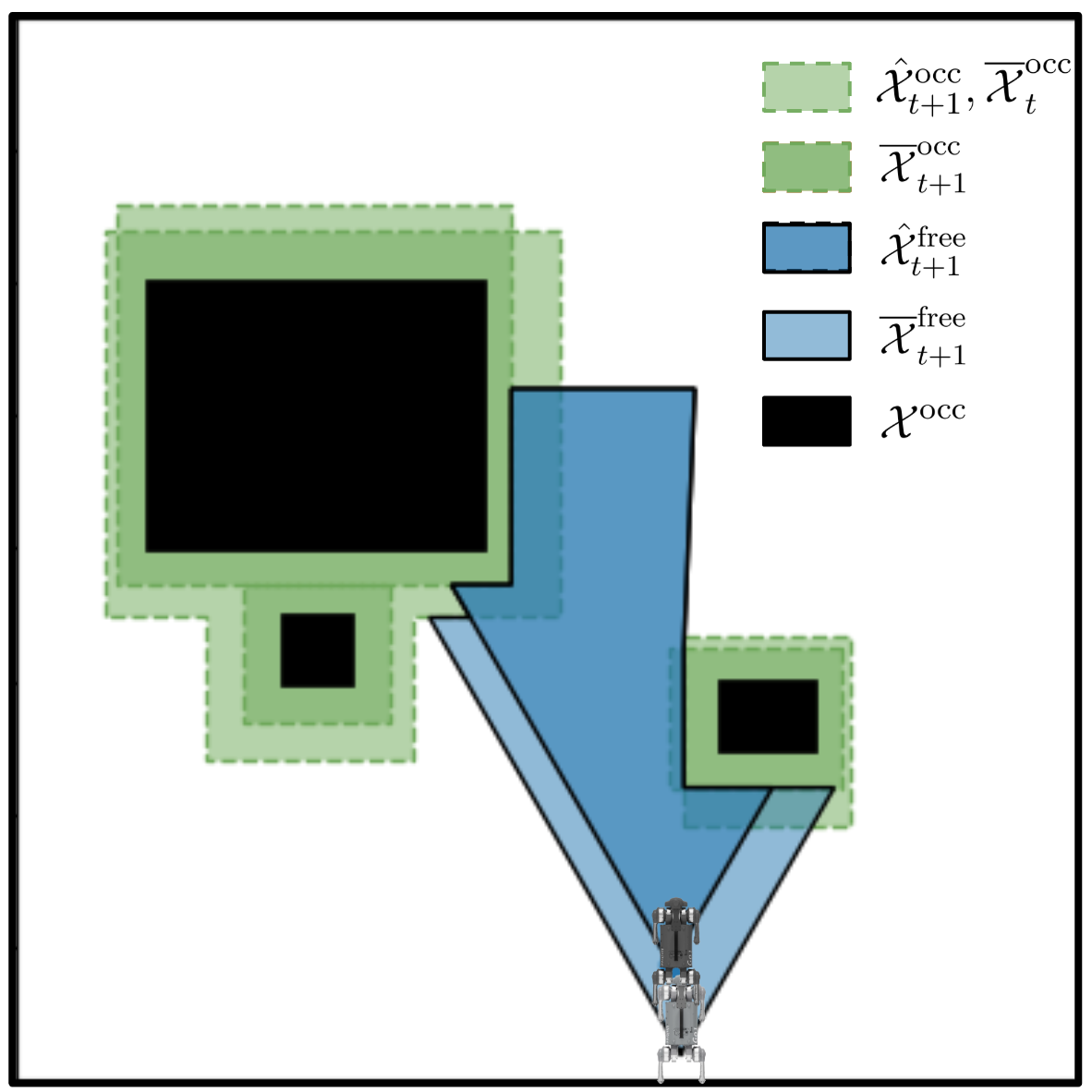}
    \caption{The filter takes union over the free space.}
    \label{fig:nonDetOnly}
\end{minipage}
\end{figure}

We now focus on the online implementation of the method described in Section~\ref{sec:approach-calibration} to reduce conservatism when used in conjunction with a safe planner. In general, a safe planner takes into account the dynamics of the robot and produces plans in the state space $\mathcal S$. Let $\X$ be the configuration space of the robot (e.g., $x$-$y$ location for a point). The configuration space of any given environment $E$ can be partitioned into the ground-truth occupied space $\X^\text{occ}$, the known free space $\X^\text{free}$, and the unknown space $\X^\text{unknown}$ (Figure~\ref{fig:configSpace}). 

\smallskip
\noindent\textbf{Non-deterministic filter.} 
We utilize the assurance obtained from Section~\ref{sec:approach-calibration} to implement a \emph{non-deterministic filter}~\citep[Ch. 11.2.2]{lavalle2006planning}, which shrinks the occupied space and grows the known free space over time (Figure~\ref{fig:nonDetOnly}). 
Suppose the robot's perceived partition of the configuration space $\mathcal{X}$ at time $t$ is denoted by the triplet $\{\overline{\mathcal{X}}^\text{free}_t,\overline{\mathcal{X}}_t^\text{occ}, \overline{\mathcal{X}}_t^\text{unknown}\}$, which represents the overall map $m_t$ of the environment. At a new time step $t+1$, the robot's perception system returns a new estimation for the occupied space, $\hat{\mathcal{X}}_{t+1}^\text{occ}$. The filter intersects the occupied spaces: $\overline{\mathcal{X}}_{t+1}^\text{occ} = \overline{\mathcal{X}}_t^\text{occ}\cap\hat{\mathcal{X}}_{t+1}^\text{occ}.$
We compute the new estimation of free space $\hat{\mathcal{X}}_{t+1}^\text{free}$ based on $\overline{\mathcal{X}}_{t+1}^\text{occ}$, considering occlusion and limited field of view. The new perceived free space is updated by taking the union: $\overline{\mathcal{X}}_{t+1}^\text{free} = \overline{\mathcal{X}}_t^\text{free}\cup\hat{\mathcal{X}}_{t+1}^\text{free}$.

The non-deterministic filter pairs effectively with our method in Section~\ref{sec:approach-calibration} for two key reasons: 1) it mitigates the conservatism of our expansion procedure for the predicted occupied space by intersecting $\overline{\mathcal X}_t^\text{occ}$, rapidly reducing its size even if the initial prediction with CP bounds appears generous; and 2) Proposition~\ref{prop:calibration} ensures that with high probability in a new test environment, $\overline{\mathcal X}_t^\text{free}$ never intersects the true occupied space ${\mathcal X}^\text{occ}$. We demonstrate the rapid expansion of known free space in Figure~\ref{fig:nonDet} for our simulated setup (Section ~\ref{sec:simulations}).

\smallskip
\noindent\textbf{Safe planning.} With our formal assurance on the estimated free space $\overline\X_t^\text{free}$, we can utilize \emph{any} safe planner~\citep{schouwenaars2002safe,bouraine_real-time_2016,pairet_online_2022} to ensure end-to-end safety, as long as the planner includes a safety filter that takes into account the robot's dynamics in order to reject potentially unsafe actions with the assumption of known robot states and a static (but unknown) environment~\citep[Corollary 1.4]{hsu_safety_2023}.

For our simulation and hardware experiments, we use the safe planner proposed in \citep{janson_safe_2018} due to its approximate optimality. The safety filter in this case is an inevitable collision set (ICS) constraint \citep{fraichard_inevitable_nodate}, where the robot is forbidden to enter any state that will eventually result in collision no matter what control actions are taken. Within the known free space $\overline\X_t^\text{free}$, the robot plans using the fast marching tree algorithm (FMT$^\star$) \citep{janson_fast_2015} with dynamics \citep{schmerling_optimal_2015}. If the goal is not visible within $\overline\X_t^\text{free}$, the robot plans to an intermediate goal on the boundary of its free space. The intermediate goals are chosen based on the cost-to-come from current robot state to the intermediate goal, and the distance-to-go from the intermediate goal to the actual goal. The robot replans whenever it receives a sensor update and an updated $\overline\X_{t+1}^\text{free}$ from its non-deterministic filter, and accounts for ICS constraints \citep{schmerling_optimal_2015-1} in-between sensor updates.

\begin{proposition}
    For any user-specified safety tolerance $\epsilon$, the \emph{calibrated} perception system $\tilde{\phi}$ in Proposition~\ref{prop:calibration} combined with any safe planner that chooses actions based on the outputs of the non-deterministic filter ensures the end-to-end safety for the overall policy $\pi^{\tilde \phi}$:
    \begin{equation}
     C_{\D_\E}^\text{safe}(\pi^{\tilde{\phi}}) := \underset{E \sim \D_\E}{\EE} \ \Big{[} C_E^\text{safe}(\pi^{\tilde{\phi}}) \Big{]} \ \leq  \ \epsilon ,
\end{equation}
where $C_E^\text{safe}(\pi^{\tilde{\phi}})$ is the cost for safety from Section \ref{sec:problem formulation}. 
\label{prop:overallSafe}
\end{proposition}

\begin{proof}
    As shown in Proposition \ref{prop:calibration}, the misdetection rate of the calibrated perception system $\tilde\phi$ is $\epsilon$-bounded on environments drawn from $\D$ at each time step $t$, where the robot is at state $s_t$. In other words, the predicted occupied space $\hat\X_t^{\text{occ}}$ at each time step contains the true occupied space $\X^\text{occ}$ with high probability across environments. Conversely, the predicted free space $\hat\X_t^{\text{free}}$ at each time step does not intersect with the true occupied space $\X^\text{occ}$ with high probability across environments. If we consider a safety-relevant misdetection cost at time step $t$:
   \begin{equation}
    \hat{C}^{\text{safe}}_{E}(\tilde\phi, s_t) = \begin{cases}
1 \ \quad \text{if } \X^\text{occ}\subseteq \hat\X_t^{\text{free}}\text{ (unsafe)},\\ 
0 \ \quad \text{otherwise},
\end{cases}
\end{equation}
then the misdetection rate over the set of states should be $\epsilon$-bounded across environments by Proposition \ref{prop:calibration}:
\begin{equation}
    \underset{E \sim \D_\E}{\EE} \max_{t\in [0,T]} \; \; \hat{C}^\text{safe}_{E}(\tilde\phi, s_t) \leq \epsilon.
\label{eq:safetyMisdetection}
\end{equation}
Because the expectation in Equation \eqref{eq:safetyMisdetection} is over the set of environments, the following statement holds in any new environment (with probability $1-\delta$ over the calibration dataset of environments),
\begin{equation}
   \mathbb{P}\Big{\{}\max_{t\in [0,T]} \; \; \hat{C}^\text{safe}_{E}(\tilde\phi, s_t) = 0 \Big{\}}\geq 1- \epsilon.
\end{equation}
Given $m_t=\{\overline\X^{\text{free}},\overline\X^{\text{occ}},\overline\X^{\text{unknown}}\}$, a safe planner never drives the robot outside of the free space. Therefore, the safe planner guarantees $C_{E}^\text{safe}(\pi^{\tilde{\phi}}) \leq \tilde{C}^\text{safe}_{E}(\tilde\phi) $.
\begin{equation}
   \mathbb{P}\Big{\{}C_{E}^\text{safe}(\pi^{\tilde{\phi}}) = 0 \Big{\}}\geq 1- \epsilon.
\end{equation}
\end{proof}

This result is a direct consequence of the formal assurance on the calibrated perception system that ensures correctness from \textit{any} state in a new test environment (sampled i.i.d. from the same distribution as the calibration environments) with probability $1-\epsilon$ \textit{over environments}. 

We discuss extensions of our perception and planning approach to problems with sensor and dynamics uncertainty in Appendix~\ref{appendix:extensions}.

\begin{figure}[th]
    \begin{minipage}{\linewidth}
        \centering\captionsetup[subfigure]{justification=justified}
        \includegraphics[width=\linewidth]{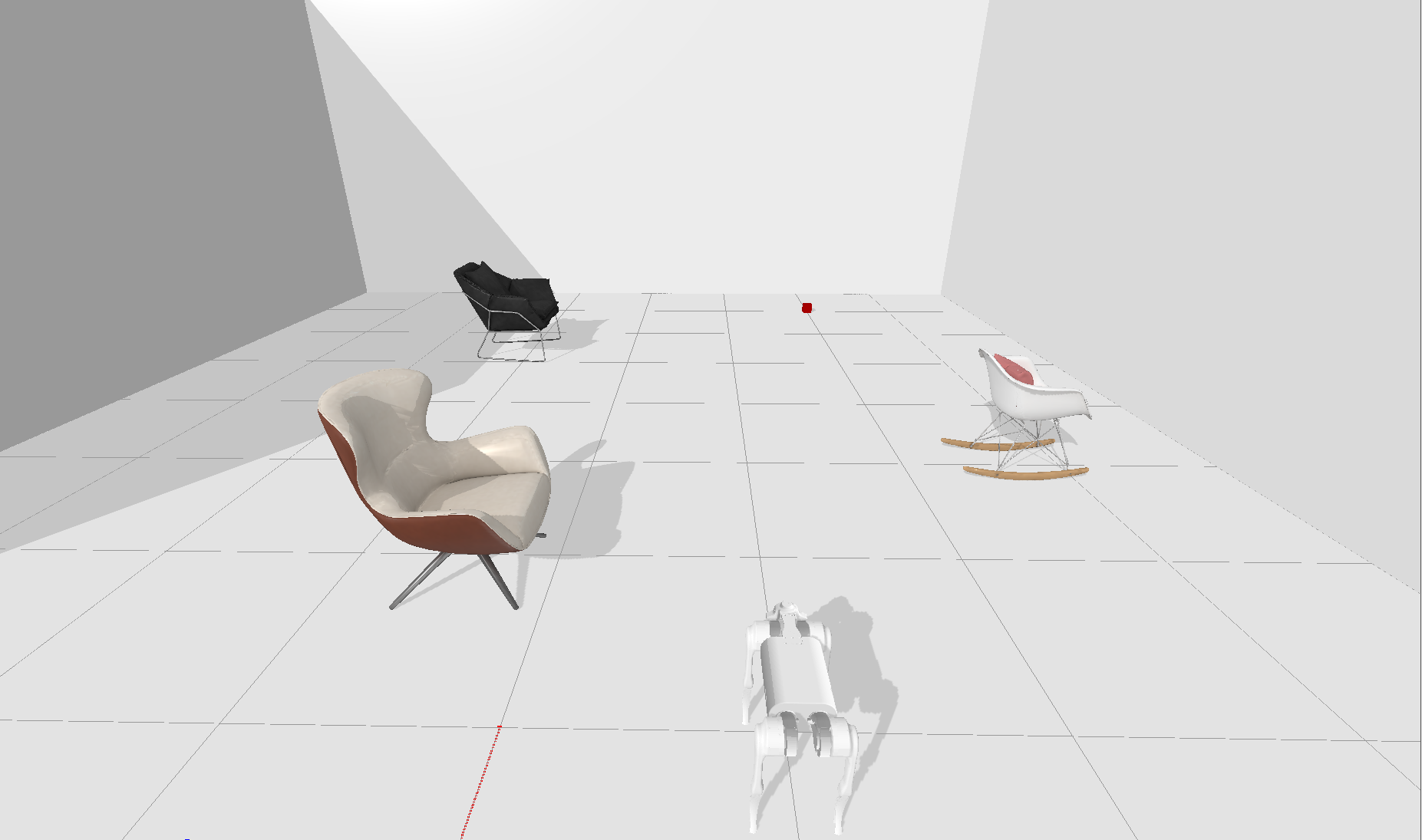}
        \caption{Simulation environment in Pybullet.}
        \label{fig:simEnv}
    \end{minipage}
\end{figure}
\begin{figure}[th]
    \begin{minipage}{\linewidth}
        \centering\captionsetup[subfigure]{justification=justified}
        \includegraphics[width=\linewidth]{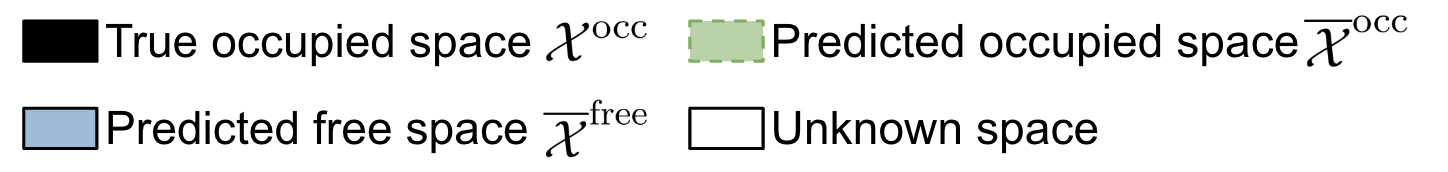}
        \label{fig:nonDetLegend}
        \vspace{-10pt} 
    \end{minipage}
    \begin{minipage}{0.48\linewidth}
        \centering\captionsetup[subfigure]{justification=justified}
        \includegraphics[width=\linewidth]{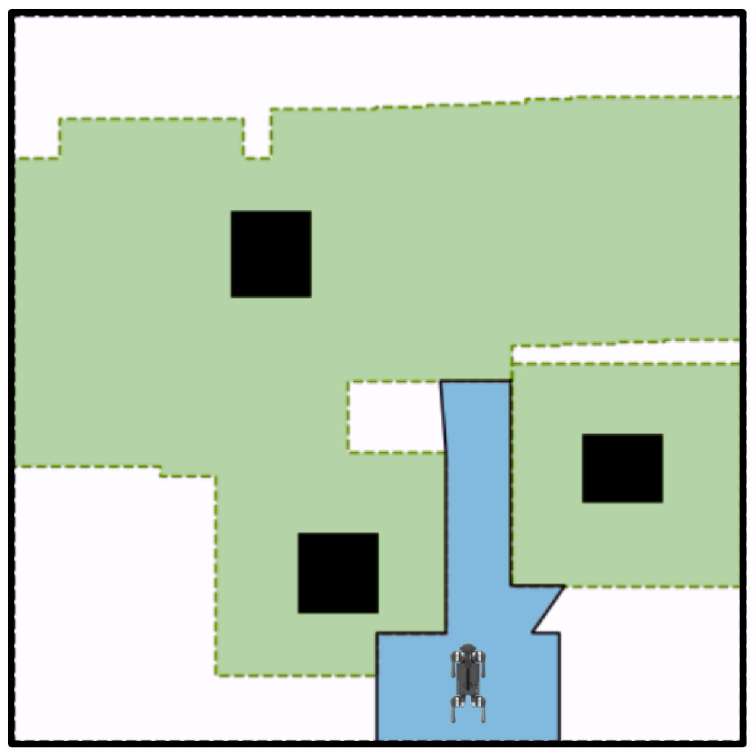}
        $t=1$
        \vspace{5pt}
        \label{fig:nonDet1}
    \end{minipage}
    \begin{minipage}{0.48\linewidth}
        \centering\captionsetup[subfigure]{justification=justified}
        \includegraphics[width=\linewidth]{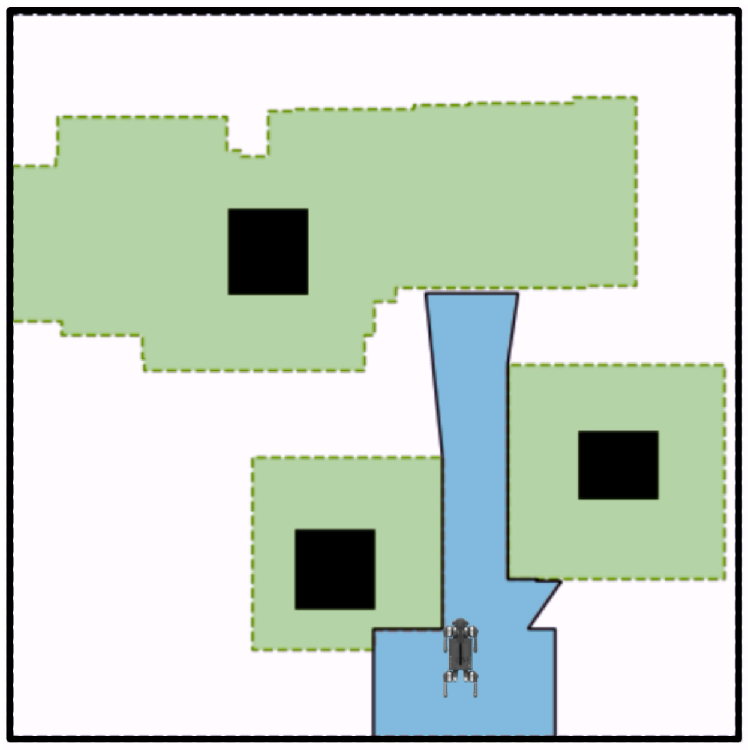}
        $t=8$
        \vspace{5pt}
        \label{fig:nonDet8}
    \end{minipage}
    \begin{minipage}{0.48\linewidth}
        \centering\captionsetup[subfigure]{justification=justified}
        \includegraphics[width=\linewidth]{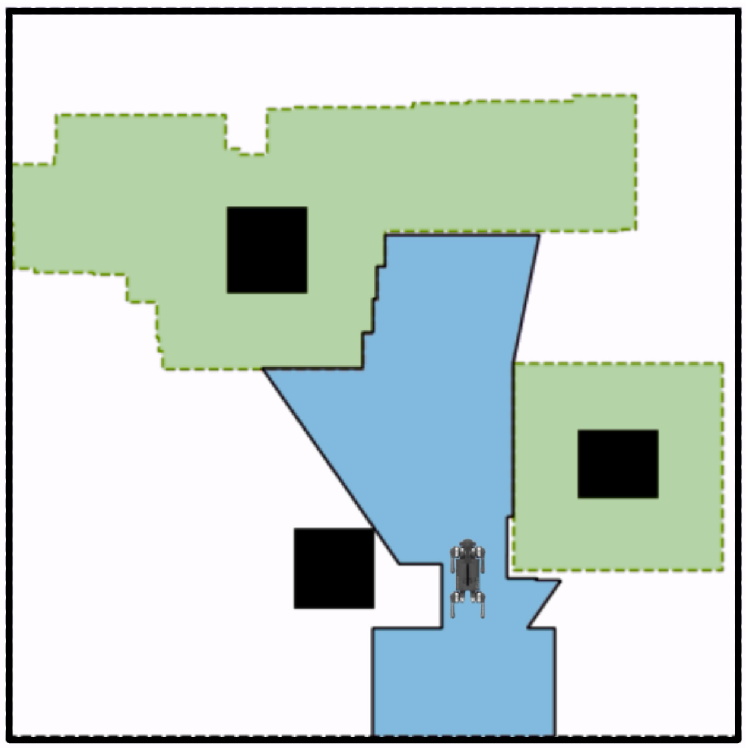}
        $t=14$
        \label{fig:nonDet14}
    \end{minipage}
    \hfill
    \begin{minipage}{0.48\linewidth}
        \centering\captionsetup[subfigure]{justification=justified}
        \includegraphics[width=\linewidth]{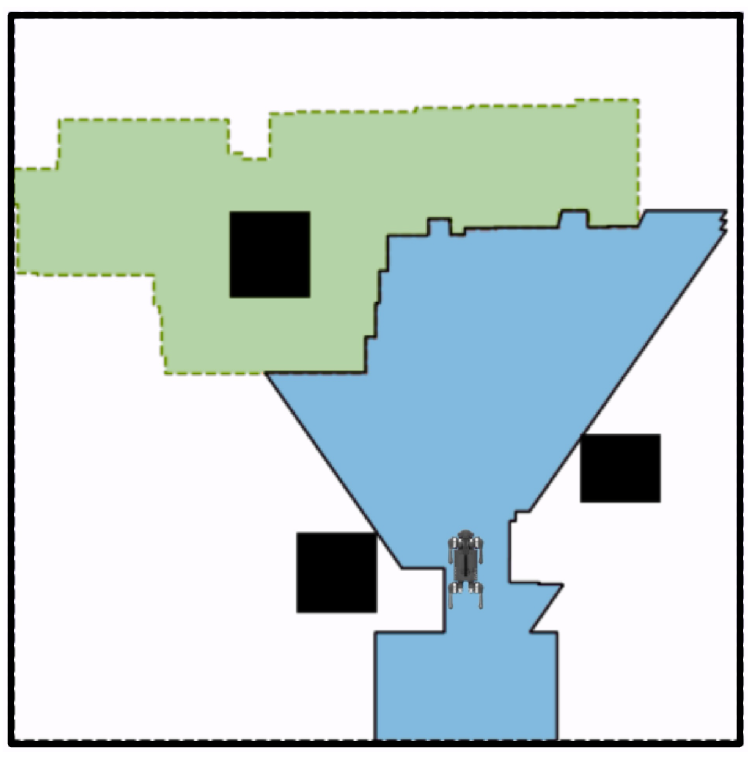}
        $t=17$
        \label{fig:nonDet17}
    \end{minipage}
    \caption{Simulation and non-deterministic filter updates. The robot begins with large occupied space predictions due to the inflation obtained through offline calibration (Section \ref{sec:approach-calibration}). After a few updates, the predicted occupied space $\overline{\mathcal{X}}^{\text{occ}}$ shrinks significantly.}
    \label{fig:nonDet}
\end{figure}

\begin{figure*}[t]
    \centering
    \includegraphics[width=0.95\linewidth]{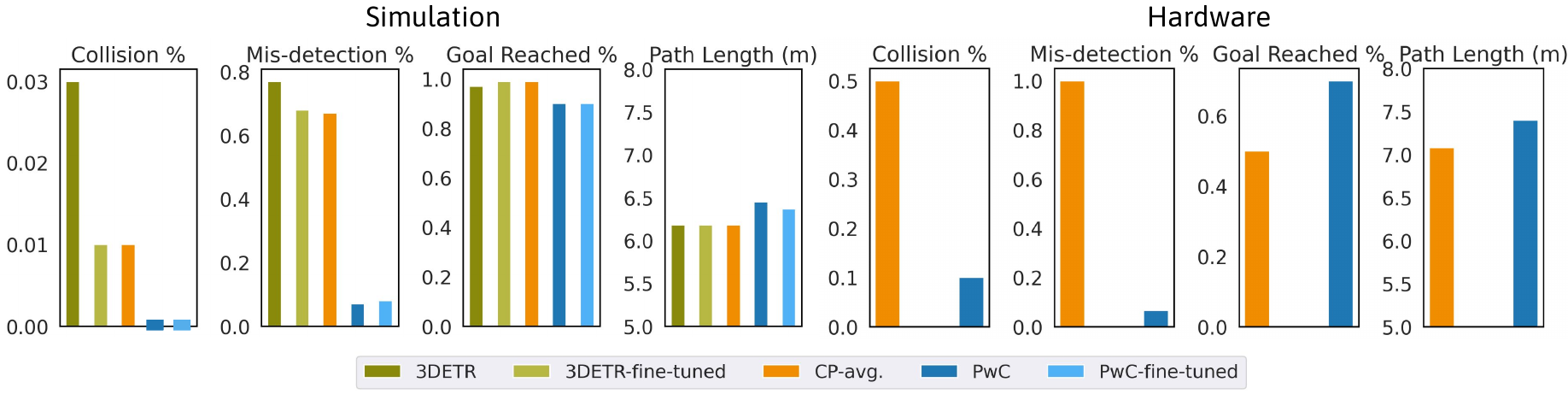}
    \caption{[${\epsilon = 0.15}$] \textbf{(Left)} Results for the simulated experiments across 100 new environments with 1 - 5 chairs (see Section \ref{sec:simulations}). \textbf{(Right)} Results for the hardware trials across 30 different chair configurations with 4-8 chairs described in Section \ref{sec:hardware}. Here the path length is averaged only for successful trials for both \pwc and CP-avg. due to the varying goal locations.}
    \label{fig:all-bars}
\end{figure*}

\section{Simulated Experiments}\label{sec:simulations}

We evaluate our approach for vision-based navigation in the PyBullet simulator ~\citep{coumans2020} using a diverse set of chairs from the 3D-Front dataset~\citep{fu20213d}. 

\smallskip
\noindent{\textbf{Simulation Environment.}}
We specify the environment distribution by randomly placing $1-5$ chairs from the diverse 3D-Front dataset~\citep{fu20213d} in an $8$ m $\times8$ m room (Figure~\ref{fig:simEnv}). We construct the simulation environment using CAD models of \emph{real} furniture pieces from the 3D-Front dataset~\citep{fu20213d}, which contains a highly diverse array of industrial CAD models developed by professional designers.

\smallskip
\noindent{\textbf{Robot Platform.}}
We evaluate each method on the Unitree Go1 quadruped robot, where we task the robot to navigate to a goal location that is about $7$m away from the initial location of the robot. The robot camera has a field of view of $70^\circ$ and a visibility range of $[1, 5]$ m. 

\smallskip
\noindent\textbf{Metrics for experiments.} We utilize the following metrics for our simulation experiments: a trial is counted as a collision if the robot collides with an obstacle and we count a misdetection for a trial if the free space predicted by the planner has any intersection with the ground-truth occupancy of the obstacles. We say that the goal has been reached in a given trial if the robot is able to navigate to within $1$ m around the goal in less than $140$ s. We also record the average path length for trials in which the goal is reached.

\subsection{Bounding-Box Predictors}~\label{sec:exps-calibration}

We first consider perception systems that output bounding-boxes. For our implementation, we use the 3DETR end-to-end transformer model~\citep{misra2021-3detr} as the pre-trained perception system. We compare our approach (\emph{Perceive with Confidence} --- \pwc) to three baselines to illustrate its effectiveness in achieving a user-specified safety rate. First, we consider the most common approach of directly using the outputs of the perception system~\citep{misra2021-3detr} in our planning pipeline. We call this baseline \textbf{3DETR}. Next, we consider the common practice of fine-tuning the outputs of the perception system using a small dataset of task-representative environments $D_{\text{tune}}$ (cf. Section~\ref{sec:fine-tuning}). We call this perception system \textbf{3DETR-fine-tuned}. Lastly, we perform calibration using conformal prediction; however, instead of accounting for the closed-loop distribution shift, we bound the misdetection rate averaged across environments \emph{and} states (similar to~\citep{sun_learning-based_2023}, which does not utilize conformal prediction, but quantifies expected errors in a perception system for a pre-defined distribution of states). We refer to this baseline as \textbf{CP- avg}. We consider two variations of our approach for comparison to the above baselines. First, we refine 3DETR outputs using our calibration procedure described in Section~\ref{sec:approach-calibration}. We call this approach \textbf{\pwc}. Second, the 3DETR outputs are fine-tuned and calibrated using split conformal prediction as described in Appendix~\ref{sec:fine-tuning}; we call this approach \textbf{\pwc-fine-tuned}. 

\smallskip
\noindent{\bf Calibration and Planning.} 
We implement the calibration procedure presented in Section~\ref{sec:approach-calibration} with the perception model $\phi$ instantiated as a bounding box predictor, mapping the observation $o_t$ to a union of bounding boxes. Formally, we represent each bounding box $j$ with the minimum and maximum coordinates in each dimension, $\big[d^{\min}, d^{\max}\big]_j$, where $d = (x,y)$ represents the spatial coordinates. The predicted occupied space is the union of 15 predicted most likely bounding boxes: $\overline\X^\text{occ} = \cup_{j=1}^{15} \big[d^{\min}, d^{\max}\big]_j$.
The parameter $q$ for this perception model is the inflation of the bounding boxes along each dimension. Therefore, in a given calibration environment $E_i$, from a given state $s$, and with a specific inflation parameter $q_i$, the predicted occupied region is defined as:
\begin{equation}
    \overline{\mathcal X}^\text{occ}_{s,i}(q_i)\coloneqq \cup_{j=1}^{15} \big[d^{\min}_{s,i}-q_i,d^{\max}_{s,i}+q_i\big]_{j}. 
\end{equation}
We collect a calibration dataset of $400$ environments as specified. In the $8$ m $\times8$ m room, we use a fixed set of $400$ sampled configurations for the sampling-based motion planner and use the same set of samples for the calibration procedure. Similarly, we collect an additional fine-tuning dataset $D_{\text{tune}}$ consisting of $100$ environments. These environments include ones with occlusions of the goal and objects in the scene. 
With an allowable misdetection rate of $\epsilon=0.15$, we obtain $\hat{q}_{0.85} = 0.75$ m for {\pwc}, $\hat{q}_{0.85} = 0.65$ m for {\pwc-fine-tuned}, and $\hat{q}_{0.85} = 0.05$ m for {CP-avg.} through calibration.  The planner replans and obtains a new sensor observation to update the filter every $0.5$ s or less (if the previous plan is already completed). 

\smallskip
\noindent\textbf{Misdetection Rate.}
We examine the misdetection rate, 
i.e., whether obstacles in the scene are classified as free space at any point during a trial, of our method, \pwc, and the baseline CP-avg., which is also calibrated using conformal prediction but without accounting for the closed-loop distribution shift. We vary the allowable misdetection bound $\epsilon$ for each method and compute the rate of misdetections in $100$ test environments. As seen in Figure~\ref{fig:missRate}, our method guarantees a misdetection rate lower than the threshold $\epsilon$ while CP-avg. violates this threshold for every $\epsilon$ considered. 

\begin{figure}[h]
    \centering
    \includegraphics[width = 0.9\linewidth]{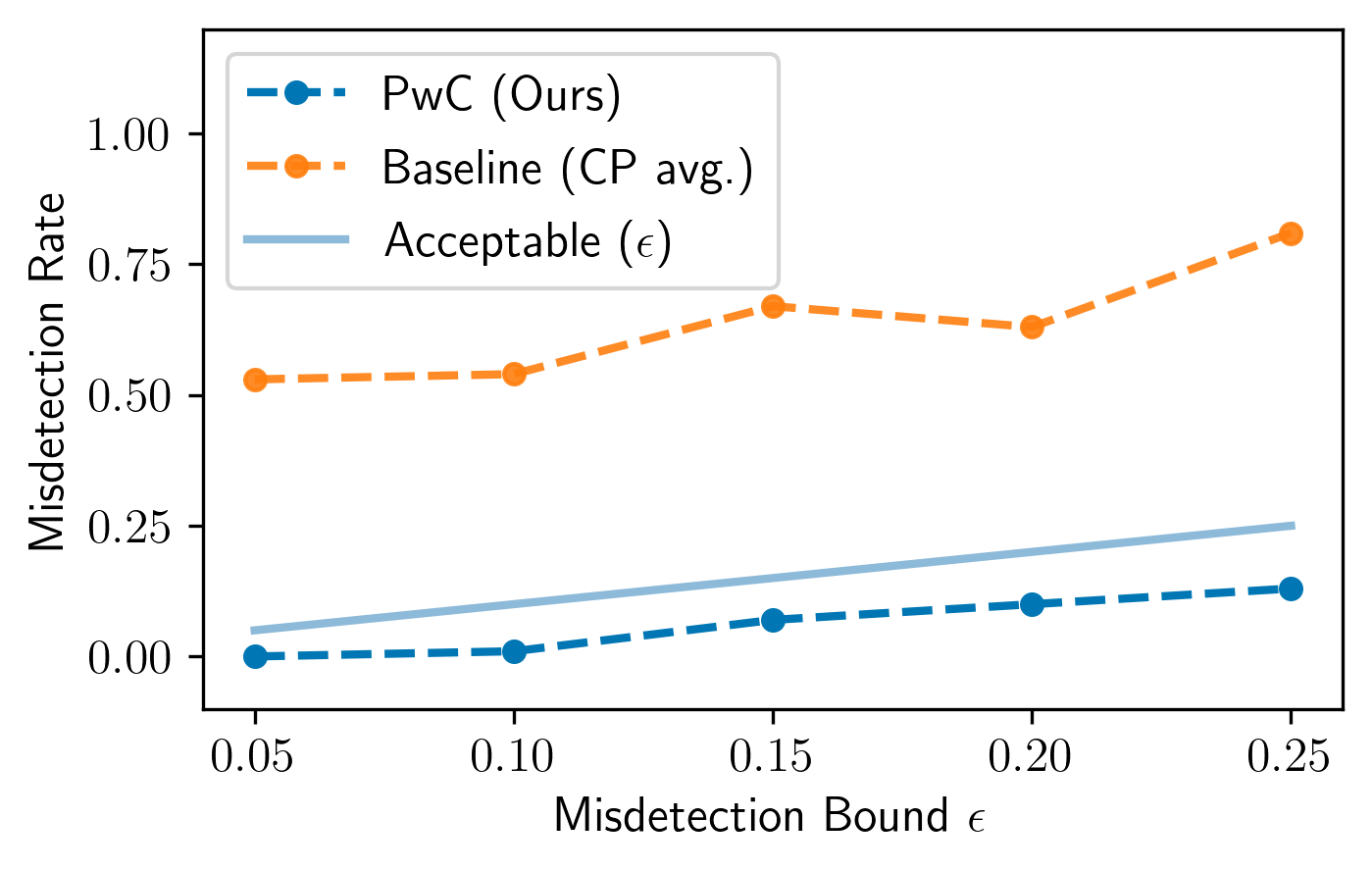}
    \caption{As we relax the confidence threshold by increasing $\epsilon$, the misdetection rate increases but remains bounded for \pwc. The baseline method has a misdetection rate much higher than acceptable.}
    \label{fig:missRate}
\end{figure}

\smallskip
\noindent\textbf{Collision Rate.}
We compare \pwc to the baselines in $100$ new environments drawn from the same distribution as calibration environments. Figure~\ref{fig:simEnv} illustrates one such test environment. Figure~\ref{fig:nonDet} shows the evolution of the free space in this environment using \pwc. Although the initial calibrated perception system outputs are inflated, the non-deterministic filter is able to expand the predicted free space in a few time steps and ensure that the robot can navigate without unnecessary conservatism, while guaranteeing safety. The results are summarized in Figure~\ref{fig:all-bars}. 
We observe that our proposed approaches, \pwc and \pwc-fine-tuned, have no collisions in any environments. While the robot reaches the goal in a slightly lower percentage of environments compared to baselines, we emphasize that ours is the only approach that is able to ensure a low,  statistically guaranteed misdetection rate across test environments. 

\smallskip
\noindent\textbf{Ablations.}
To further illustrate the effect of misdetections on safety, we consider a different distribution of environments wherein we randomly place a \emph{single} chair in the straight line path between the initial position of the robot and the goal. 
For a safety threshold $1-\epsilon = 0.85$, we compare $\pwc$, CP-avg, and 3DETR.
The results are provided in Figure~\ref{fig:small-bars} for $100$ new test environments, wherein the goal is reached if the robot navigates to within $2$ m of the goal. In these environments, the desired safety rate is not met by the baselines while our approach is still statistically guaranteed to be safe.

\begin{figure}[h]
    \centering
    \includegraphics[width=0.9\linewidth]{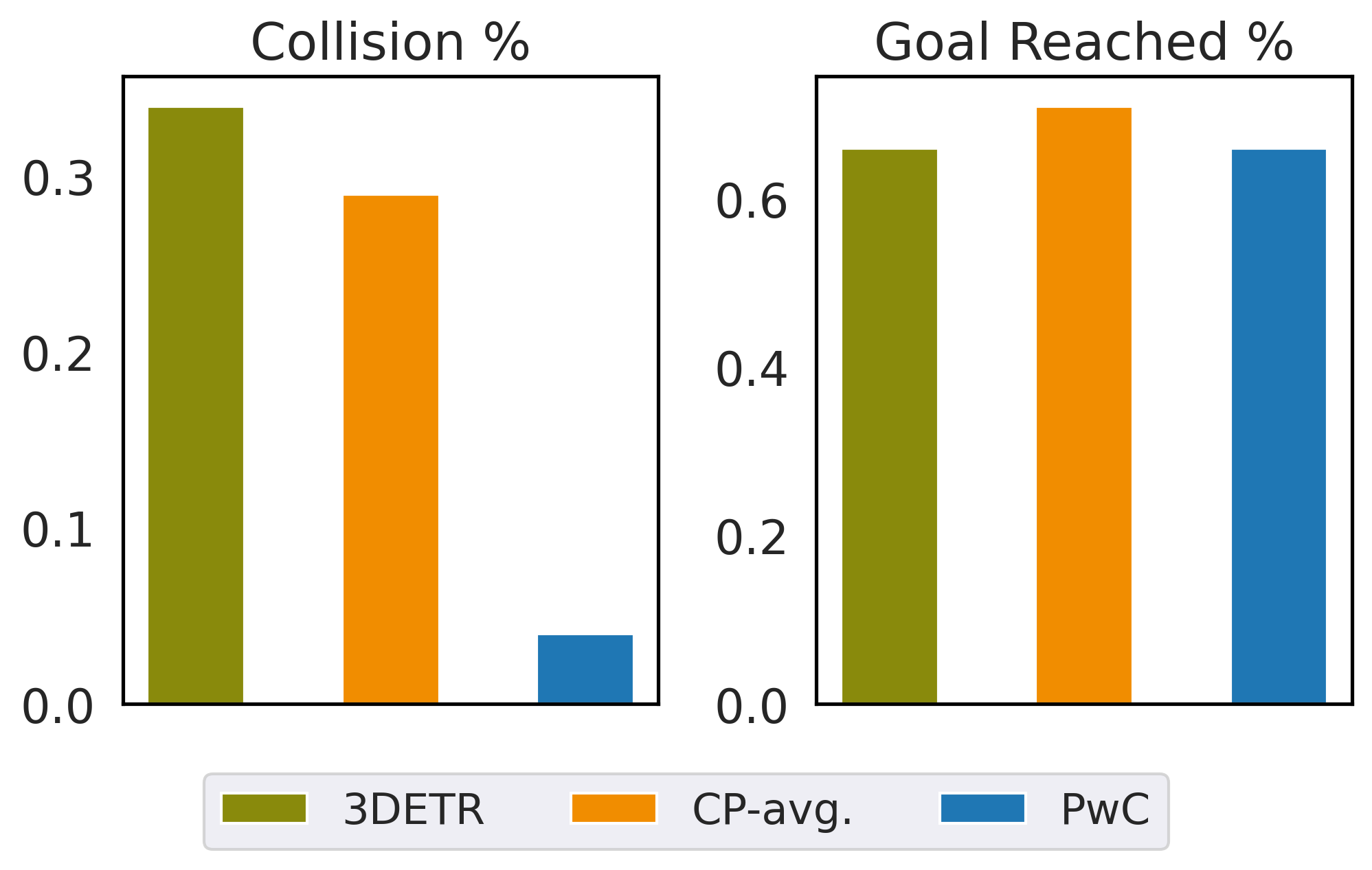}
    \caption{A comparison between collision rates of different perception systems that use the same planner.}
    \label{fig:small-bars}
\end{figure}

We provide additional simulation results in Appendix~\ref{appendix:kl-divergence-results} that illustrate: 1) the effects of closed-loop distribution shifts on safety wherein \pwc is robust to an increase in the level of closed-loop distribution shift, while the baseline CP-avg. is not, leading to higher collision rates; 2) the tradeoff in different partition sizes for fine-tuning using split-CP; 3) the effect of varying the allowable safety rate $\epsilon$; 4) the effect of varying the number of sampled configurations; and 5) comparing \pwc to a method of heuristically inflating bounding boxes.

\begin{figure*}[t]
    \centering
    \includegraphics[width=\linewidth]{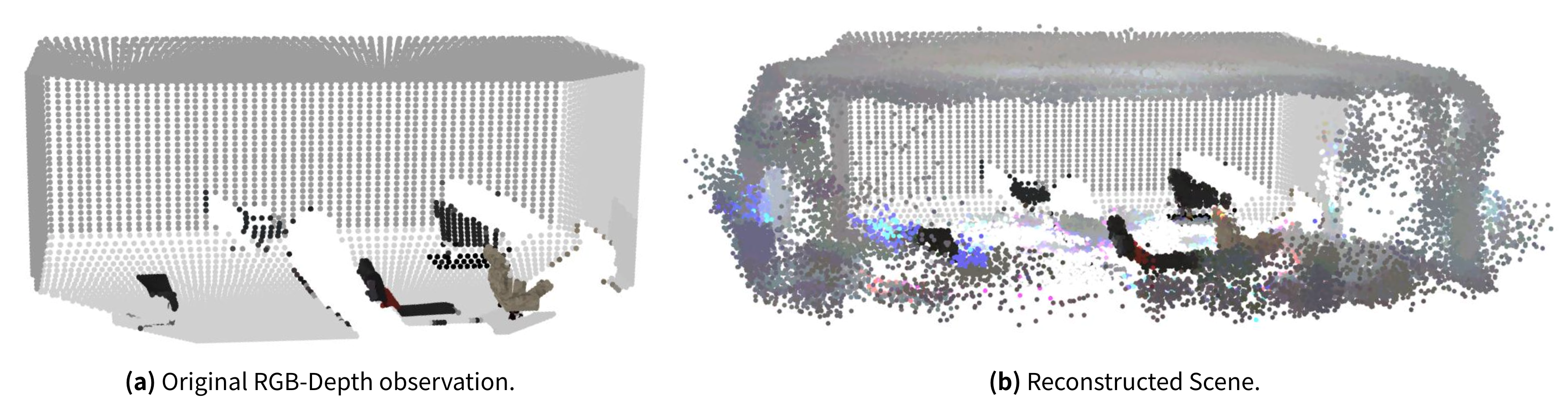}
    \caption{NU-MCC scene completion with 3D reconstruction.}
    \label{fig:numcc-scene-completion}
\end{figure*}
    
\subsection{Occupancy Predictors}
\label{sec:numcc}
Now, we consider perception systems that predict occupancy maps.
Specifically, we demonstrate the framework with the scene completion model NU-MCC~\citep{lionar2023nu}.
NU-MCC takes an RGB-Depth image as input and predicts the value of the unsigned distance function (UDF) of each point in space. The original NU-MCC model includes a 3D reconstruction phase (Figure~\ref{fig:numcc-scene-completion}), where points with UDF less than a threshold, $q$, are kept and shifted to the surface of objects. Although this procedure results in better 3D visualizations, it breaks the correspondence between the threshold $q$ and spatial coverage. In our method, we only use the predicted UDF from NU-MCC. In addition, we exponentially scale all the predicted UDFs to achieve a more uniformly covered range of UDF values.

\smallskip

\smallskip
\noindent\textbf{Calibration and Planning.} 
To define the non-conformity score, we find the smallest threshold $q_i$ in each environment $E_i$ such that the predicted occupancy covers the ground truth occupancy. Formally, the observation acquired at time step $t$ is denoted $o_t$, represented by an RGB-depth image. The perception model, denoted $\phi$, is the combination of (i) the adapted NU-MCC model which predicts UDF for each point in space, (ii) filtering out points in space with $\text{UDF}>q$, and (iii) projecting the resulting pointcloud onto a 2D occupancy grid as a bird’s eye-view. This perception pipeline is summarized in Figure~\ref{fig:numcc-pipeline}. Thus, $\phi$ maps $o_t$ to the occupancy representation of the environment, $\mathcal Z$. In our implementation, $\mathcal Z$ is an $n\times n$ grid with boolean entries, with $1$ representing occupied and $0$ otherwise. We use $P\in\mathcal Z$ to denote one point on the grid in $\mathcal Z$.
\begin{figure*}[tb]
    \centering
    \includegraphics[width=\linewidth]{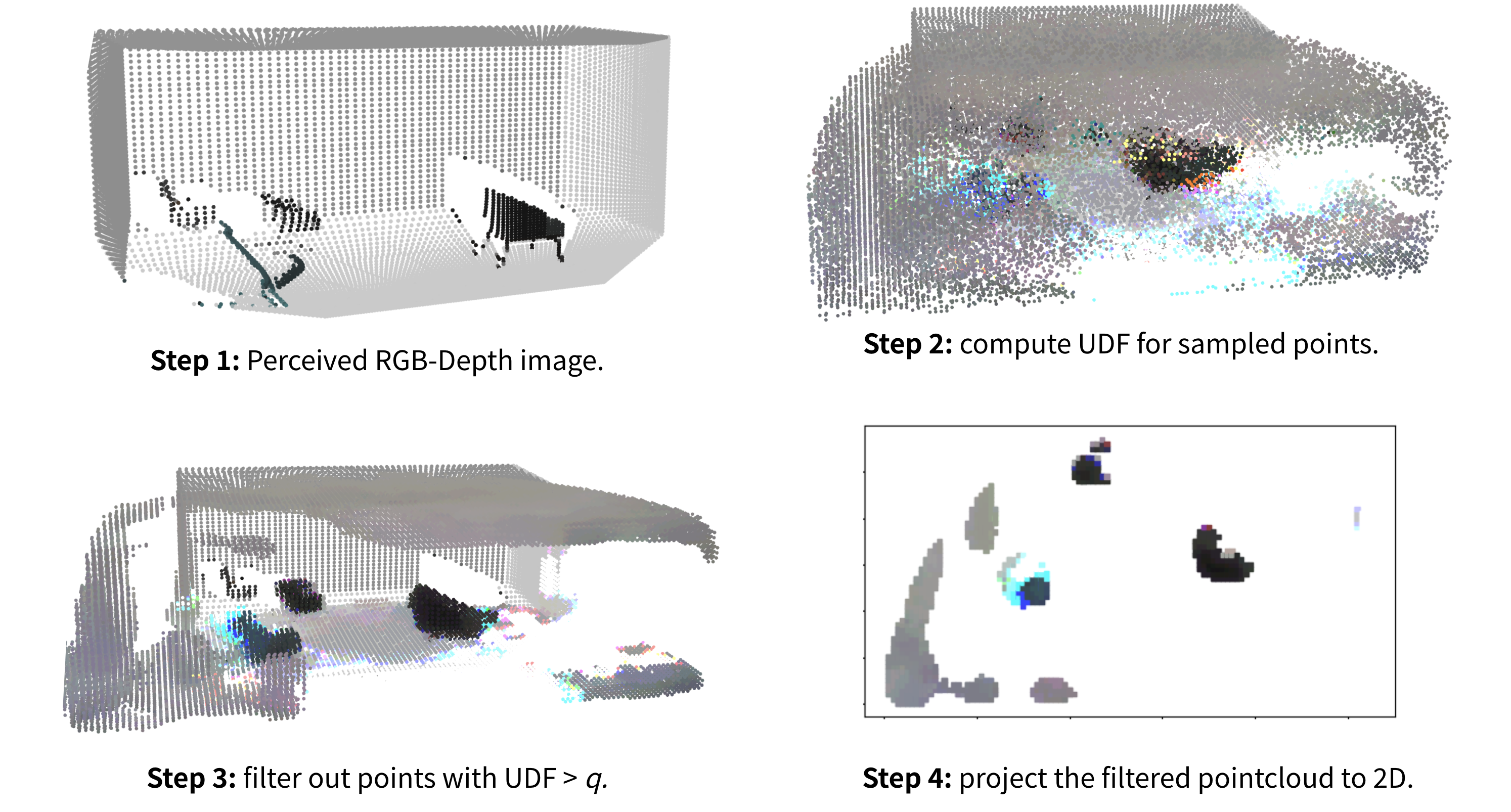}
    \caption{Perception model $\phi$ based on NU-MCC, showing the construction of calibrated occupancy maps with provable guarantees on the correctness of the resulting map.}
    \label{fig:numcc-pipeline}
\end{figure*}
In a given calibration environment $E_i$, from a given state $s$, and a specific threshold $q_i$, the predicted occupancy is defined as: 
\begin{equation}
    \overline{\mathcal X}^\text{occ}_{s,i}(q_i)\coloneqq \{P\in \mathcal Z \mid \text{UDF}(P)\leq q_i\}.
\end{equation}
As $q_i$ increases, $\overline{\mathcal X}^\text{occ}_{s,i}(q_i)$ expands monotonically, satisfying the property stated in Section~\ref{sec:approach-calibration}. Therefore, proposition~\ref{prop:calibration} follows, guaranteeing that the calibrated system $\tilde \phi$ predicts occupancies that cover the ground truth with high probability.

We generate a calibration dataset consisting of 300 environments from the distribution described in Section~\ref{sec:simulations}. In addition,  the chairs are randomly rotated about the $z$-axis (Figure~\ref{fig:bbox_numcc_traj}, middle and right). We use a fixed set of $812$ sampled configurations, same as the set used by the sampling-based planner described in Section~\ref{sec:approach-planning}, modified for the occupancy map setting. With an allowable misdetection rate of $\epsilon=0.15$, we obtain $\hat q_{0.85}$ for \pwc-NU-MCC and NU-MCC-CP-avg through calibration, and use the (exponentiated) default for NU-MCC.
For \pwc calibrated on  task distribution with rotated chairs and more states, we obtain $\hat q_{0.85} = 1.10$ m. Note that the $\hat q_{1-\epsilon}$ for \pwc stands for the bounding box inflation rather than the UDF threshold.

\smallskip
\noindent\textbf{Results.} Figure~\ref{fig:numcc-sim-rot} summarizes the simulation results. We compare our method based on occupancy prediction (\pwc-NU-MCC) against the non-calibrated version (NU-MCC), as well as the method applying conformal prediction without accounting for closed-loop distribution shift (NU-MCC-CP-avg). We also compare against our method based on bounding box predictors, as described in Section~\ref{sec:approach-calibration} (\pwc). We use the same metrics as described in Section~\ref{sec:simulations}. 

For the results shown in Figure~\ref{fig:numcc-sim-rot}, we use a test dataset of 100 environments from the same distribution as the calibration dataset. The rotated chairs are no longer axis-aligned on the $xy$-plane, causing unnecessary conservatism when using the bounding box representation. Indeed, Figure~\ref{fig:numcc-sim-rot} shows that \pwc-NU-MCC has a much higher success rate (40\% improvement) and shorter path length compared to \pwc, while the safety rate is maintained. The four methods in the figure are arranged in the order of least to most conservative from left to right, showing a significant drop in collision rate and mis-detection rate with our methods, which also fall within the guarantee of less than $15\%$.
\begin{figure}[h]
    \centering
    \includegraphics[width=\linewidth]{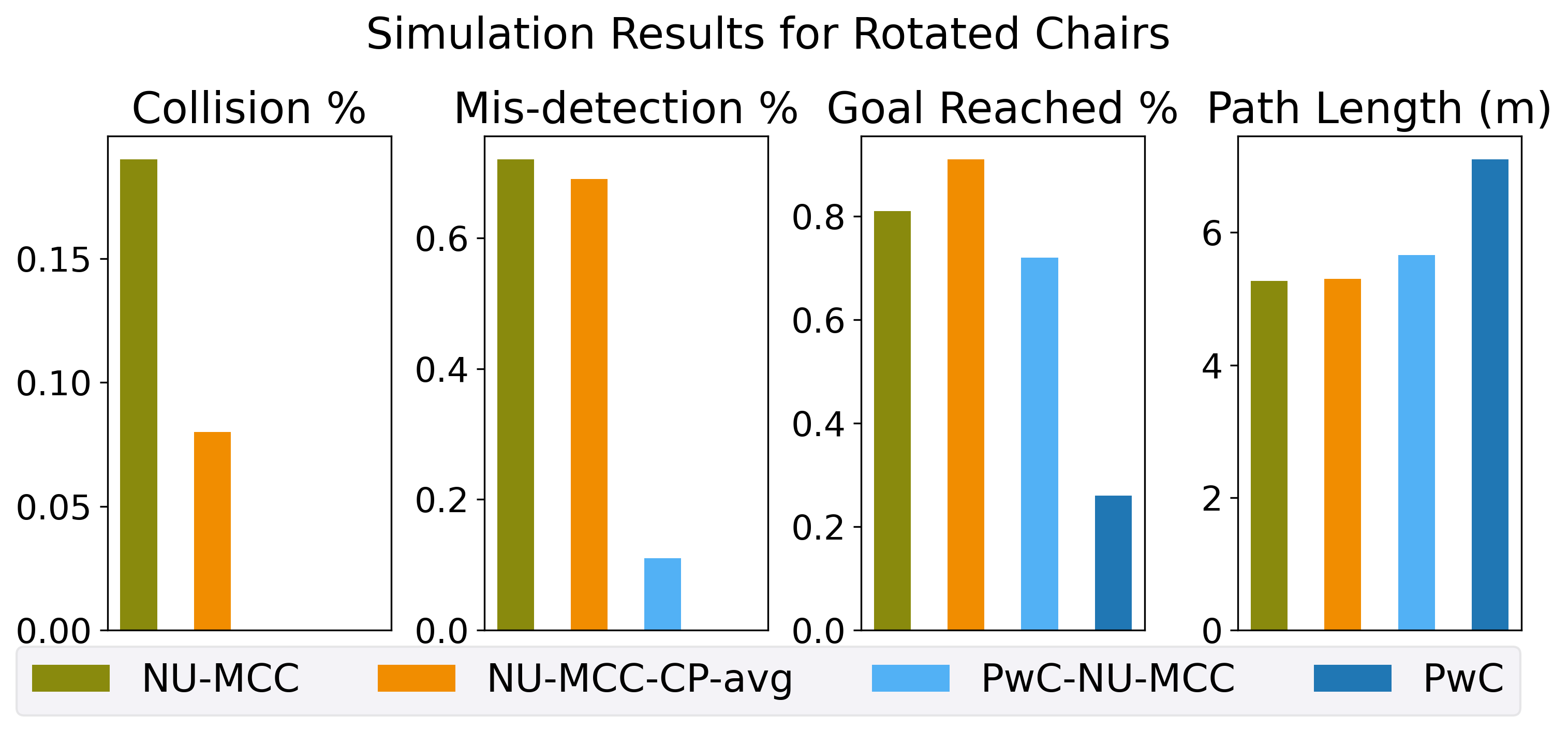}
    \caption{Results for the simulated experiments with occupancy predictors, across 100 environments with rotated chairs.}
    \label{fig:numcc-sim-rot}
\end{figure}

Figure~\ref{fig:bbox_numcc_traj} shows the trajectory of the robot in the same simulation environment, using three different perception modules. The left plot shows \pwc as described in Section~\ref{sec:exps-calibration}, while the middle and right plots show \pwc-NU-MCC and NU-MCC respectively. For \pwc, the bounding boxes are unnecessarily inflated, causing the robot to get stuck in the overly conservative estimate of free space. \pwc-NU-MCC preserves the safety guarantee while characterizing the true free space much more accurately, reaching the goal safely. NU-MCC overestimates the free space and collides with the obstacle.

\begin{figure*}[tb]
    \centering
    \includegraphics[width=\linewidth]{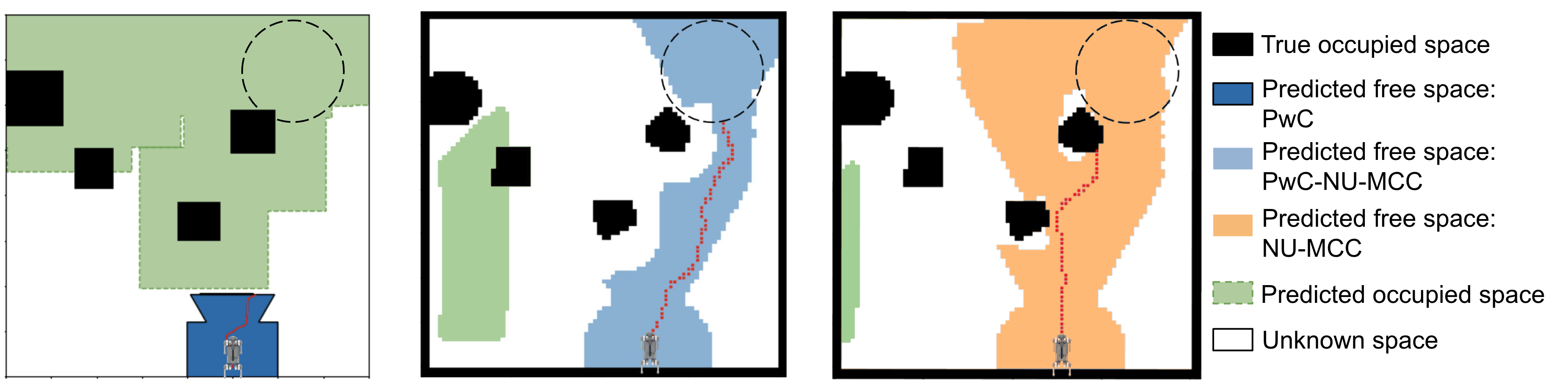}
    \caption{Comparison of trajectories in the same environment using three different perception systems: \textbf{(Left)} \pwc with 3DETR and bounding box representations, \textbf{(Middle)} \pwc-NU-MCC, \textbf{(Right)} NU-MCC. The robot marks the start position, and the dashed circle represents the goal area. The blue region shows predicted free space, and the black regions represent the ground truth occupied space, either as bounding boxes or as occupancy grids. The robot's trajectory is marked in red. Among all methods, only \pwc-NU-MCC enables the robot to safely reach the goal.}
    \label{fig:bbox_numcc_traj}
\end{figure*}
\section{Hardware Experiments}\label{sec:hardware}
Now, we validate the end-to-end statistical safety assurance of our approach on a quadruped robot in the task of vision-based navigation with two sets of experiments. As in our simulation setup in Section~\ref{sec:simulations}, the robot is tasked with navigating to a goal location while avoiding different chairs placed in varying configurations across an $8$ m $\times 8$ m room. We conduct two sets of experiments, which we term ``nominal'' and ``fast''. In the nominal experiments, the robot navigates with an average forward speed of $0.4$ m/s, whereas in the fast experiments, we speed up the robot to $1.5$ m/s. In both sets of experiments, we utilize the perception system calibrated in simulation with a guaranteed safety rate of $1-\epsilon = 0.85$, as described in~\Cref{sec:exps-calibration}. Our calibration in simulation environments with realistic and diverse environments ensures that the performance of the perception system remains similar in its simulation and hardware implementations.
We compare our \pwc method against CP-avg. (defined in Section \ref{sec:simulations}). We run the nominal experiments across 30 different physical environments (60 trials total) and run the fast experiments across 15 environments (30 trials total). \addendum{One challenge is to ensure a minimal sim-to-real gap for perception. In order to address this, we utilize depth measurements as the robot's sensory input. This choice facilitates a small sim-to-real gap, as observed in prior work \citep{loquercio2021learning, gervet2023navigating}}. 

\subsection{Experiment Setup}
We represent the robot's state as $s_t = [x, y, v_x, v_y]^T$ where $x$ and $y$ are its position in the environment and $v_x$ and $v_y$ are the respective velocities (See Figure~\ref{fig:hardware_res} for the coordinate system). For each trial, the robot is initialized around position $[4, 0]$ m (with the origin set to the bottom left corner of the room) and has 60 seconds to reach the goal. For the nominal experiments, the robot replans every second in a receding horizon manner using the safe planner described in Section~\ref{sec:approach-planning}. The goals are varied every 10 environments and include positions $[2, 7]$ m, $[4, 7]$ m, and $[6, 7]$ m, with a radius of $1$ m. For the fast experiments, the robot replans every $0.8$ s. The goal is set at $[6, 7]$ m, and the radius is increased to $1.5$~m. 

\smallskip
\noindent\textbf{Hardware.} We use the Unitree Go1 quadruped robot with fully onboard sensing and computation. The robot is equipped with a ZED 2i RGB-D camera and a ZED Box computer attached to the base of the robot as shown in the top row of Figure~\ref{fig:hardware_res}. The Zed 2i provides the Go1 with point cloud observations with a $70^\circ$ field of view  and a visibility range of [1, 5]m. The Zed 2i also uses vision-inertial odometry to provide accurate positional state estimates in the environment. The Zed Box includes an 8-core ARM processor and a 16GB Orin NX GPU. This allows us to process the point cloud observations in order to produce bounding boxes using the pre-trained 3DETR model \citep{misra2021-3detr}. The bounding boxes are aggregated over time to update the estimated free, occupied, and unknown spaces as described in Section~\ref{sec:approach-planning}. The safe planner described in Section~\ref{sec:approach-planning} is used to output Cartesian velocity commands bounded at a speed of 0.8m/s; these commands are sent from the Zed Box over UDP to the Go1's processor. \addendum{Our method is implemented in real-time on the Zed Box hardware with replanning every $0.5$ seconds of which the non-deterministic filter takes $0.00025$ seconds to run}. The dynamics of the Go1 are estimated using MATLAB's System Identification Toolbox \citep{MATLAB} and are provided in Appendix \ref{appendix:sysid}. 

\smallskip
\noindent\textbf{Environments.} For both sets of experiments, we test the robot in different environments, consisting of various chair configurations and geometries in an $8$ m $\times8$ m room. Configurations range from random, occluded goal, occluded chairs, clustered chairs, and narrow paths (approximately $1.8$ m in width leaving $0.4$ m of available free space for \pwc to find). For the nominal experiments, each environment has between 4 and 8 chairs present. See Appendix \ref{appendix:chairs} and \ref{appendix:envs} for the unseen chairs used in testing and the environment configurations respectively. We use a Vicon motion capture system to log the ground-truth placement and bounding boxes of the chairs for each environment. For the fast experiments, each environment has 6 chairs present. Since recording ground-truth data introduces latency that prevents the robot from reaching its target velocity of $1.5$ m/s, we report only collision and goal-reach rates for this set of experiments.

\begin{figure*}[t]
    \begin{minipage}{\linewidth}
        \centering\captionsetup[subfigure]{justification=justified}
        \begin{minipage}{0.32\linewidth}
            \centering
            \includegraphics[width=\linewidth]{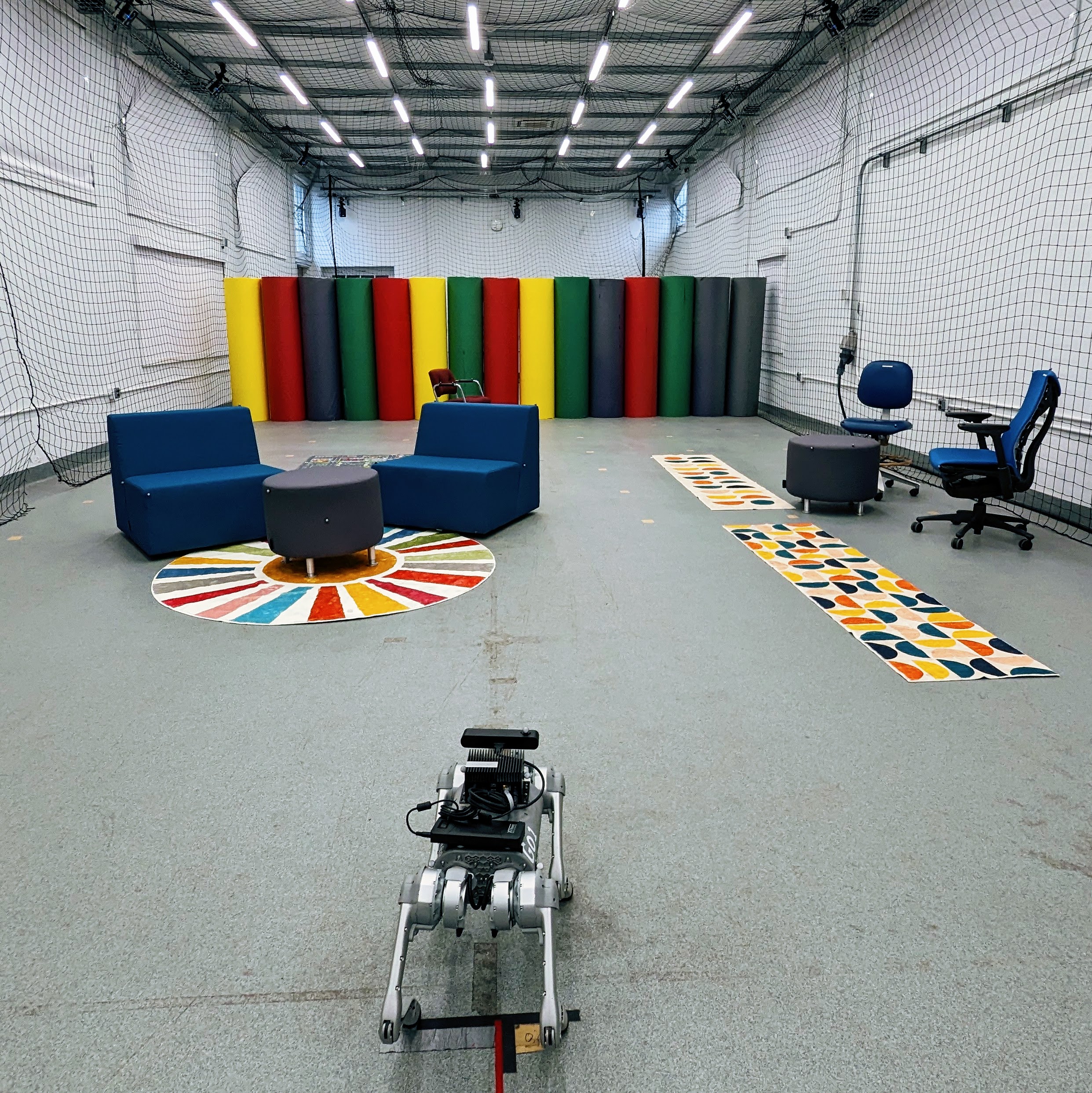}
            \label{fig:env1_layout}
        \end{minipage}
        \hfill
        \begin{minipage}{0.32\linewidth}
            \centering
            \includegraphics[width=\linewidth]{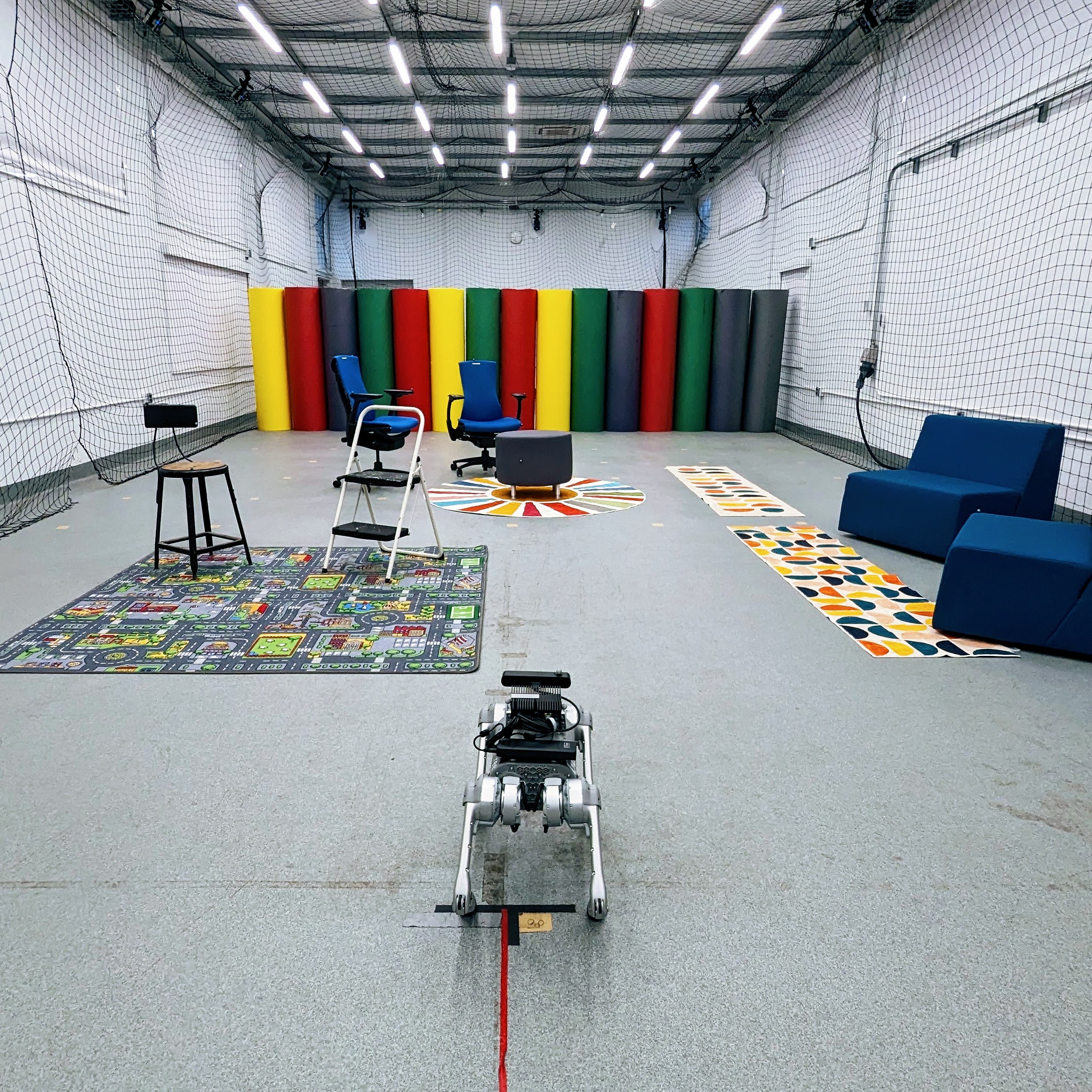}
            \label{fig:env4_layout}
        \end{minipage}
        \hfill
        \begin{minipage}{0.32\linewidth}
            \centering
            \includegraphics[width=\linewidth]{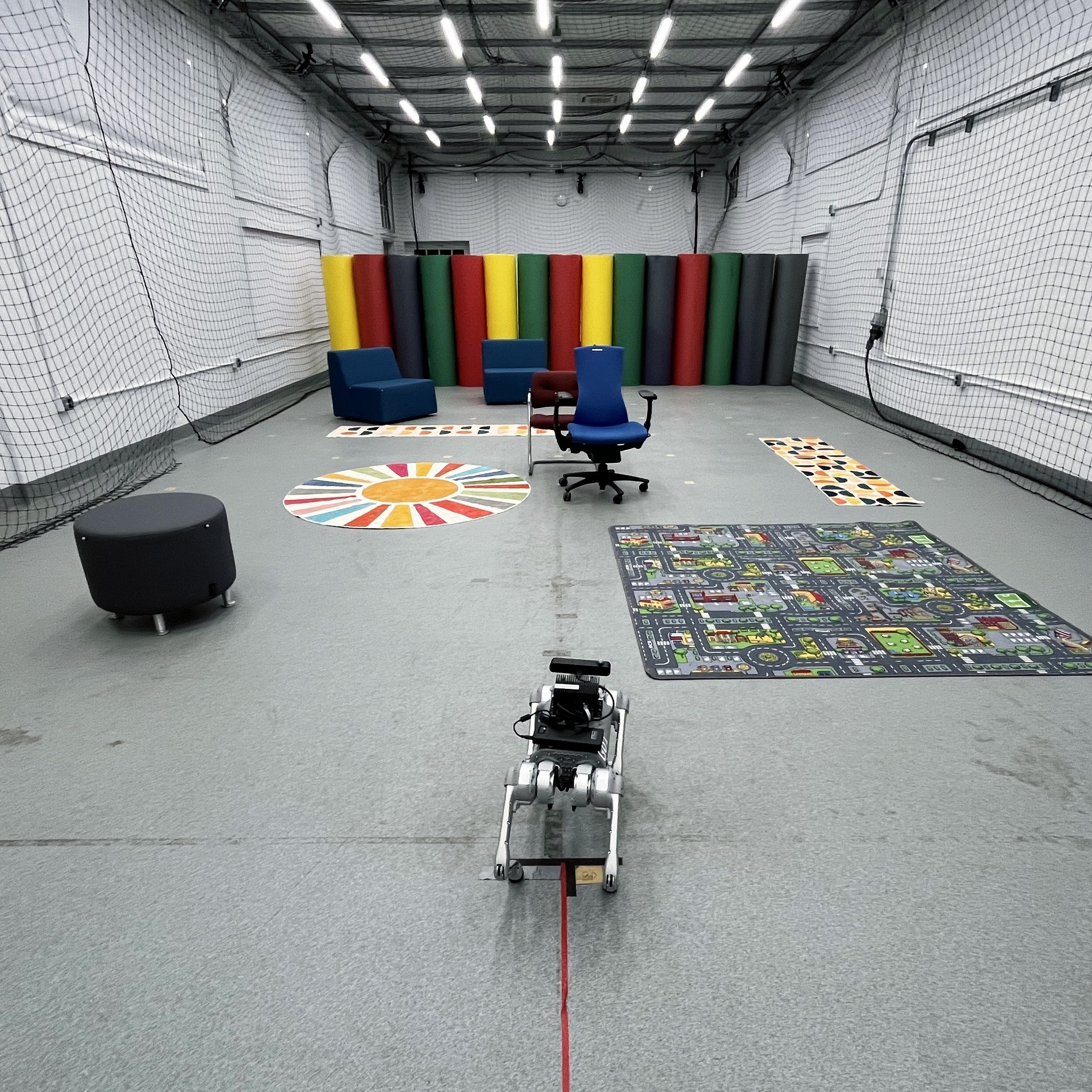}
            \label{fig:env14_layout}
        \end{minipage}
    \end{minipage}
    \vspace{5pt} 

    \begin{minipage}{\linewidth}
        \centering\captionsetup[subfigure]{justification=justified}
        \begin{minipage}{0.331\linewidth}
            \centering
            \includegraphics[width=\linewidth]{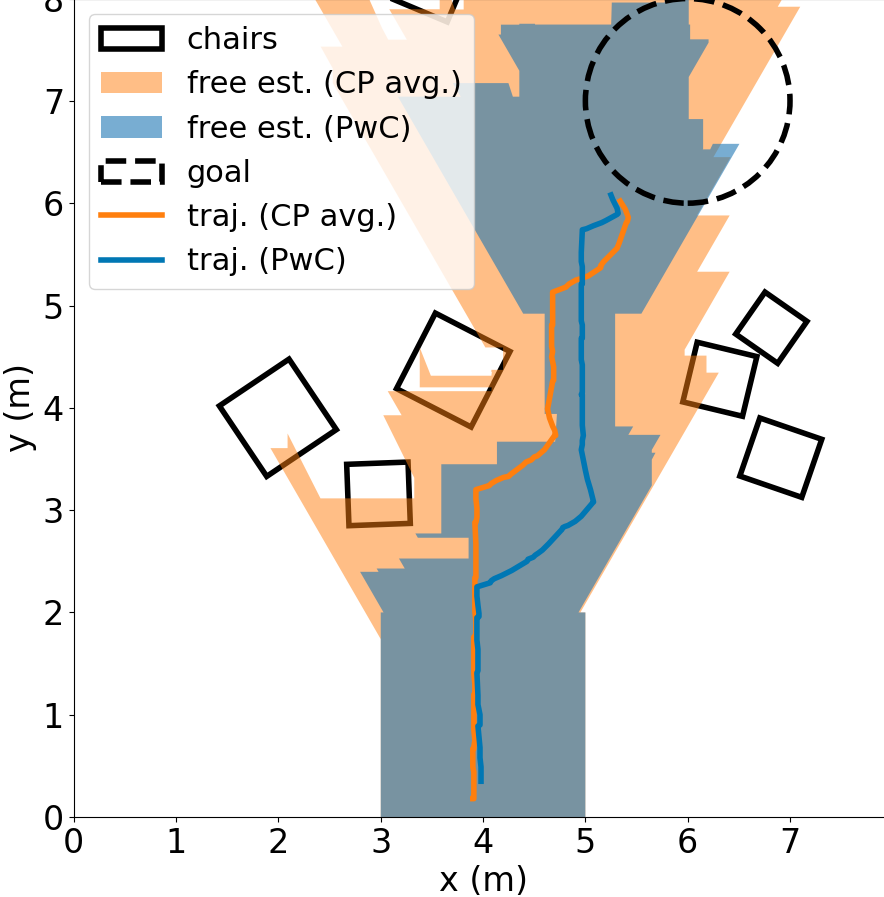}
            \label{fig:env1_traj}
        \end{minipage}
        \hfill
        \begin{minipage}{0.3\linewidth}
            \centering
            \includegraphics[width=\linewidth]{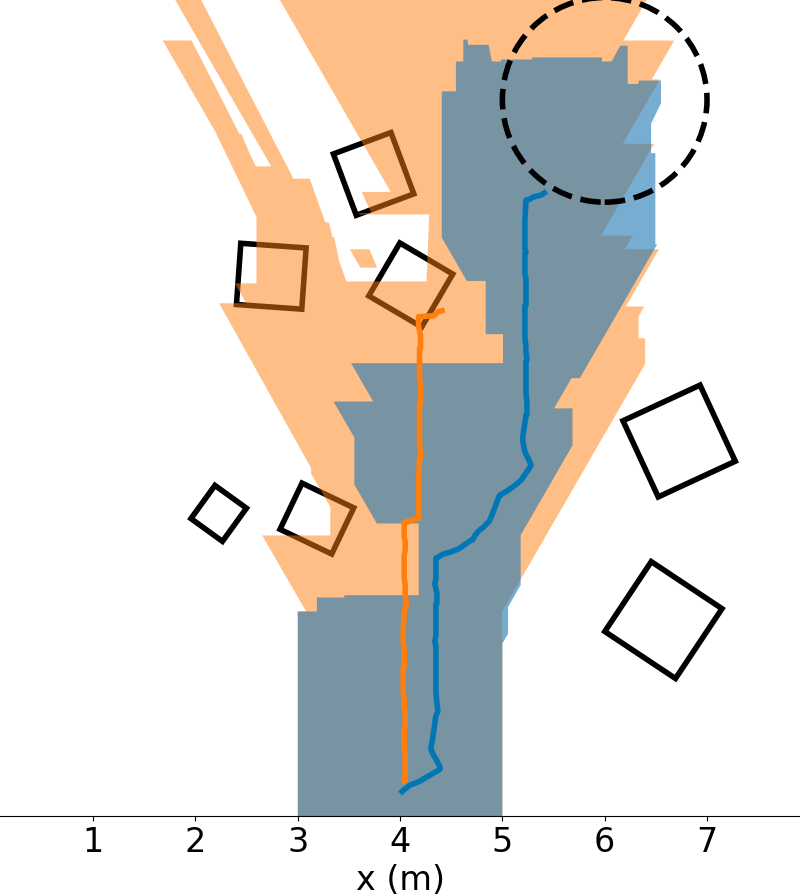}
            \label{fig:env4_traj}
        \end{minipage}
        \hfill
        \begin{minipage}{0.3\linewidth}
            \centering
            \includegraphics[width=\linewidth]{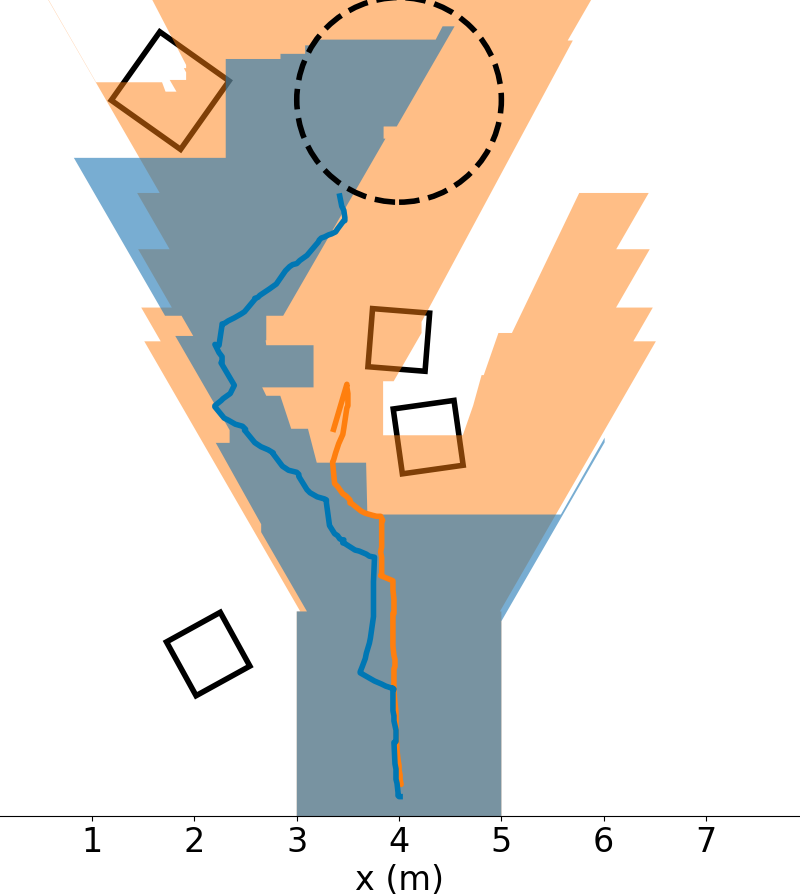}
            \label{fig:env14_traj}
        \end{minipage}
    \end{minipage}

    \caption{Hardware trial results. \textbf{(Top)} The physical layouts of three example environments. 
    \textbf{(Bottom)} The robot trajectories performed in these environments. Estimated free space is shaded, and robot trajectories are represented by solid lines: our method in blue and the baseline in orange. 
    Our method (\pwc) successfully navigates to the goal through challenging areas, whereas the baseline misdetects free spaces, leading to collisions in some cases.}
    \label{fig:hardware_res}
\end{figure*}

\subsection{Results}
\label{sec:fast-robot}
For \pwc, we used the ${\hat q_{0.85} = 0.73}$ m threshold found in simulation to inflate the predicted bounding boxes returned from 3DETR in order to achieve 85\% confidence that our robot will remain safe in new environments. We summarize key statistics of \pwc compared to CP-avg. (${\hat q_{0.85} = 0.02}$ m) across 30 different environments in Figure~\ref{fig:all-bars} (right). Importantly, our trials demonstrate that our confidence bound holds on hardware in real environments and without being too conservative. \pwc was safe through 90\% of the trials and also had comparable path length to the baseline. Meanwhile, the baseline struggled in the real environments by having misdetections in each trial and colliding with a chair in half of the trials. See Figure~\ref{fig:hardware_res} for trajectories and free space estimations through several environments with narrow spaces, occluded chairs, and occluded goals. The supplementary video contains full example trials. 

\pwc's low misdetection rate and higher success rate in these trials  emphasize the efficacy of the bounding box inflation provided by CP paired with the non-deterministic filter. This principled pairing inflates the (potentially poor) bounding box detections to properly capture obstacles but quickly shrinks the occupied space with the filter such that the robot can still navigate effectively. 

For faster navigation, we employ two complementary strategies: minimizing idle planning time via concurrent planning and execution, and increasing the robot’s velocity in the pre-sampled configuration space. 
First, we implement a concurrent planning and execution framework using threading. The planning process is divided into two stages: an initial policy computed from the starting state and a continually updated future policy computed from the robot's predicted future state. At the start of each trial, the robot calculates an initial policy to reach the goal. During execution, a separate thread uses the robot's predicted future state, the state at the end of the current policy, to concurrently generate the next policy phase. This synchronization of execution and planning minimizes idle planning time.
To support the increased speed, we re-sample the configuration space by keeping the positions unchanged but scaling up the speed, so that the calibration results would still hold. The re-sampled states, along with pre-computed reachability sets using the robot dynamics, are used to generate planned trajectories at high speed.

As a result, we increase the robot's average forward speed from about $0.4$ m/s to $1.5$ m/s and reduce the average task completion time in similar environments from about $28$ seconds to $8$ seconds.
These performance gains are achieved without significant compromise in the safety rate of hardware validations of our previous method. We present the accelerated hardware results across 15 new environments (different from those shown in Appendix~\ref{appendix:envs}) and compare the collision rate and success rate against the CP-avg. baseline, as shown in Table~\ref{tab:result_fastpwc}.

\begin{table}[h!]
\centering
\caption{Results for accelerated hardware experiments with \pwc and CP-avg.}
\begin{tabular}{ccc}
\hline
 \textbf{Method}& \textbf{Collision} & \textbf{Goal Reached}\\ 
 \hline\hline
Accelerated \pwc& 20\% & 53.3\% \\ 
\hline
CP-avg.& 66.7\% & 33.3\% \\ 
\hline
\end{tabular}
\label{tab:result_fastpwc}
\end{table}

\section{Discussion and Conclusions}\label{sec:conclusion}

We present a modular framework, \pwc, for rigorously quantifying the uncertainty of a pre-trained perception model in order to provide an end-to-end statistical safety assurance for perception-based navigation tasks. Notably, our statistical assurance holds for generalization to new environmental factors (e.g., new obstacle geometries and configurations) and allows for the distribution shift of states that may occur during closed-loop deployment of the perception system with the planner. 
Additionally, we address the conservatism introduced by the inflation of bounding boxes, by applying \pwc to occupancy predictors and achieving much better performance without sacrificing the safety assurance.
We validate the theoretical safety assurances provided by \pwc with our simulation and hardware experiments, demonstrating significant empirical improvements in safety compared to baseline approaches that do not consider closed-loop distribution shift. 

\smallskip
\noindent\textbf{Limitations and Future Work.} One limitation of our work is the assumption of static obstacles. As a future direction, we are interested in quantifying uncertainty in both the state of agents moving in the environment and predictions of their \emph{semantic labels} (e.g., ``pedestrian" vs. ``bicyclist"), and utilizing game-theoretic planning techniques that account for the uncertainty in the agents' current state and future motion. 
Additionally, while our definition of safety is limited to collision avoidance, we are interested in extending it to richer settings such as navigation on limited surface area.
Lastly, we are interested in uncertainty quantification for perception models that support tasks beyond point-to-point navigation, e.g., calibrating the outputs of multi-modal foundation models for language-instructed navigation where we ensure accurate detection of target objects as well as semantically unsafe regions~\citep{santos2024updating}. 
We expect that rigorous uncertainty quantification is a necessary step towards fully leveraging the power of large foundation models \citep{firoozi2023foundation} while safely integrating them into future robotic systems. 

\begin{acks}
The authors were partially supported by the Toyota Research Institute (TRI), the NSF CAREER Award [\#2044149], the Office of Naval Research [N00014-23-1-2148], and the Princeton SEAS Innovation Award from The Addy Fund for Excellence in Engineering. This article solely reflects the opinions and conclusions of its authors and not NSF, ONR, TRI or any other Toyota entity.
\end{acks}

\vspace{4pt}
\noindent\textbf{Statements and declarations}

\noindent Not applicable.

\vspace{4pt}
\noindent\textbf{Ethical considerations}

\noindent Not applicable.

\vspace{4pt}
\noindent\textbf{Consent to participate}

\noindent Not applicable.

\vspace{4pt}
\noindent\textbf{Consent for publication}

\noindent Not applicable.

\vspace{4pt}
\noindent\textbf{Declaration of conflicting interest}

\noindent The author(s) declared no potential conflicts of interest with respect to the research, authorship, and/or publication of this article.

\vspace{4pt}
\noindent\textbf{Funding statement}

\noindent The authors were partially supported by the Toyota Research Institute (TRI), the NSF CAREER Award [\#2044149], the Office of Naval Research [N00014-23-1-2148], and the Princeton SEAS Innovation Award from The Addy Fund for Excellence in Engineering. This article solely reflects the opinions and conclusions of its authors and not NSF, ONR, TRI or any other Toyota entity.

\bibliographystyle{SageH}
\balance
\bibliography{references}

\clearpage

\begin{appendices}
\section{Perception and Planning Extensions}
\label{appendix:extensions}

In this section, we outline a few extensions to the basic technical approach described in Sections~\ref{sec:approach-calibration} and \ref{sec:approach-planning}: (i) fine-tuning a pre-trained perception model and (ii) incorporating sensor and dynamics uncertainty.

\subsection{Fine-Tuning a Pre-Trained Perception Model}
\label{sec:fine-tuning}

In Section~\ref{sec:approach-calibration}, we assumed access to a pre-trained perception model $\phi$ that outputs occupancy predictions of the environment. The conformal prediction-based uncertainty quantification procedure then uses the calibration dataset $D = \{E_1, \dots, E_N\}$ of environments to produce a calibrated perception system $\tilde{\phi}$, which lightly processes the outputs of $\phi$ scaling with a parameter $q$. In practice, it may also be useful to \emph{fine-tune} $\phi$ for our target deployment environments before performing uncertainty quantification. 

This can be achieved using \emph{split conformal prediction}~\citep{angelopoulos_gentle_2022}, where one splits the overall dataset $D$ into $D = D_\text{tune} \cup D_\text{cal}$. If the perception model takes the form of a neural network $\phi_w$ parameterized by weights $w$, we can use $D_\text{tune}$ to fine-tune $w$ (or the weights of a residual network). We can then utilize $D_\text{cal}$ in order to perform the CP-based calibration as described in Section~\ref{sec:approach-calibration}. We demonstrate the fine-tuning process for the case of bounding box predictions in Section~\ref{sec:simulations}, and show that this additional fine-tuning step before calibration can reduce the conservatism of outputs and improve end-to-end success rates. 

The typical choice of loss function for training a bounding box predictor is the \emph{generalized intersection-over-union (gIoU) loss} \citep{Rezatofighi19}, a differentiable version of the IoU loss: given a ground-truth bounding box $A$ and a predicted box $B$, one computes ${L(A,B) := |A \cap B| / |A \cup B|}$. However, while this loss is popular in computer vision, it is not suitable for robot navigation. In particular, the IoU loss is \emph{symmetric}: it does not distinguish between the ground-truth and predicted bounding box and thus does not encourage the predicted box to \emph{contain} the ground-truth box. We propose a modification to the gIoU loss in Appendix~\ref{appendix:fine-tuning loss}, which encourages that the predicted bounding box encloses the ground-truth box while also ensuring that the predicted box is not too large. Similar to the gIoU loss, this loss is (almost-everywhere) differentiable and scale invariant. We utilize this loss for fine-tuning in our experiments (Section~\ref{sec:simulations}). {However, one could use any other method for finetuning not limited to training a simple neural network with the gIoU loss~\citep{neklyudov2018variance}.}

\subsection{Sensor Errors and Dynamics Uncertainty}

In Section~\ref{sec:problem formulation}, we modeled the robot's sensor as a deterministic mapping $\sigma: \S \times \E \rightarrow \O$, which provides observations from a particular state in a given environment. This formulation allows us to also incorporate sensor errors. Specifically, any errors or randomness in the sensor can be formally included as part of the environment $E \in \E$. Thus, in addition to sampling environmental variables such as obstacle locations, geometries, etc., each environment $E$ also samples random variables that prescribe sensor errors from each state $s \in \S$ in the environment. This way of modeling sensor errors allows: (i) $\sigma$ to be deterministic (since all sources of randomness are included in $E$), (ii) the sensor errors to be dependent on the relative pose of the robot relative to obstacles (e.g., modeling the fact that depth estimates are often further from ground-truth depth values as distance increases), and (iii) the modeling of correlations in sensor errors from different locations (e.g., capturing the fact that sensor errors from nearby robot locations can be highly correlated). Modeling time-varying sensor errors (i.e., different sensor errors from the robot state at different times) is not as immediate, but could potentially be incorporated by augmenting the state space $\S$ to include the time-step. 

In addition to errors in sensing, one can also account for uncertainty in the dynamics of the robot by using a robust planner (see~\citep{hsu_safety_2023} for an overview). In the experiments described in Section~\ref{sec:hardware}, we incorporate uncertainty by generating plans that prevent the robot from entering the inevitable collision set (cf. Section~\ref{sec:approach-planning}) even with bounded uncertainty in the dynamics.

\section{Loss Function for Fine-Tuning}~\label{appendix:fine-tuning loss}
{\centering
\includegraphics[width=0.7\linewidth]{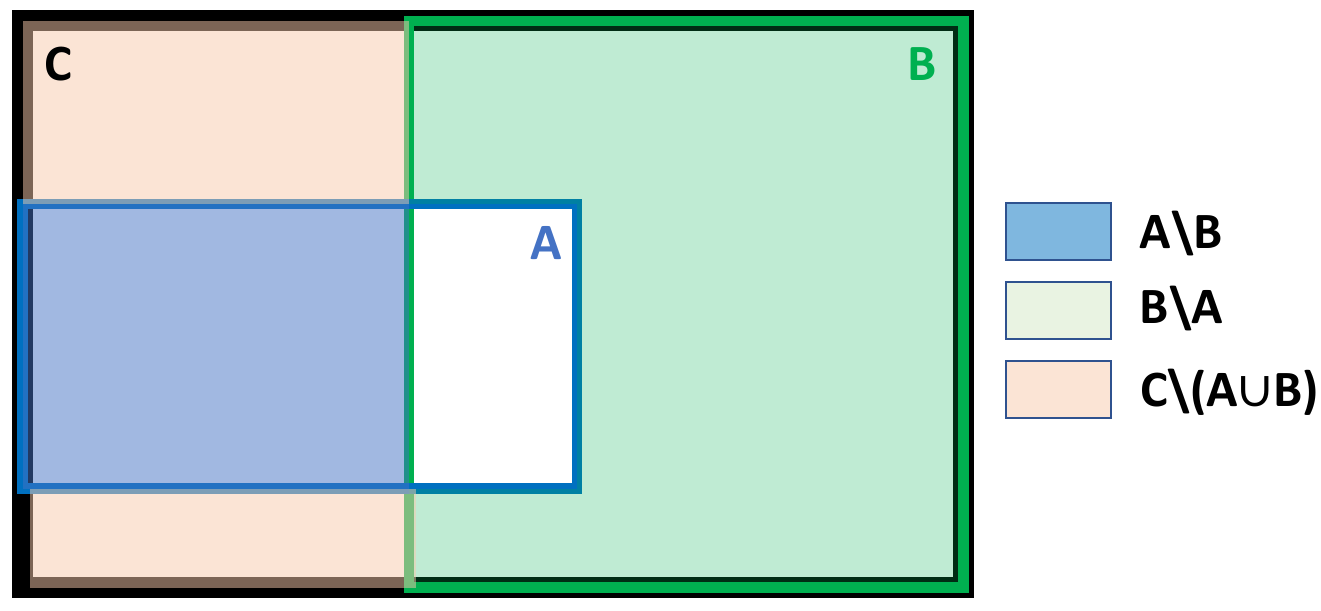}
\captionof{figure}{Visualization of different terms in the loss function for a single object setting.}\label{fig:loss}}

We use an almost-everywhere differentiable loss function for training. The loss function seeks to ensure that the predicted shape (e.g., bounding box) encloses the ground truth shape while also ensuring that the predicted shape is not too large.

Let's consider the simplest setting wherein we have one object in the scene and we are making a single prediction. In this case, $A$ denotes the (convex) ground-truth shape and $B$ denotes the (convex) predicted shape. Let $C$ denote the convex hull of $A$ and $B$. Our loss function is a weighted combination of three terms,
\begin{align*}
    L &\coloneqq \ w_1 l_1 + w_2 l_2 + w_3 l_3 \\
    &=  \ w_1 \frac{|A \backslash B|}{|A|} + w_2 \frac{|B \backslash A|}{|B|} + w_3 \frac{|C \backslash (A \cup B)|}{|C|}.
\end{align*}

The first term is the most important; it tries to ensure that $B$ encloses $A$. The second term tries to make sure that $B$ is not much larger than it needs to be, see Figure~\ref{fig:loss}. The first and second terms are sufficient if $A$ and $B$ are overlapping. However, if they do not overlap, there is no gradient information provided by the first two terms. Following \citep{Rezatofighi19}, we introduce a third loss term in order to provide gradient information when the shapes do not intersect. The loss terms $l_1, l_2, l_3$ are each bounded within $[0,1]$. Hence, if we choose $w_1,w_2,w_3$ such that $\sum_i w_i = 1$, then the overall loss is also bounded within $[0,1]$.

Now let's consider the setting wherein, $A$ denotes the union of multiple ground-truth bounding boxes (say we have $m$ objects in the scene)
and $B$ is the union of all the predicted bounding boxes (we predict $n$ boxes). We consider all the individual bounding box predictions $B_i, \forall i \in \{1,\dotsc n \}$ and associate the closest \textit{visible} ground-truth bounding box $A_i$ to each prediction. Now we can define $C_i$ as the convex hull of $A_i$ and $B_i$ and the resulting loss function, $L_i$,
\begin{equation*}
    L_i \coloneqq \ w_1 \frac{|A_i \backslash B_i|}{|A_i|} + w_2 \frac{|B_i \backslash A_i|}{|B_i|} + w_3 \frac{|C_i \backslash (A_i \cup B_i)|}{|C_i|}.
\end{equation*}
Hence, the overall loss is,
\begin{equation*}
    L = \frac{1}{n}\sum_{i=1}^{n}L_i.
\end{equation*}

Please refer to ~\citep[Appendix 4.3]{Rezatofighi19} for instructions on how to compute the loss analytically for axis-aligned bounding boxes. 

\section{Ablations}
We provide additional simulation results to illustrate the effects of: (1) closed-loop distribution shifts on safety 
wherein \pwc is robust to an increase in the level of closed-loop distribution shift while the baseline, CP-avg., is not which leads to higher collision rates for CP-avg., (2) the tradeoff in different partition sizes for fine-tuning using split-CP
\addendum{, (3) the effect of varying $\epsilon$ on the safety rate
, (4) impact of using different number of sampled configurations for calibration and online planning
, and (5) comparison against additional uncertainty-aware perception systems that use a heuristic notion of uncertainty
}

\smallskip
\noindent \textbf{Effects of closed-loop distribution shift on misdetections.}~\label{appendix:kl-divergence-results}
\addendum{In addition to the challenge of generalization, we highlight another challenge that any uncertainty quantification method for perception must tackle. 
Suppose we fix a policy $\pi^{\phi}$ (that uses perception system $\phi$) and collect a dataset of observations in different calibration environments from the states that result from applying $\pi^{\phi}$. We can use ground-truth bounding boxes in these environments to produce a calibrated perception system $\tilde{\phi}$ with a statistical assurance on correctness for the distribution of observations induced by $\pi^{\phi}$. However, if we now apply the policy $\pi^{\tilde{\phi}}$ using the \emph{calibrated} perception system $\tilde{\phi}$, the resulting distribution of states will be \emph{different} from the distribution that forms the calibration dataset, thus invalidating the statistical assurance. We refer to this challenge as \emph{closed-loop distribution shift}, which is similar to challenges that arise in offline reinforcement learning \citep{Levine2020OfflineRL} and imitation learning \citep{pmlr-v15-ross11a}.} 

To illustrate the effect of closed-loop distribution shifts on misdetections, we used exactly the same setup described above to obtain the simulation results in Figure~\ref{fig:all-bars}. We changed the planner cost to have a different weighting on the cost-to-go. 
For one setting, we chose a weight $w=1$ on the cost-to-go, which is the same as the weighting on the cost-to-come. In another setting, we chose a weight $w=10$ on the cost-to-go, and hence a $10\times$ more emphasis on the cost-to-go compared to the cost-to-come. 
Table~\ref{tab:results_KL} shows the KL-divergence between the states visited by the planner and the sampling distribution of states for calibration as a measure of the closed-loop distribution shift. Increasing closed-loop shifts lead to higher misdetections. One can see that a simple change in the planner parameters can lead to potentially large changes in the safety rates for CP-avg. The closed-loop shift we may see in practice is unknown apriori. Hence, it is difficult to make any statements on the planner safety in closed-loop despite using CP for calibration of the perception system. \pwc, on the other hand, is robust to the closed-loop shifts and can still satisfy the misdetection and safety assurance regardless of the planner parameters used.

\begin{table}[h]
\centering
\caption{{A comparison of the effect of changing the planner parameters on CP-avg. and \pwc. }}
\label{tab:results_KL}
\begin{tabular}{ p{0.3\linewidth} p{0.13\linewidth} p{0.13\linewidth} p{0.15\linewidth}}
\hline
\textbf{Method}  & \textbf{Collision} & \textbf{Mis-detection} & \textbf{KL-divergence}  \\ 
 \hline
 \hline
 CP-avg. ($w=1$) & $14 \% $ & $54\%$ & $2.09$    \\
 \hline  CP-avg. ($w=10$) & $2\% $ & $64\%$ & $2.72$    \\
 \hline
  \pwc ($w=1$) & $0\%$ & $0\%$  & $1.48$  \\
   \hline
  \pwc  ($w=10$) & ${0  \%}$ & $2\%$ & 2.04\\
   \hline
\end{tabular}
\end{table}

\smallskip
\noindent\textbf{Effect of finetuning dataset size. }~\label{appendix:fine-tuning-results}
Upon collecting a calibration dataset of about $400$ environments, as described in the experiment setup in Section~\ref{sec:simulations}, we may choose to use a smaller subset of the calibration dataset to further finetune the pre-trained perception model to perform better in the types of environments we are interested in deploying the robot in. We consider the effect of different dataset split sizes for finetuning and then calibration. Using a larger set of environments for finetuning $|D_{\text{tune}}|$ may result in a better tuned model, but will leave fewer environments for calibration, $|D_{\text{cal}}|$, resulting in a more conservative ${\hat{\epsilon}}$ and ${\hat{q}_{1-\epsilon}}$ that satisfies the dataset-conditional guarantee~\eqref{eq:dataset_conditional_cp}, and vice versa. This trade-off is seen in Table~\ref{tab:results_fine}, where we observe the best performance when we have an equal split between finetuning and calibration.

\begin{table*}[h]
\centering
\caption{A comparison of the effect of various partition sizes for fine-tuning and calibration for \pwc.}
\begin{tabular}{ccccc}
    \hline
    \textbf{Split size ($|D_\text{tune}| + |D_\text{cal}|$)} & \textbf{$\hat{q}_{0.85}$ (in m)} & \textbf{Collision} & \textbf{Misdetection} & \textbf{Goal Reached} \\
    \hline
    \hline
    $100 + 300$ & $0.68$ & $0\%$ & $1\%$ & $89\%$ \\
    \hline
    $200 + 200$ & $\boldsymbol{0.64}$ & $0\%$ & $\boldsymbol{1\%}$ & $\boldsymbol{94\%}$ \\
    \hline
    $300 + 100$ & $0.93$ & $0\%$ & $2\%$ & $76\%$ \\
    \hline
\end{tabular}
\label{tab:results_fine}
\end{table*}

\smallskip
\noindent\textbf{Effects of varying $\epsilon$ on safety rate.}~\label{appendix:safety-cost-results}
We compare our method, \pwc, to the baseline CP-avg. We vary the allowable safety rate $\epsilon$ for each method, and compute the rate of safety in $100$ test environments. As seen in Table~\ref{tab:results_safety} and Figure~\ref{fig:missRate}, our method not only guarantees that the rate of misdetections to be bounded, but also the safety rate. The safety rate of \pwc is also consistently better than that of CP-avg.

\begin{table}[h]
\centering
\caption{\addendum{A comparison of the safety rates of CP-avg. and \pwc when we vary the confidence threshold $\epsilon$. }}
\label{tab:results_safety}
\begin{tabular}{ccc}
\hline
$\epsilon$ & \textbf{CP-avg.} & \textbf{\pwc}  \\ 
 \hline
 \hline
 $0.20$ & $95\%$ & $\mathbf{100\%}$   \\
 \hline  $0.10$ & $98\%$ & $\mathbf{99\%}$    \\
 \hline
  $0.15$ & $99\%$  & $\mathbf{100\%}$  \\
   \hline
  $0.10$ & $98\%$ & $\mathbf{100\%}$\\
   \hline
   $0.05$ & $98\%$ & $\mathbf{100\%}$\\
   \hline
\end{tabular}
\end{table}

\begin{table}[h]
\centering
\caption{\addendum{A comparison of the CP inflation $\hat{q}_{0.85}$ when we vary the number of sampled configurations. }}
\label{tab:results_samples}
\begin{tabular}{ p{0.17\linewidth} p{0.14\linewidth} p{0.13\linewidth} p{0.13\linewidth} p{0.18\linewidth}}
\hline
$\#$ samples & $\mathbf{\hat{q}_{0.85}}$ & \textbf{Collision}  & \textbf{Mis-detection}&  \textbf{Goal Reached} \\ 
 \hline
 \hline  $1050$ &$0.7086$ & $0\%$ & $1\%$ & $47\%$\\
 \hline
  $1500$ &$\mathbf{0.6910}$ & $0\%$ & $\mathbf{0\%}$ & $57\%$\\
   \hline
  $2000$ &$0.75$ & $0\%$ & $7\%$ & $\mathbf{90\%}$\\
   \hline
\end{tabular}
\end{table}

\smallskip
\noindent\textbf{Effect of varying the number of sampled configurations.}~\label{appendix:samples-results}
For our experiments, we used a fixed set of $2000$ sampled configurations. However, depending on the planner configuration requirements and desired speed of computation, the user may decide to have a different number of configuration samples for calibration and planning. We study the change in the CP inflation, $\hat{q}_{0.85}$, the resulting collision, misdetection, and task completion (reaching goal) rates. As we can see in Table~\ref{tab:results_samples}, in our case, we have far fewer misdetections with fewer samples, but we also observe a decrease in number of times the robot reaches the goal. We suspect that with fewer samples of configurations (consisting of $x,y, v_x, v_y$), it is harder for the sampling-based motion planner to find feasible paths. On the other hand, we also observe a less conservative $\hat{q}_{0.85}$ when we use fewer samples; this is presumably also a result of using fewer samples to compute the non-conformity score that comprises of the worst-case perception error across all configurations.

\smallskip
\noindent\textbf{Comparison to heuristic inflation.}~\label{appendix:heuristic-results}
We compare \pwc to the baseline method of inflating the bounding box predictions based on some heuristic confidence level, i.e., we scale the bounding box with 1 - confidence (so we scale the boxes where we are less confident by a larger amount). As shown in Table~\ref{tab:results_heuristics}, While this baseline demonstrates a higher completion rate, both the collision rate and the misdetection rate increase significantly, leading to unsafe situations. Further, while our method provides a statistical safety guarantee, the baseline method does not admit any formal assurance.

\begin{table}[h]
\centering
\caption A comparison of the effects of using heuristic inflation versus \pwc.
\label{tab:results_heuristics}
\begin{tabular}{cccc}
\hline
\textbf{Method} & \textbf{Collision}  & \textbf{Misdetection}&  \textbf{Goal Reached} \\ 
 \hline
 \hline  \pwc & $\mathbf{0}\%$ & $\mathbf{7}\%$ & $90\%$\\
 \hline
  Heuristic  & $3\%$ & $67\%$  & $\mathbf{97}\%$\\
   \hline
\end{tabular}
\end{table}

\section{Experiment Setup}
\label{appendix:hardware}

\subsection{System Identification}
\label{appendix:sysid}
To perform system identification of the Unitree Go1 quadruped robot, we collected trajectories using a Vicon motion capture system. We then used MATLAB's system identification toolbox \citep{MATLAB}. Specifically, we provided an initial linear ODE grey box model guess and then used prediction error minimization (PEM) for refinement. The resulting system is shown in~\eqref{eqn:sys_id} where $x$ and $y$ describe the positional state of the robot in the environment, $v_x$ and $v_y$ describe the respective velocities, and $u_x$ and $u_y$ describe the respective commanded velocities. 

\begin{table*}[b]
\centering
\begin{minipage}{0.75\linewidth}
    \begin{equation}
    \begin{bmatrix}
        \dot x \\
        \dot y \\
        \dot v_x \\
        \dot v_y
    \end{bmatrix} = 
    \begin{bmatrix}
    0 & 0 & 1 & 0 \\
    0 & 0 & 0 & 1 \\
    0 & 0 & -2.5170 & 0.1353 \\
    0 & 0 & -0.5197 & -3.9680 
    \end{bmatrix}
    \begin{bmatrix}
    x \\
    y \\
    v_x \\
    v_y 
    \end{bmatrix}
    +
    \begin{bmatrix}
    0 & 0 \\
    0 & 0 \\
    2.3350 & 0 \\
    0 & 4.6510
    \end{bmatrix}
    \begin{bmatrix}
       u_x \\
       u_y
    \end{bmatrix}
\label{eqn:sys_id}
\end{equation}
\end{minipage}
\end{table*}

\subsection{Chair Test Dataset} 
\label{appendix:chairs}
Our test dataset of chairs for the first set of experiments conducted in Section \ref{sec:hardware} included 8 chairs with diverse sizes and geometries unseen in training and calibration for the perception system. Test chairs are shown below in Figure \ref{fig:chair_test}.

\begin{figure*}[h]
  \centering
  \begin{minipage}{0.6\linewidth}
    \includegraphics[width=\linewidth]{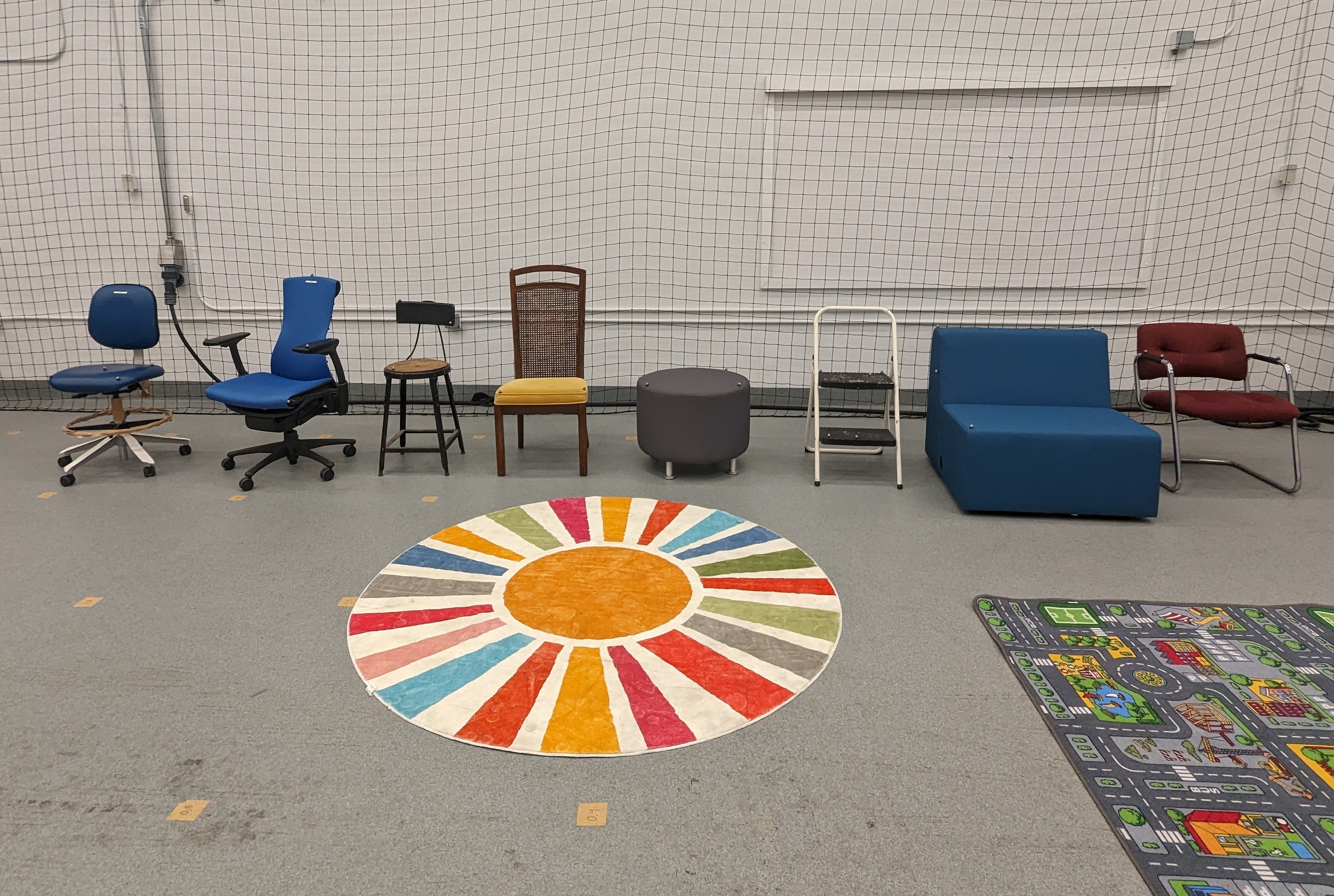} 
    \caption{New, unseen test chairs used in original hardware experiments. In the fast experiments, only chairs 1, 2, and 5 (left to right) were used.}
    \label{fig:chair_test}
  \end{minipage}
\end{figure*}

\subsection{Environments} 
\label{appendix:envs}
As described in Section \ref{sec:hardware}, in the first set of experiments, the robot was tested in 30 unique environments with varying furniture configurations and goals. The following 30 figures show an image of each configuration, accompanied by a bird's-eye map of the obstacle and goal locations. 

\begin{figure*}[h]
\centering
\begin{minipage}{\linewidth}
    \centering
    \begin{minipage}{0.28\linewidth}
        \centering
        \includegraphics[width=\linewidth]{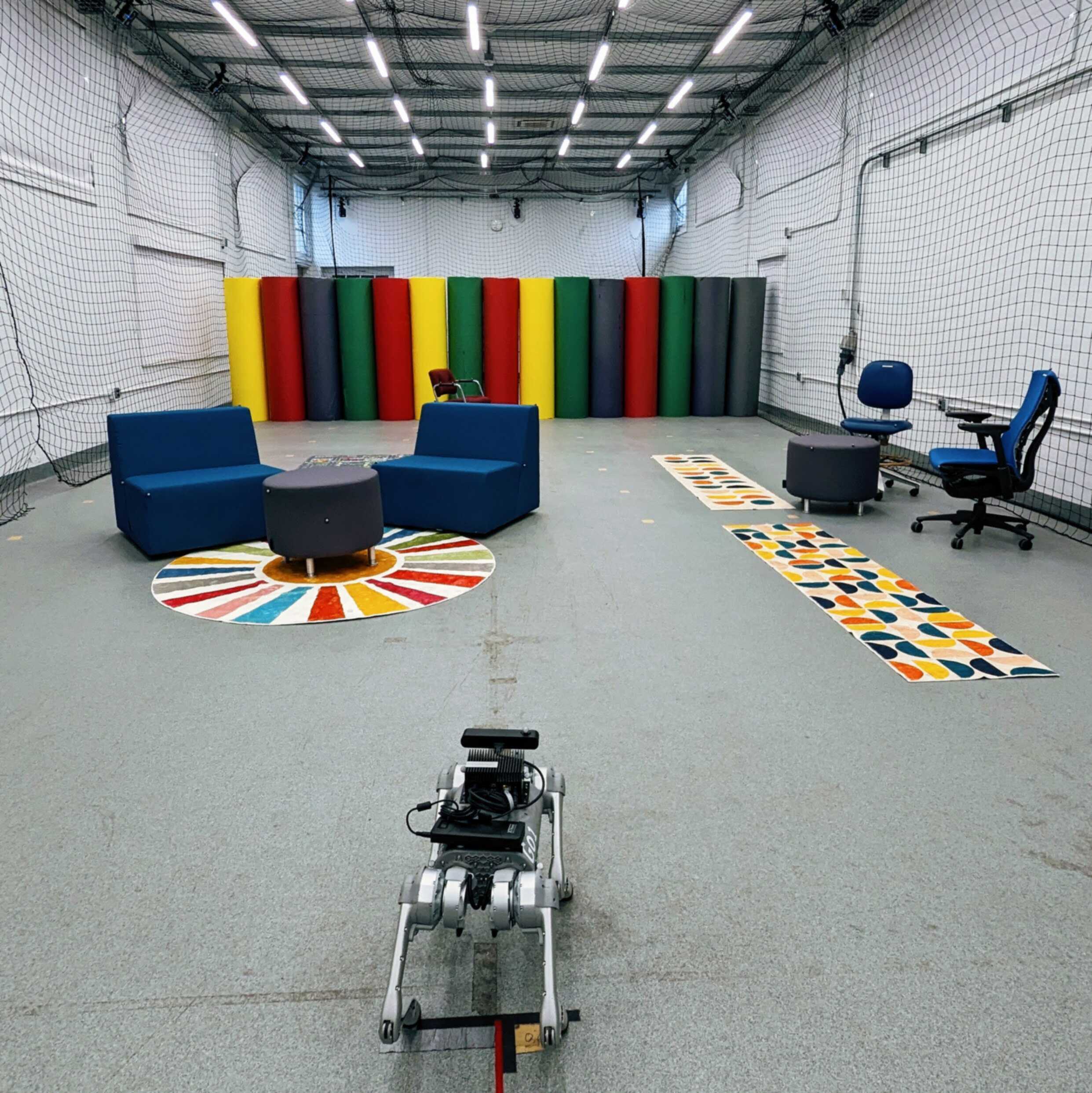}
    \end{minipage}
    \hfill
    \begin{minipage}{0.28\linewidth}
        \centering
        \includegraphics[width=\linewidth]{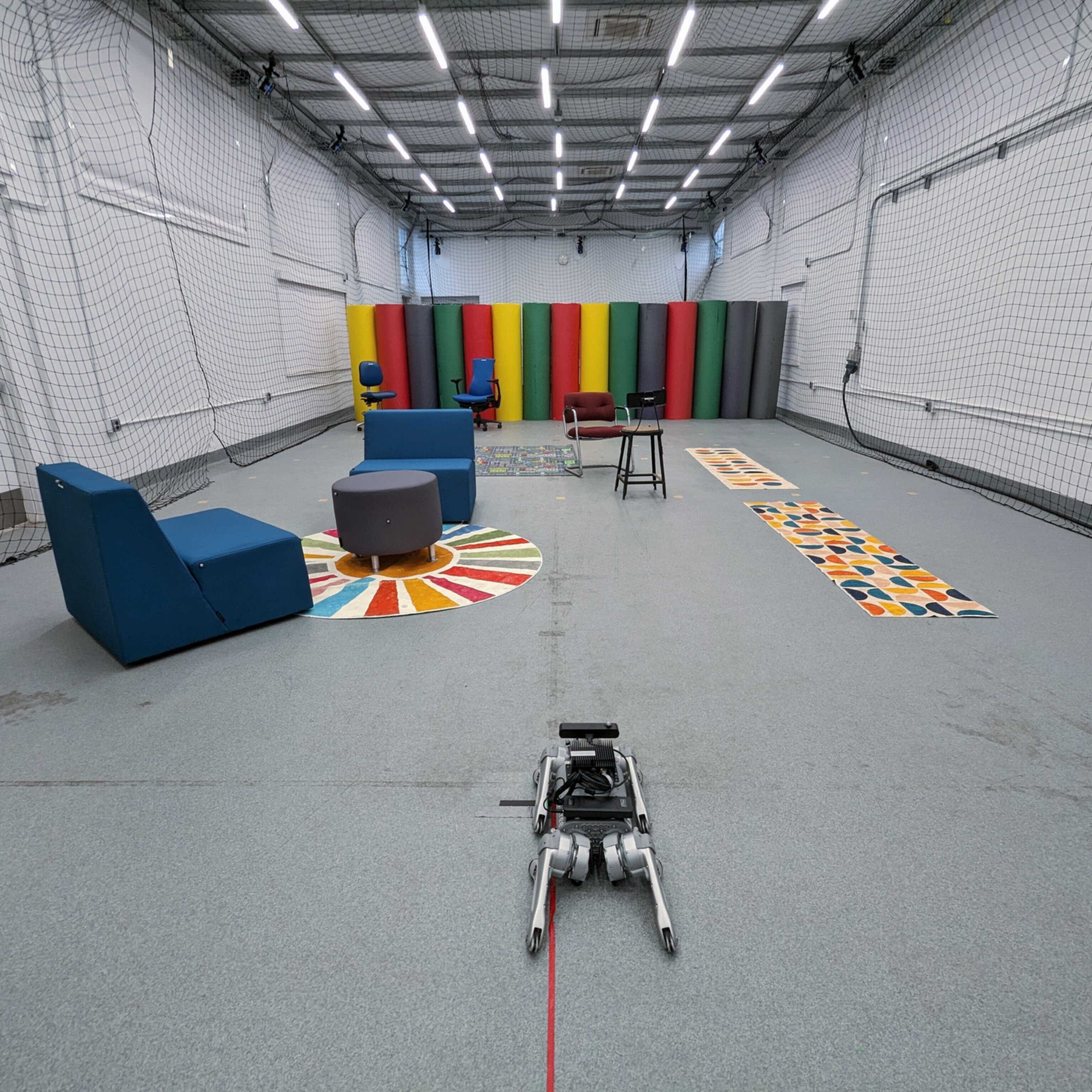}
    \end{minipage}
    \hfill
    \begin{minipage}{0.28\linewidth}
        \centering
        \includegraphics[width=\linewidth]{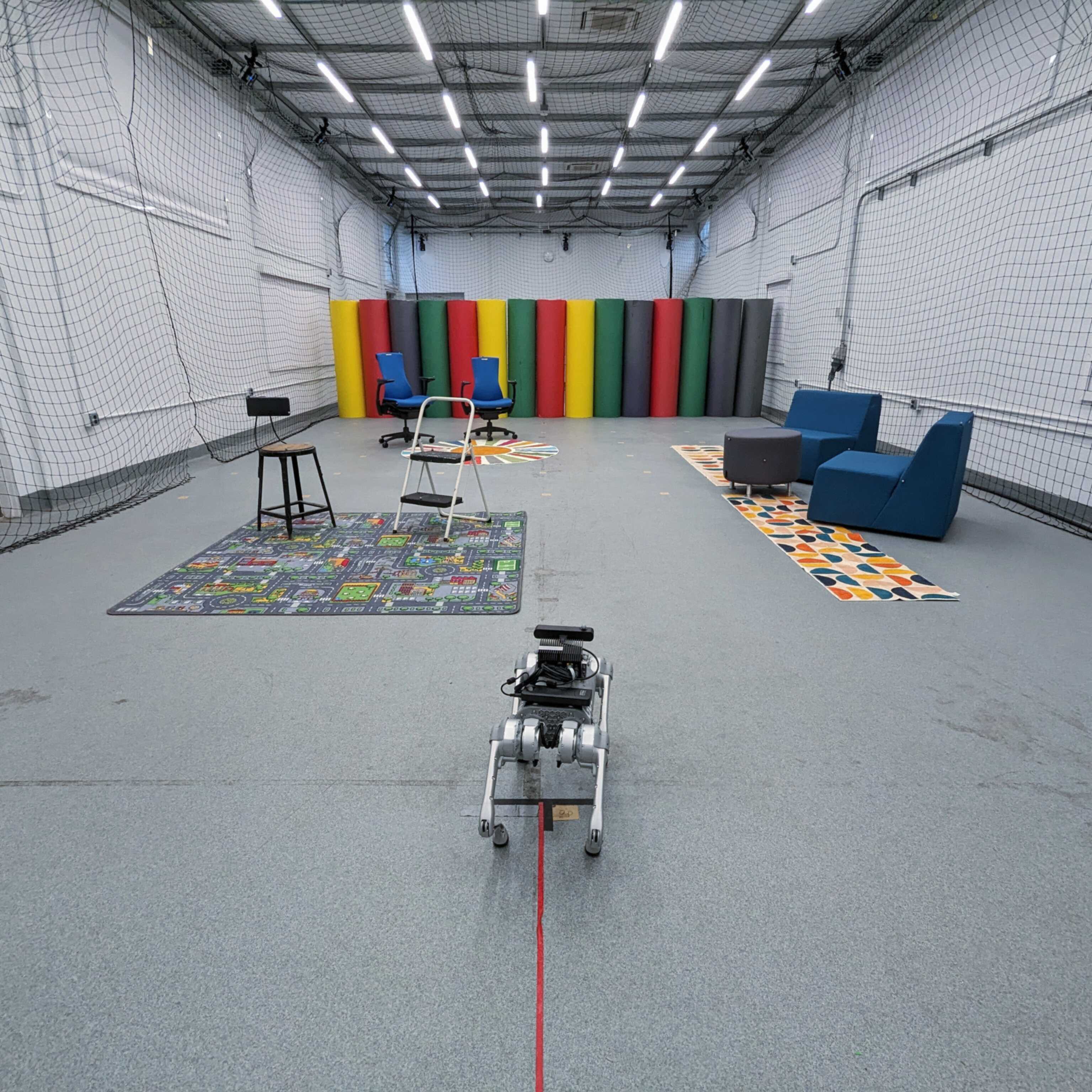}
    \end{minipage}
\end{minipage}
\vspace{12pt}
\begin{minipage}{\linewidth}
    \centering
    \begin{minipage}{0.28\linewidth}
        \centering
        \includegraphics[width=\linewidth]{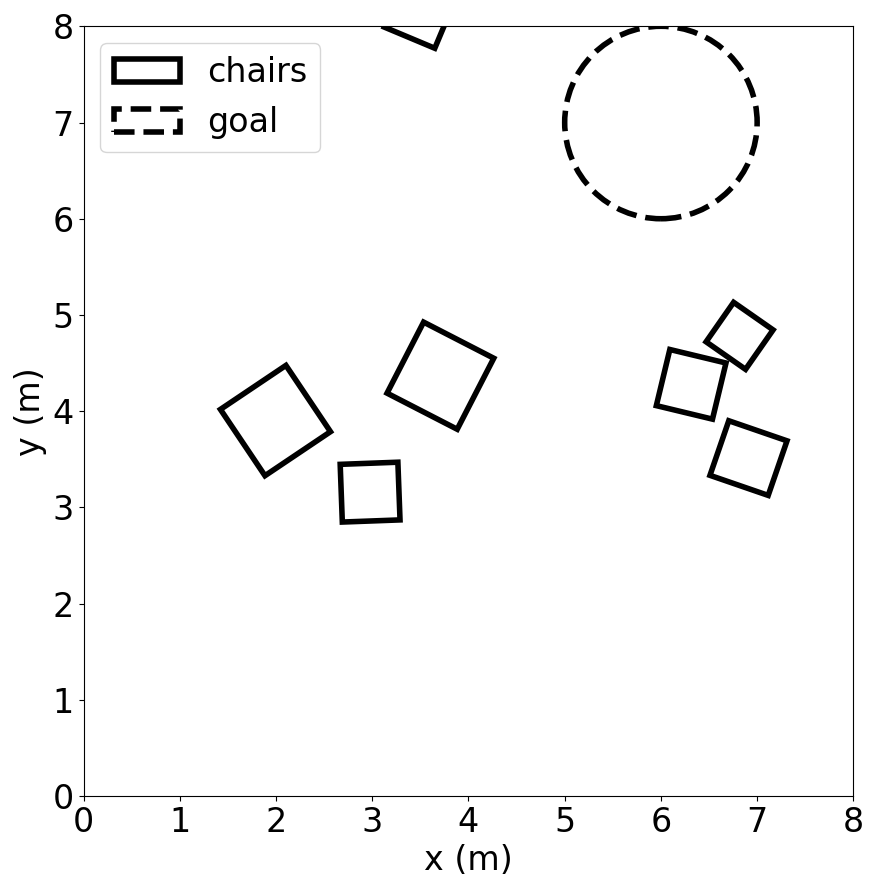}
        \caption*{(1) Environment 1}
    \end{minipage}
    \hfill
    \begin{minipage}{0.28\linewidth}
        \centering
        \includegraphics[width=\linewidth]{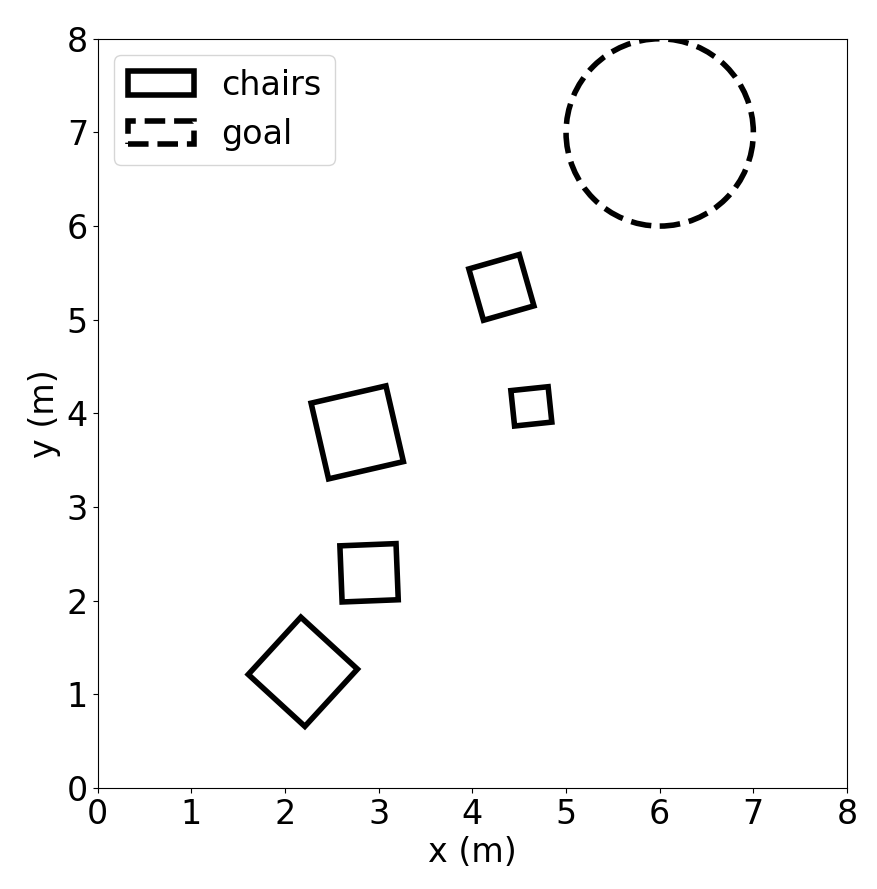}
        \caption*{(2) Environment 2}
    \end{minipage}
    \hfill
    \begin{minipage}{0.28\linewidth}
        \centering
        \includegraphics[width=\linewidth]{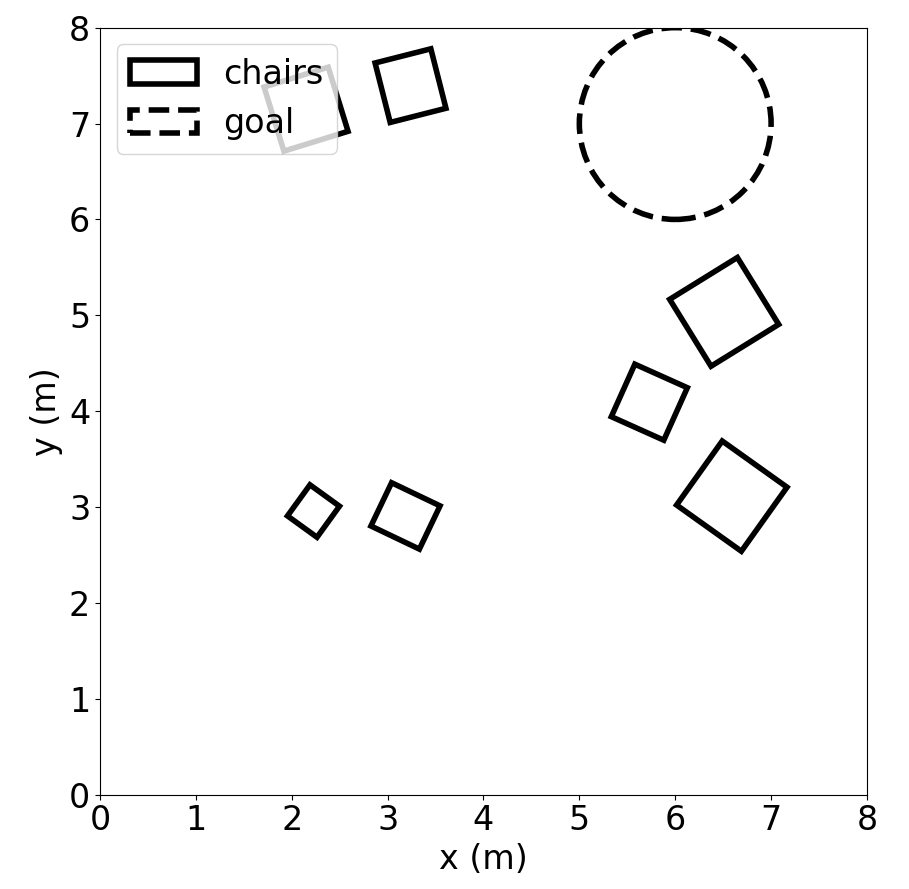}
        \caption*{(3) Environment 3}
    \end{minipage}
\end{minipage}
\end{figure*}

\begin{figure*}[h]
\centering
\begin{minipage}{\linewidth}
    \centering
    \begin{minipage}{0.28\linewidth}
        \centering
        \includegraphics[width=\linewidth]{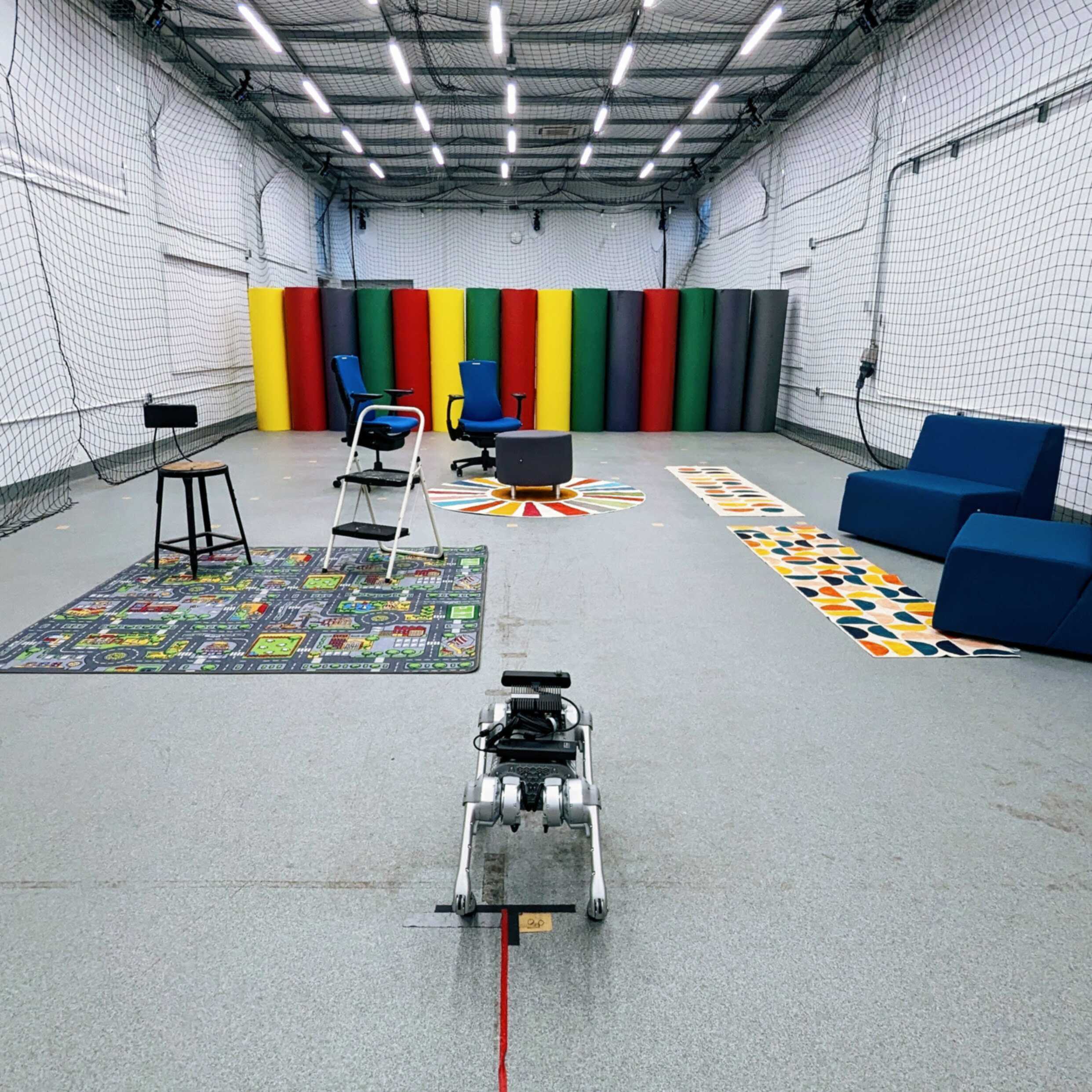}
    \end{minipage}
    \hfill
    \begin{minipage}{0.28\linewidth}
        \centering
        \includegraphics[width=\linewidth]{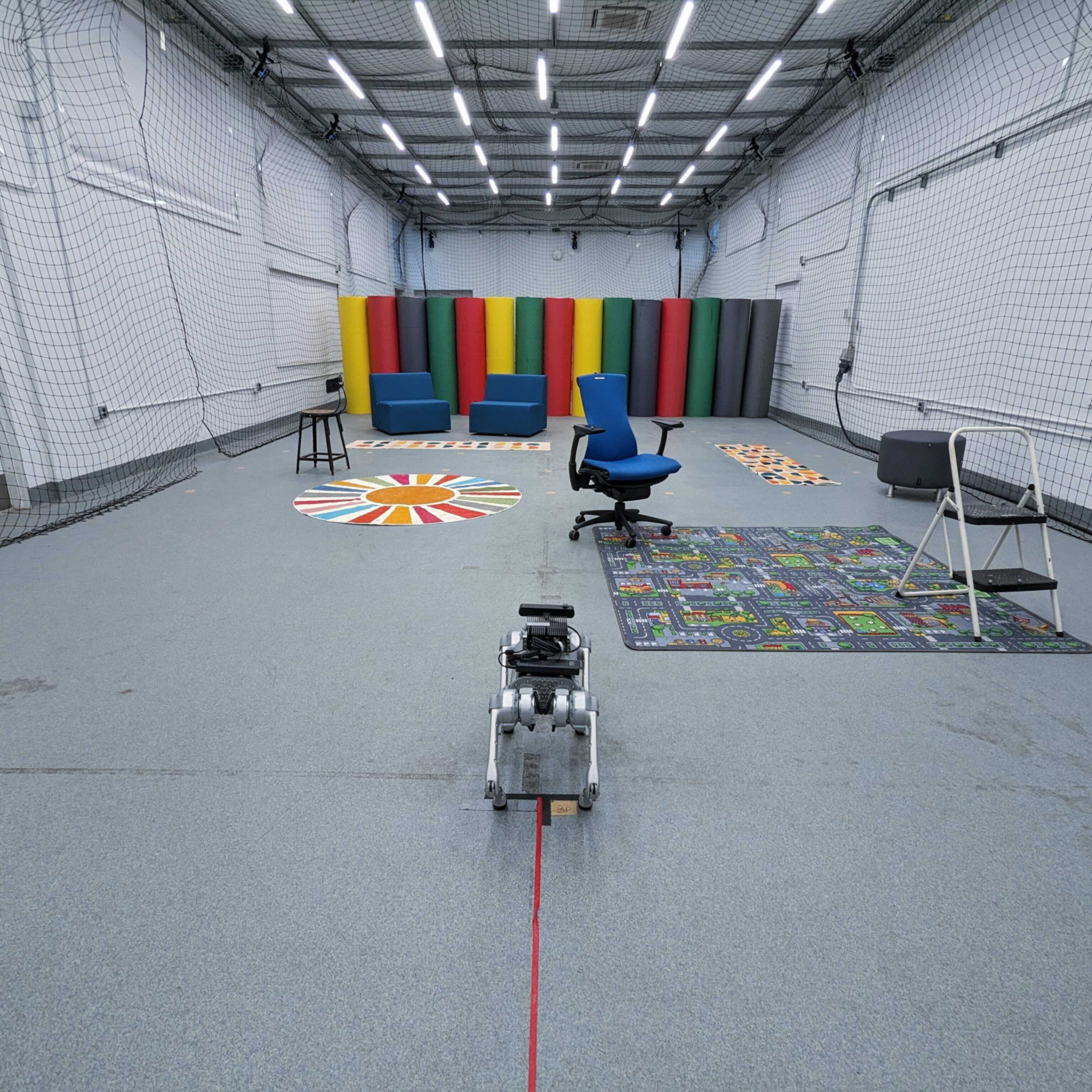}
    \end{minipage}
    \hfill
    \begin{minipage}{0.28\linewidth}
        \centering
        \includegraphics[width=\linewidth]{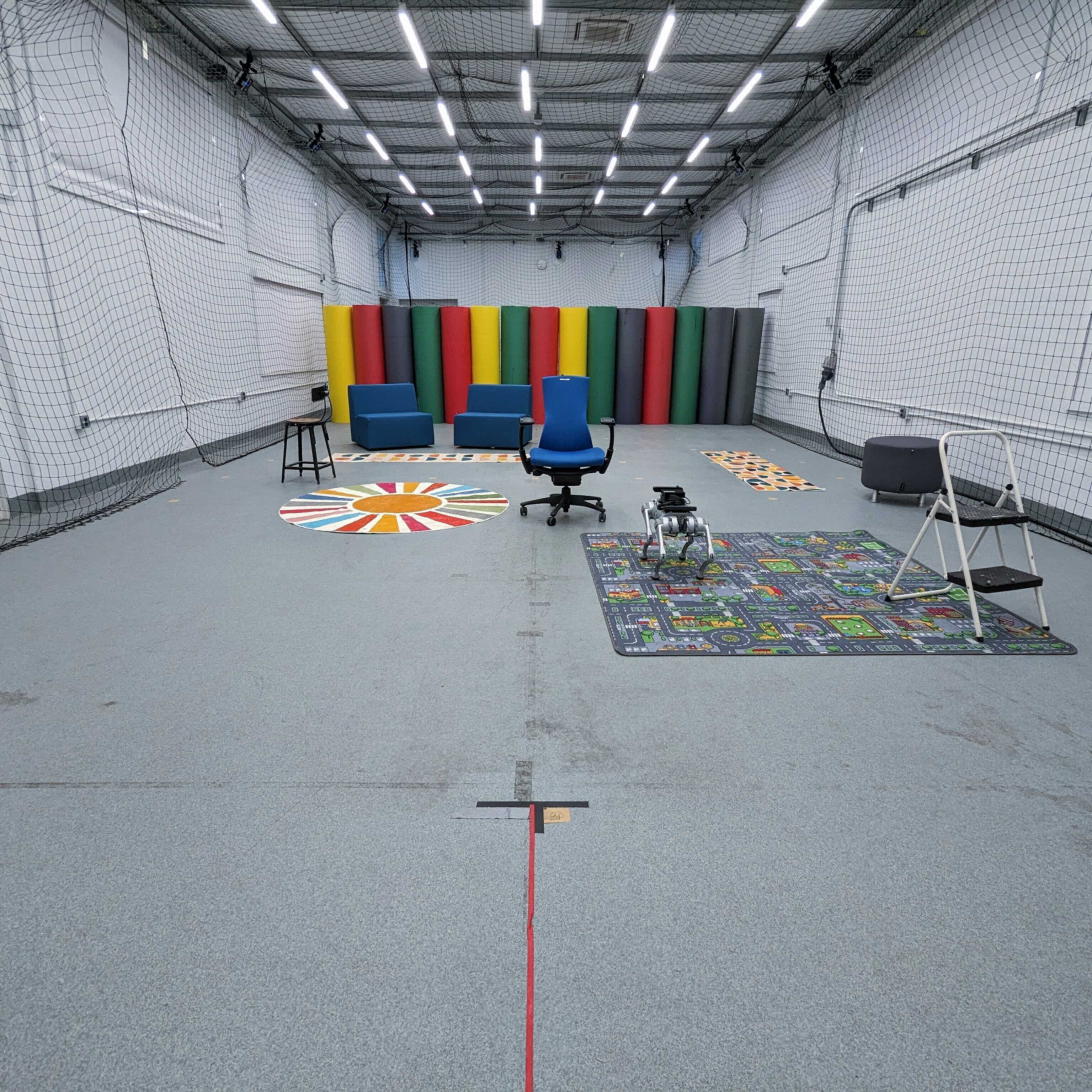}
    \end{minipage}
\end{minipage}
\vspace{12pt}
\begin{minipage}{\linewidth}
    \centering
    \begin{minipage}{0.28\linewidth}
        \centering
        \includegraphics[width=\linewidth]{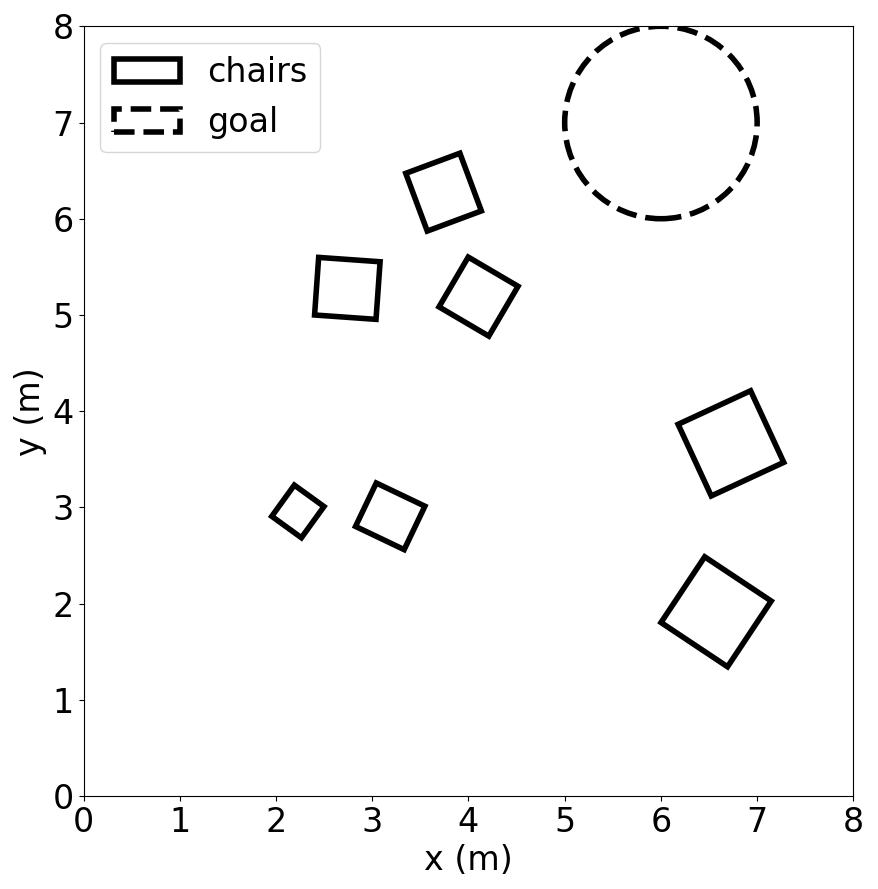}
        \caption*{(4) Environment 4}
    \end{minipage}
    \hfill
    \begin{minipage}{0.28\linewidth}
        \centering
        \includegraphics[width=\linewidth]{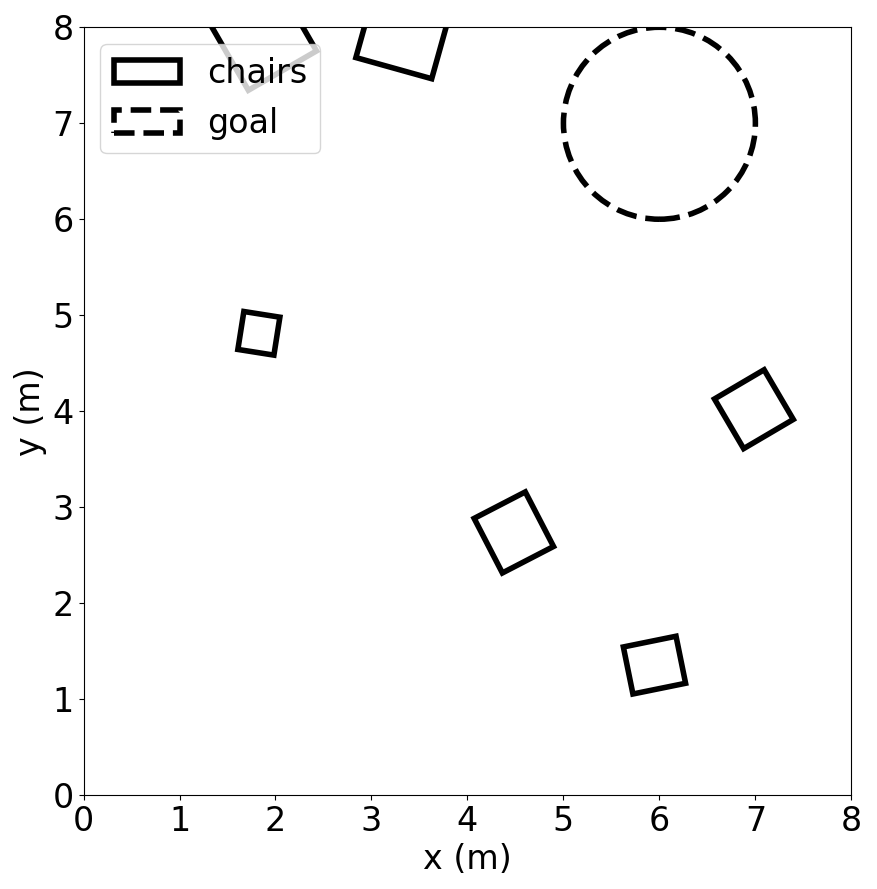}
        \caption*{(5) Environment 5}
    \end{minipage}
    \hfill
    \begin{minipage}{0.28\linewidth}
        \centering
        \includegraphics[width=\linewidth]{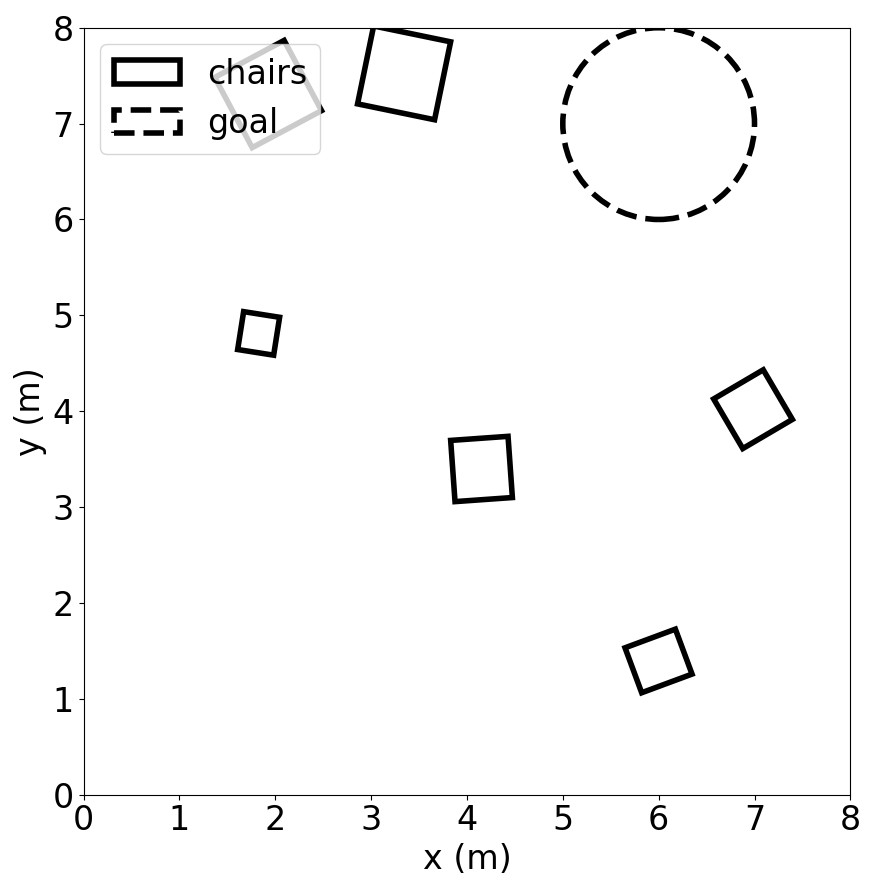}
        \caption*{(6) Environment 6}
    \end{minipage}
\end{minipage}
\end{figure*}

\begin{figure*}[h]
\centering
\begin{minipage}{\linewidth}
    \centering
    \begin{minipage}{0.28\linewidth}
        \centering
        \includegraphics[width=\linewidth]{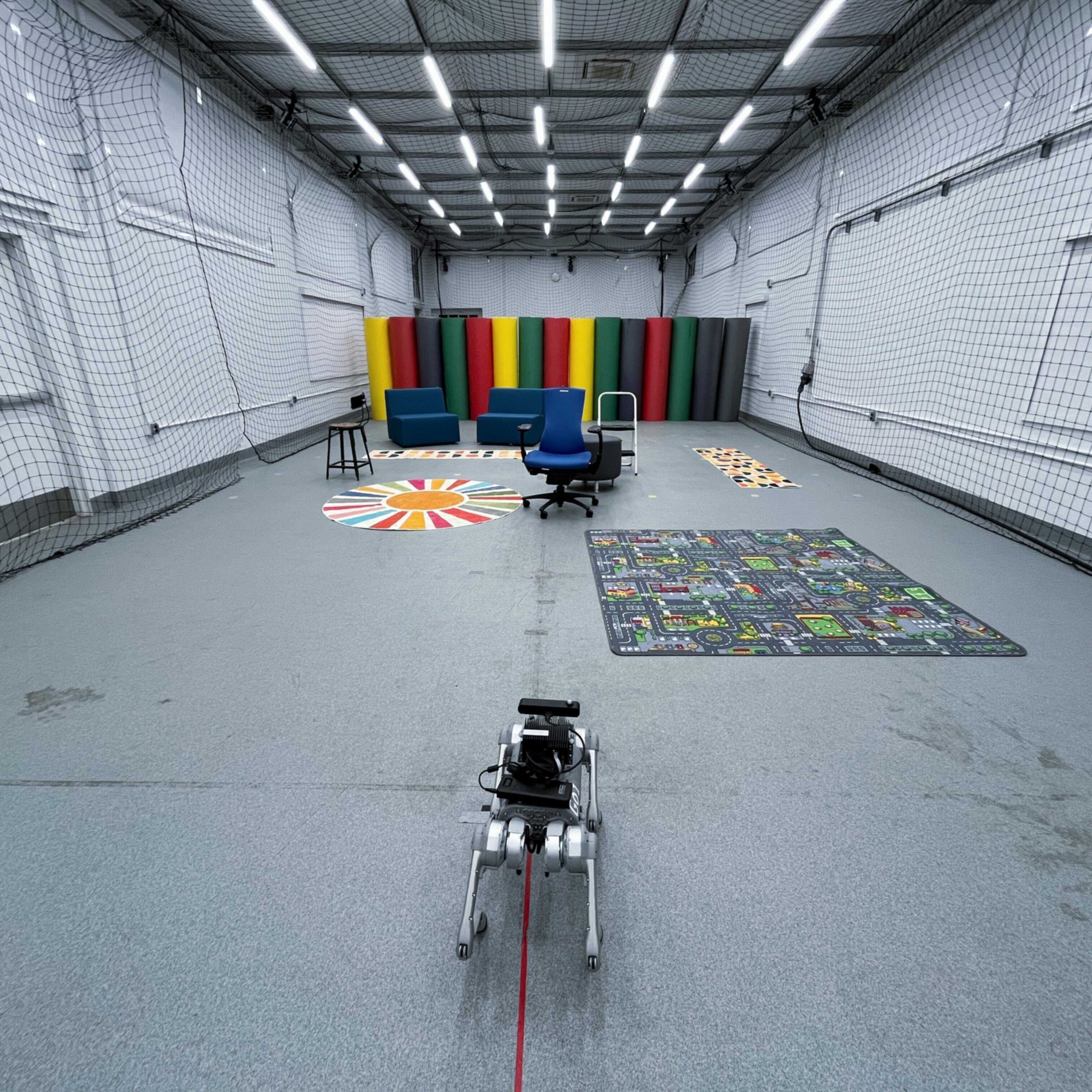}
    \end{minipage}
    \hfill
    \begin{minipage}{0.28\linewidth}
        \centering
        \includegraphics[width=\linewidth]{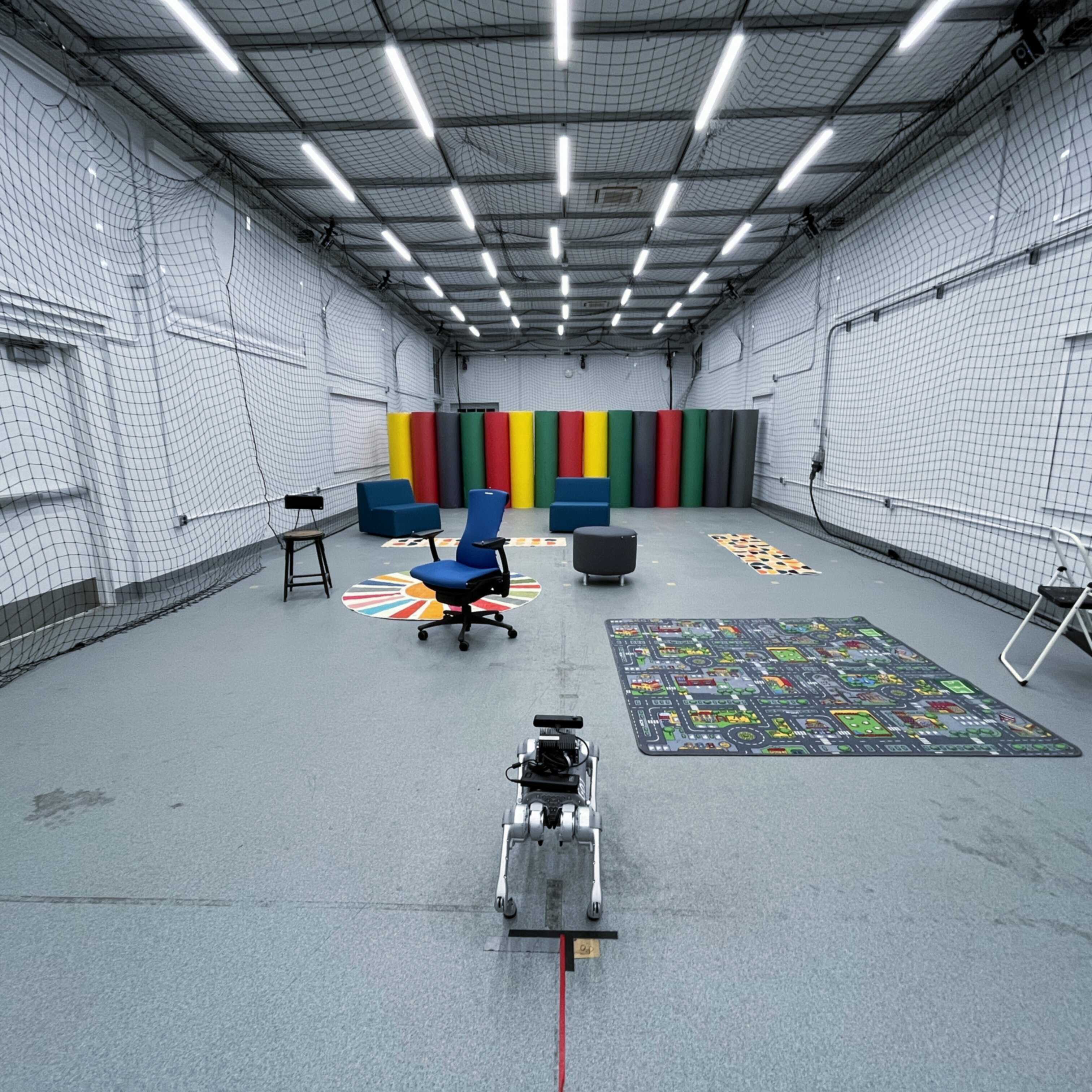}
    \end{minipage}
    \hfill
    \begin{minipage}{0.28\linewidth}
        \centering
        \includegraphics[width=\linewidth]{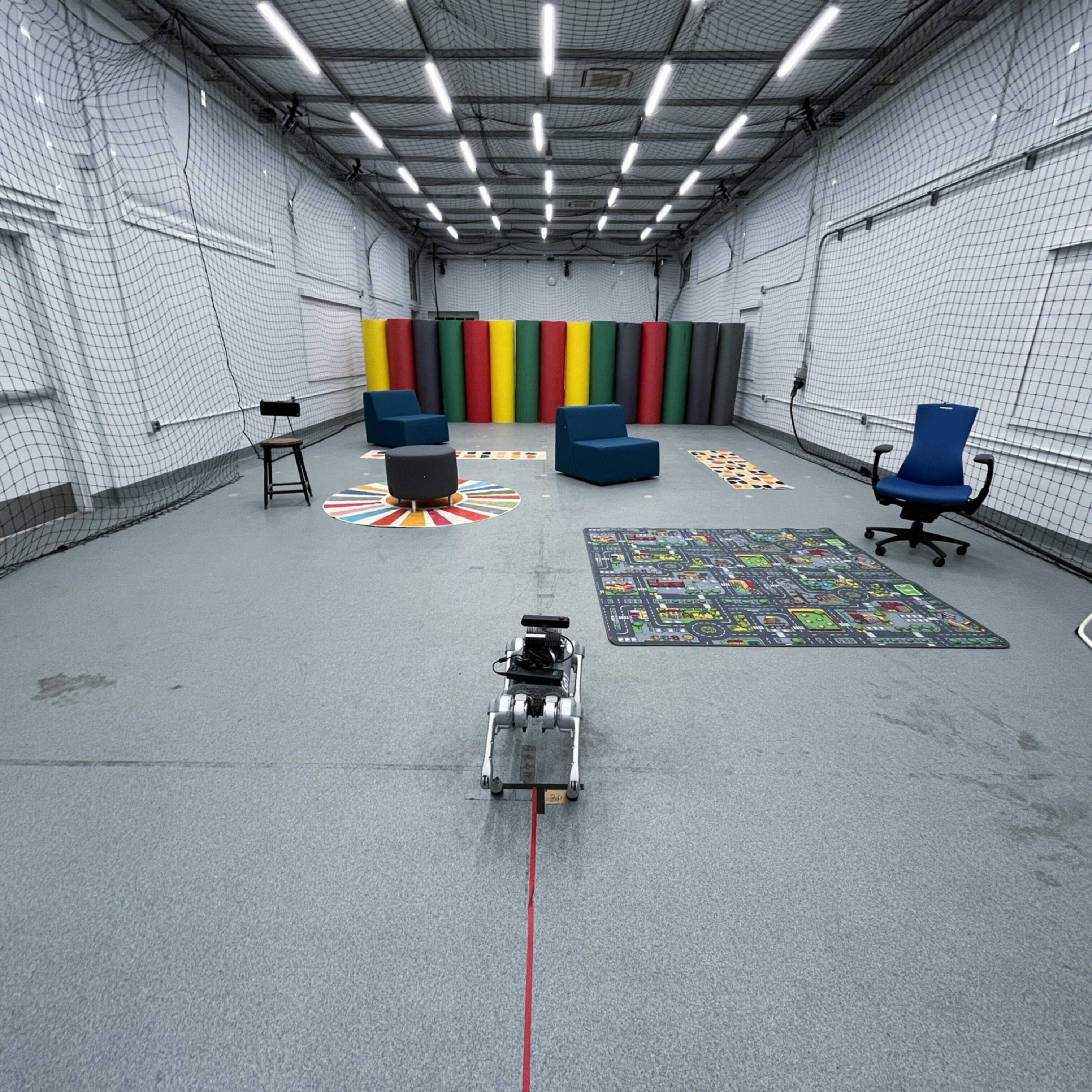}
    \end{minipage}
\end{minipage}
\vspace{12pt}
\begin{minipage}{\linewidth}
    \centering
    \begin{minipage}{0.28\linewidth}
        \centering
        \includegraphics[width=\linewidth]{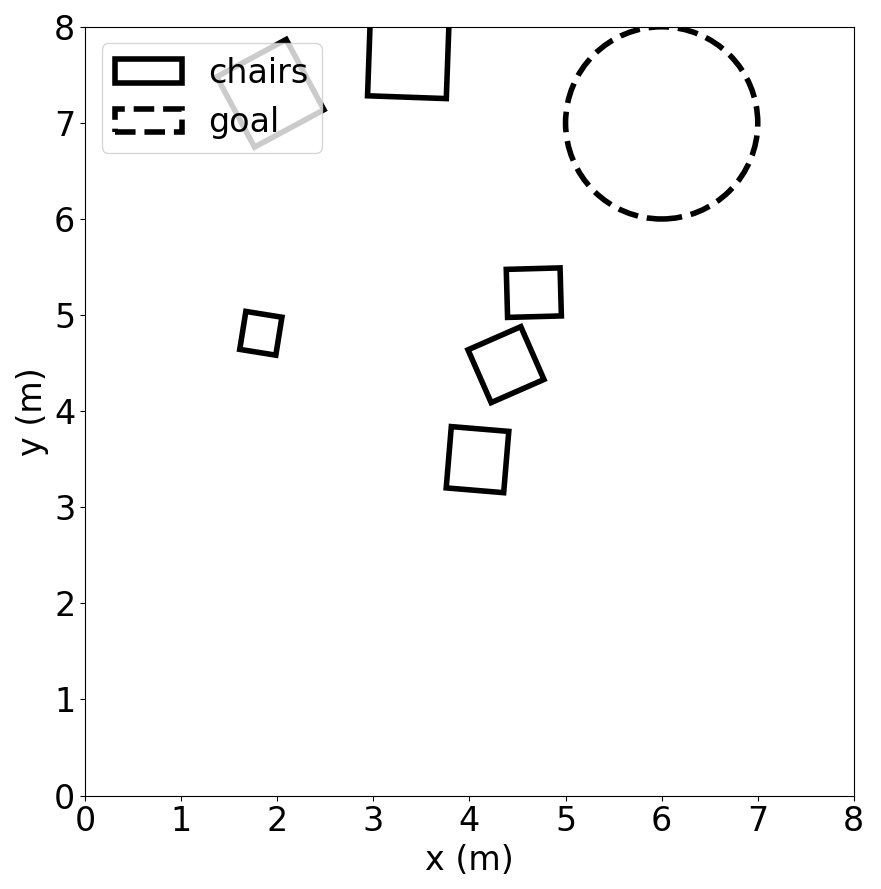}
        \caption*{(7) Environment 7}
    \end{minipage}
    \hfill
    \begin{minipage}{0.28\linewidth}
        \centering
        \includegraphics[width=\linewidth]{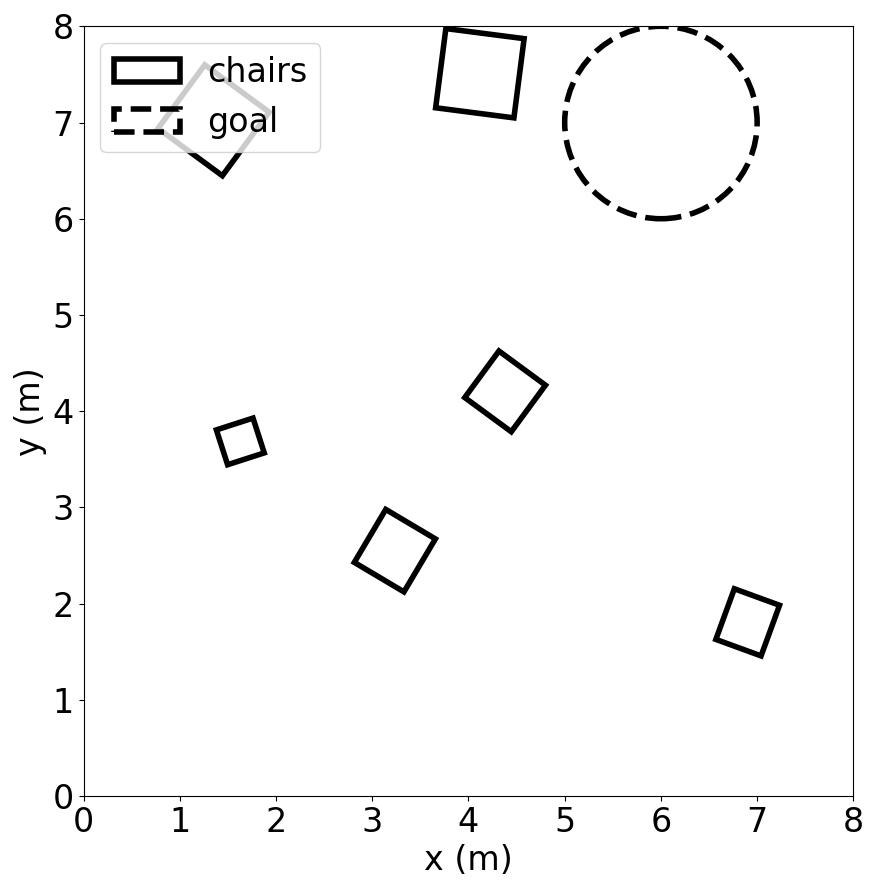}
        \caption*{(8) Environment 8}
    \end{minipage}
    \hfill
    \begin{minipage}{0.28\linewidth}
        \centering
        \includegraphics[width=\linewidth]{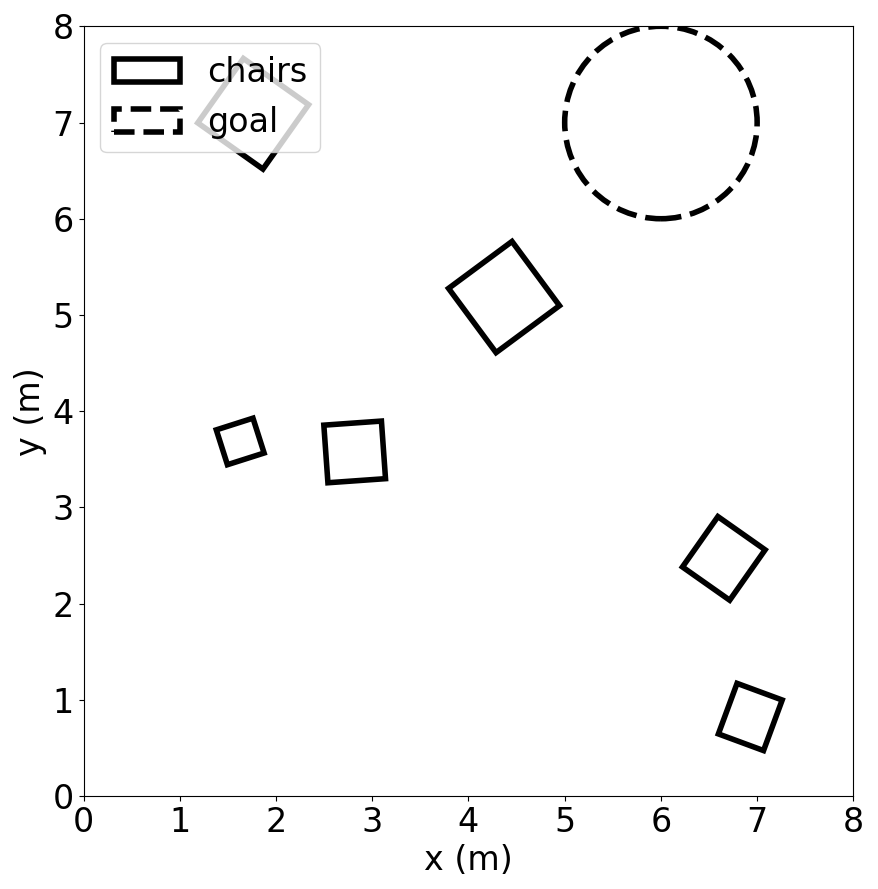}
        \caption*{(9) Environment 9}
    \end{minipage}
\end{minipage}
\end{figure*}

\begin{figure*}[h]
\centering
\begin{minipage}{\linewidth}
    \centering
    \begin{minipage}{0.28\linewidth}
        \centering
        \includegraphics[width=\linewidth]{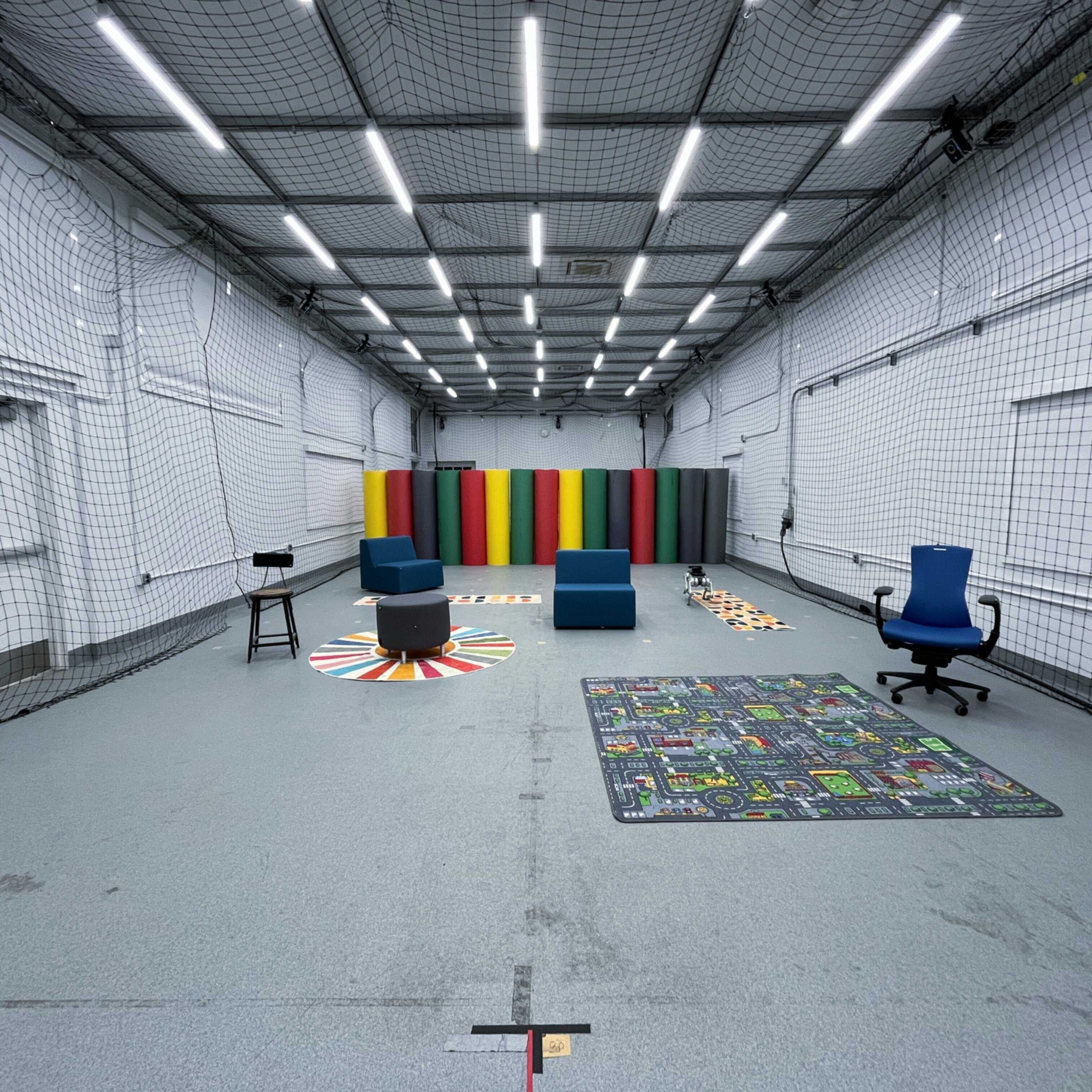}
    \end{minipage}
    \hfill
    \begin{minipage}{0.28\linewidth}
        \centering
        \includegraphics[width=\linewidth]{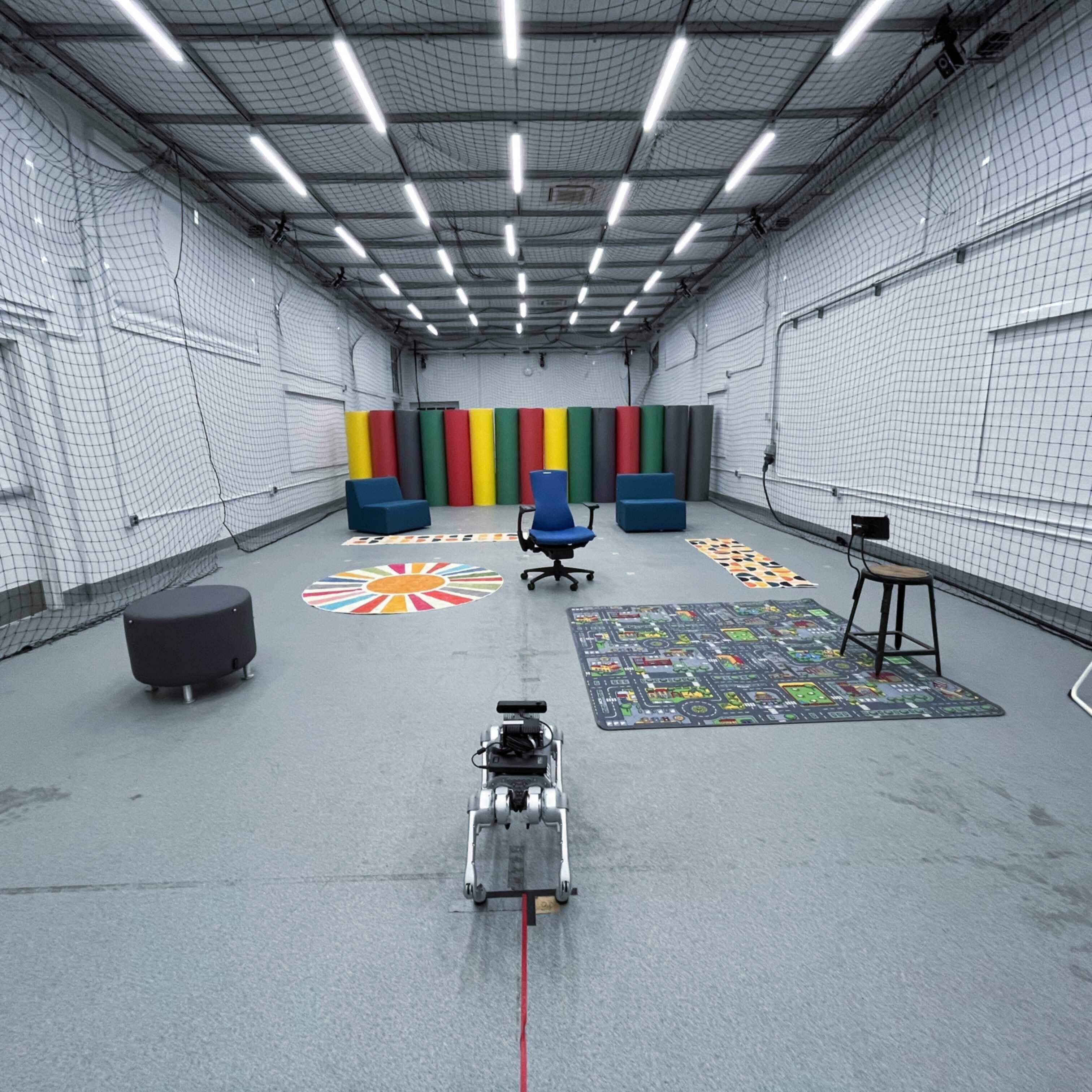}
    \end{minipage}
    \hfill
    \begin{minipage}{0.28\linewidth}
        \centering
        \includegraphics[width=\linewidth]{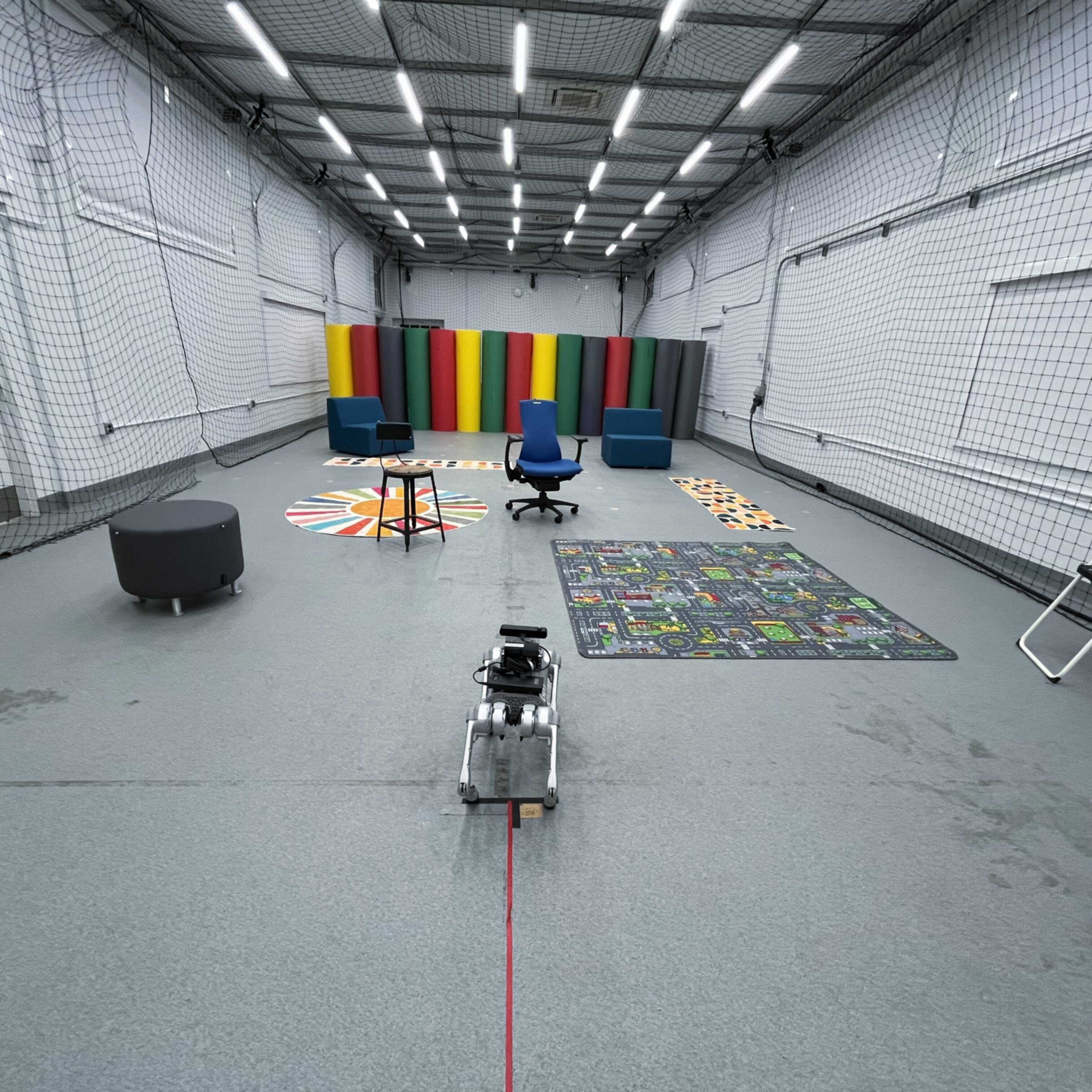}
    \end{minipage}
\end{minipage}
\vspace{12pt}
\begin{minipage}{\linewidth}
    \centering
    \begin{minipage}{0.28\linewidth}
        \centering
        \includegraphics[width=\linewidth]{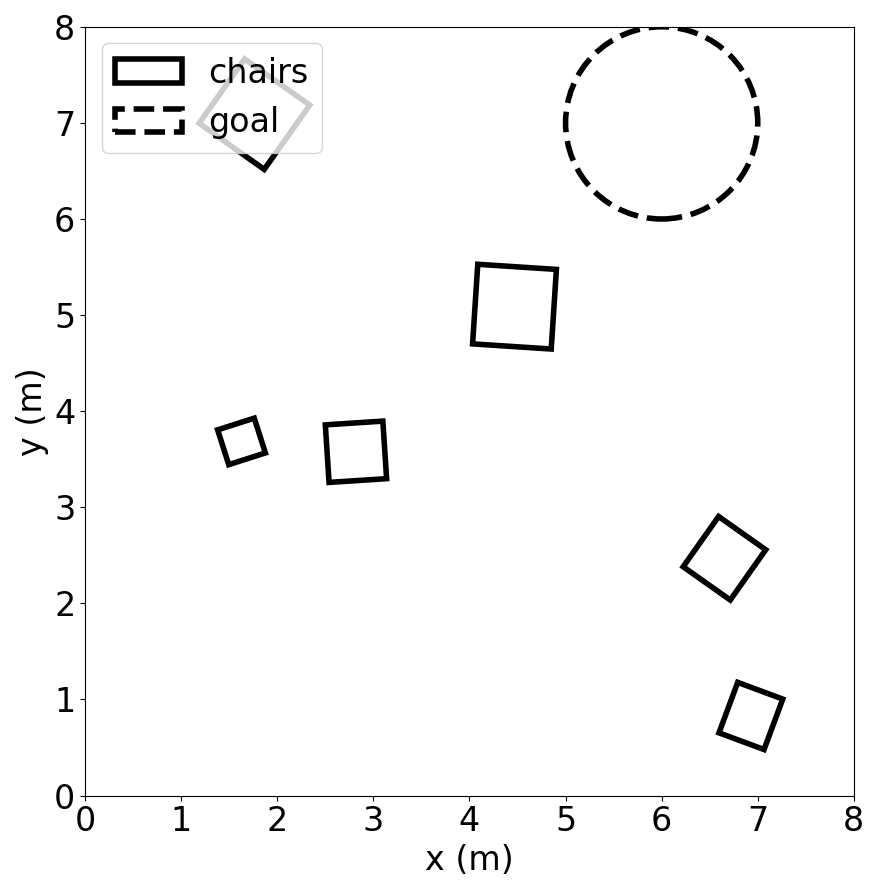}
        \caption*{(10) Environment 10}
    \end{minipage}
    \hfill
    \begin{minipage}{0.28\linewidth}
        \centering
        \includegraphics[width=\linewidth]{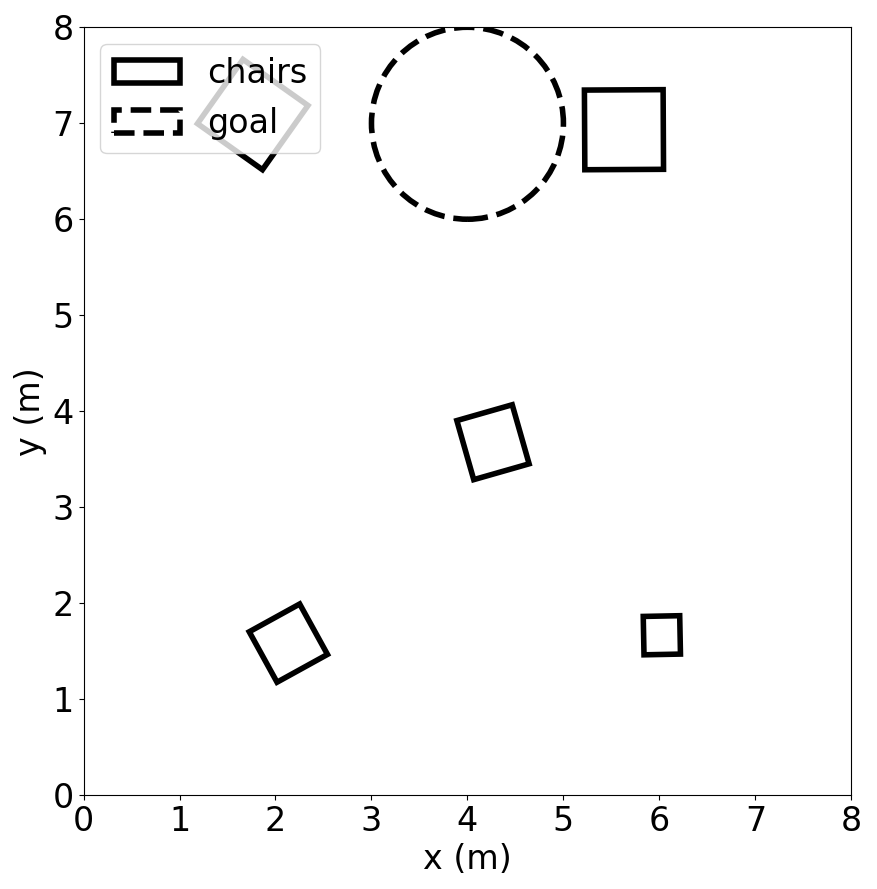}
        \caption*{(11) Environment 11}
    \end{minipage}
    \hfill
    \begin{minipage}{0.28\linewidth}
        \centering
        \includegraphics[width=\linewidth]{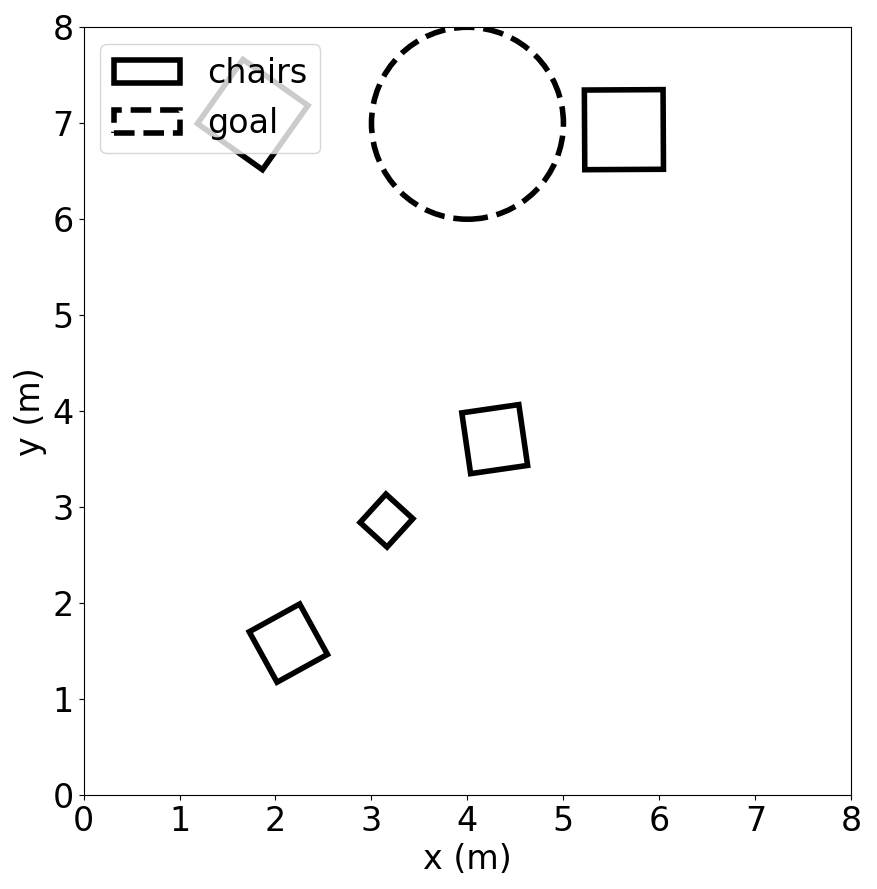}
        \caption*{(12) Environment 12}
    \end{minipage}
\end{minipage}
\end{figure*}

\begin{figure*}[h]
\centering
\begin{minipage}{\linewidth}
    \centering
    \begin{minipage}{0.28\linewidth}
        \centering
        \includegraphics[width=\linewidth]{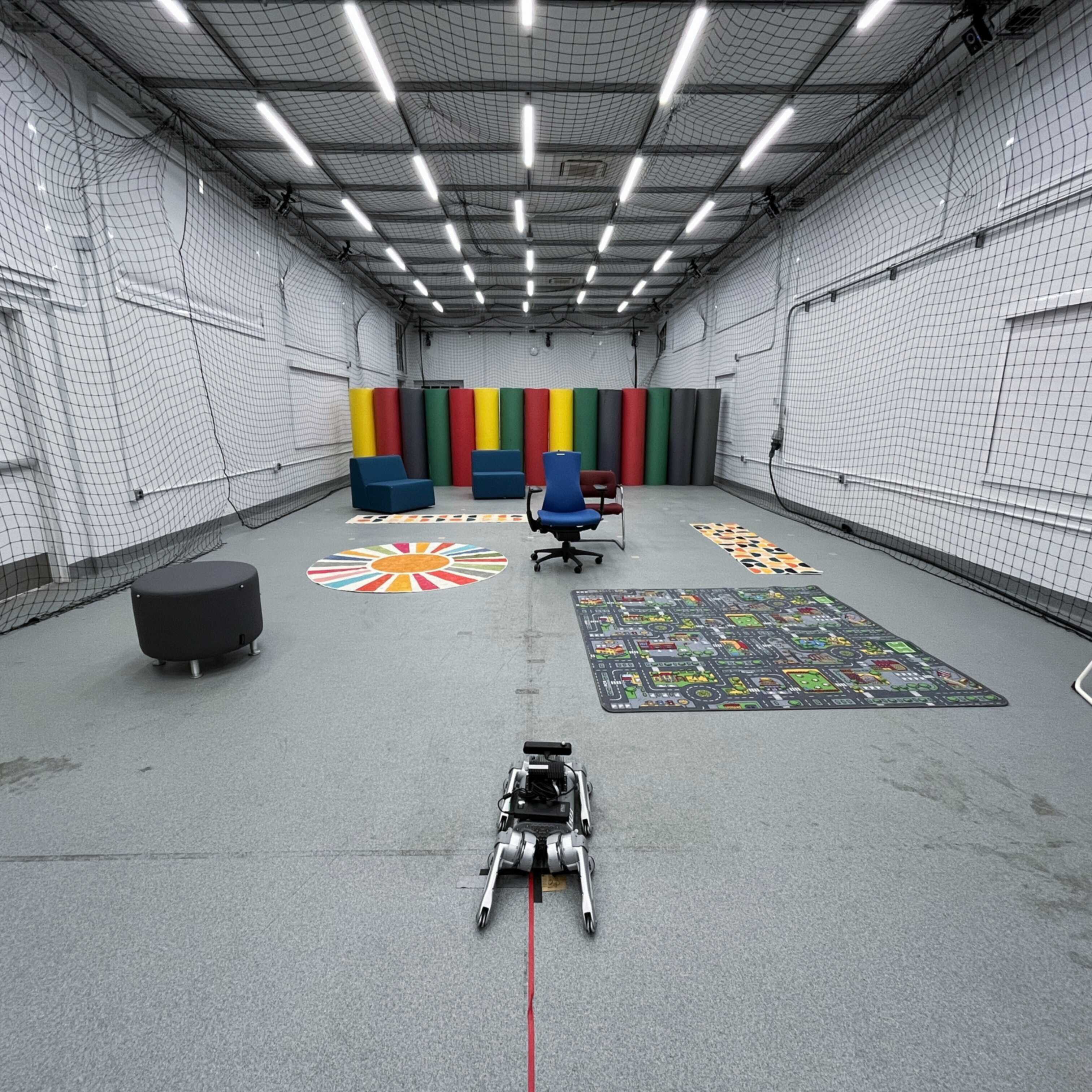}
    \end{minipage}
    \hfill
    \begin{minipage}{0.28\linewidth}
        \centering
        \includegraphics[width=\linewidth]{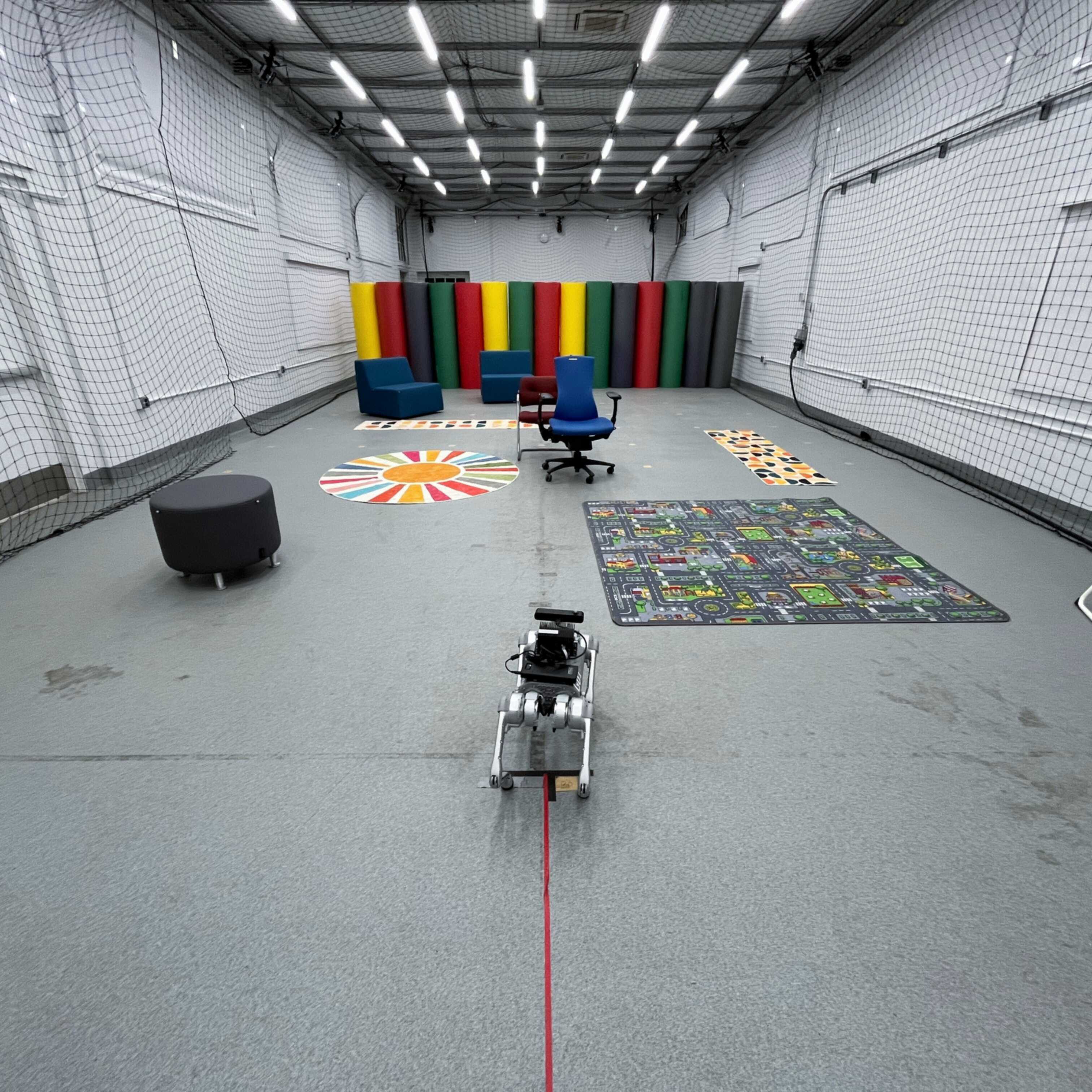}
    \end{minipage}
    \hfill
    \begin{minipage}{0.28\linewidth}
        \centering
        \includegraphics[width=\linewidth]{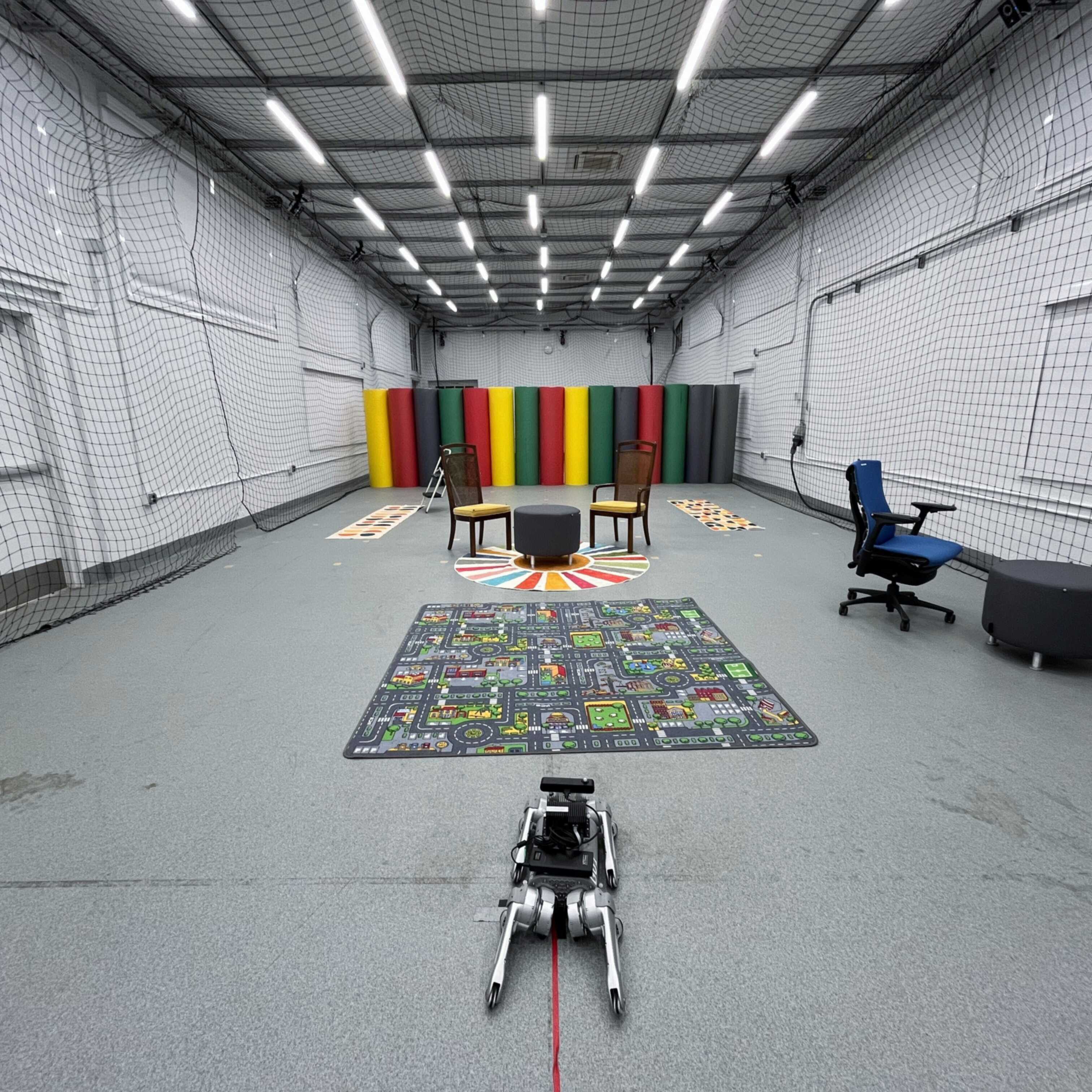}
    \end{minipage}
\end{minipage}
\vspace{12pt}
\begin{minipage}{\linewidth}
    \centering
    \begin{minipage}{0.28\linewidth}
        \centering
        \includegraphics[width=\linewidth]{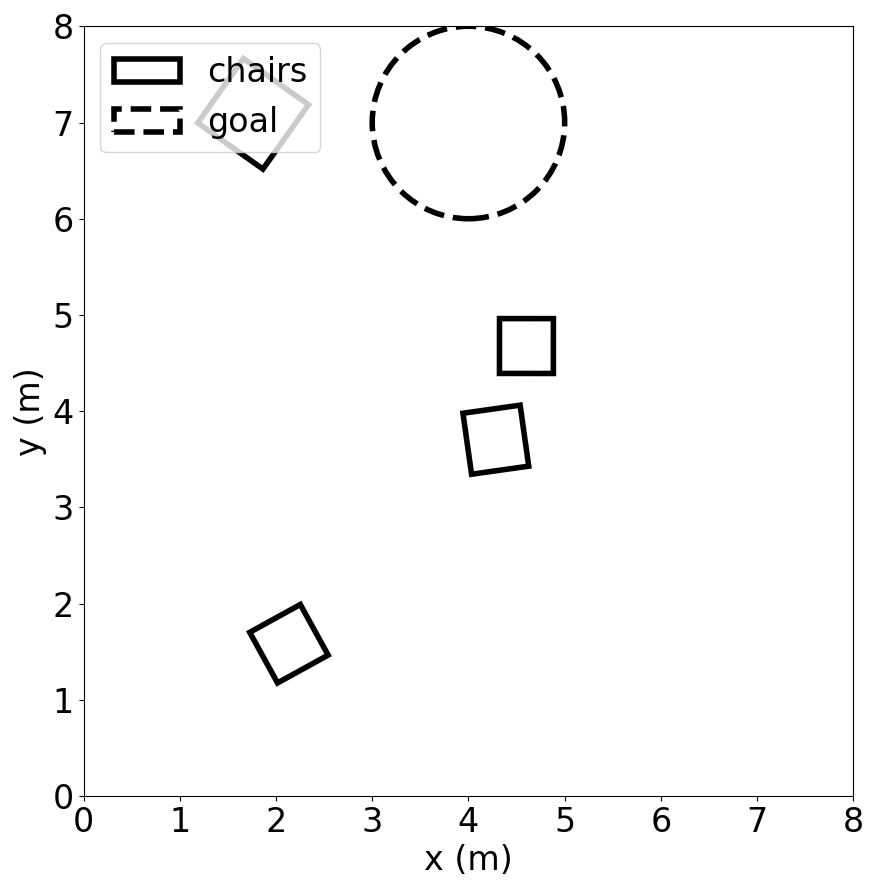}
        \caption*{(13) Environment 13}
    \end{minipage}
    \hfill
    \begin{minipage}{0.28\linewidth}
        \centering
        \includegraphics[width=\linewidth]{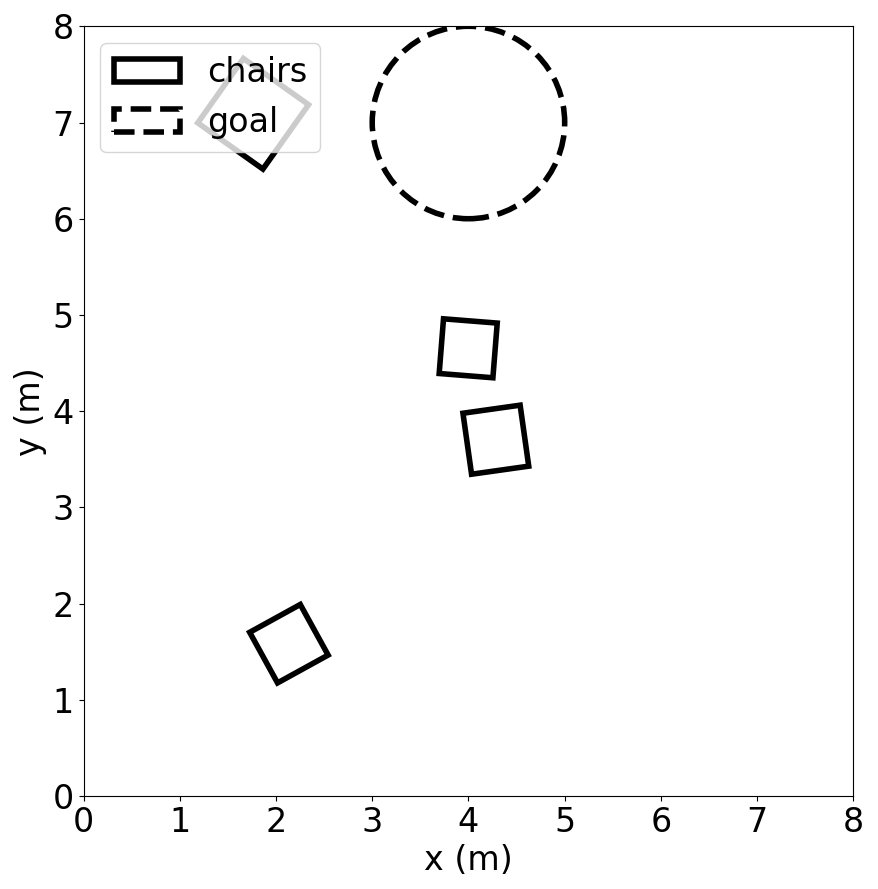}
        \caption*{(14) Environment 14}
    \end{minipage}
    \hfill
    \begin{minipage}{0.28\linewidth}
        \centering
        \includegraphics[width=\linewidth]{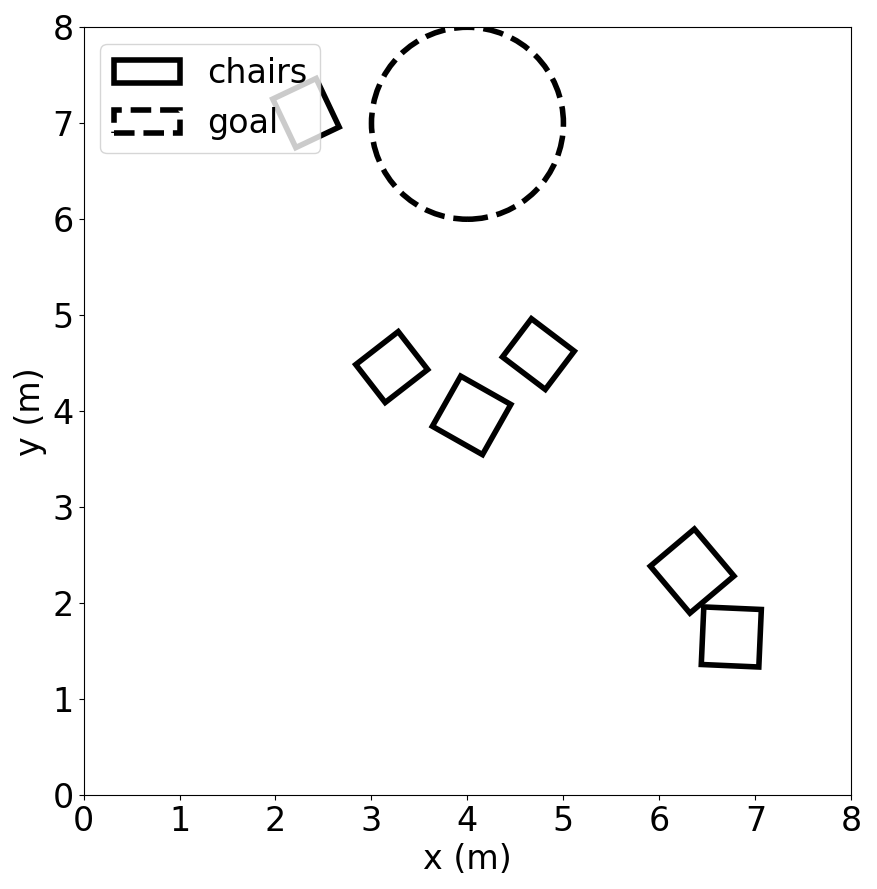}
        \caption*{(15) Environment 15}
    \end{minipage}
\end{minipage}
\end{figure*}

\begin{figure*}[h]
\centering
\begin{minipage}{\linewidth}
    \centering
    \begin{minipage}{0.28\linewidth}
        \centering
        \includegraphics[width=\linewidth]{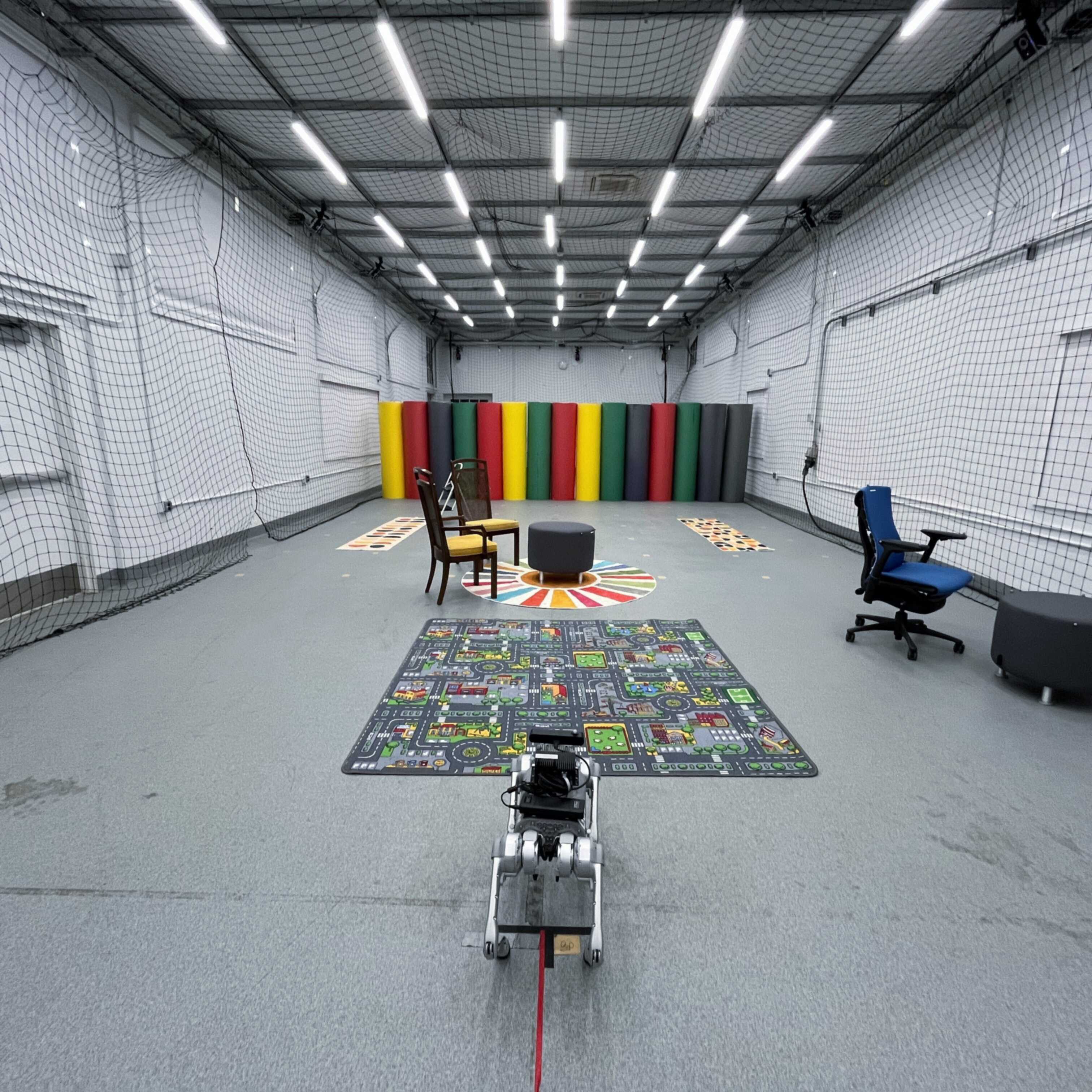}
    \end{minipage}
    \hfill
    \begin{minipage}{0.28\linewidth}
        \centering
        \includegraphics[width=\linewidth]{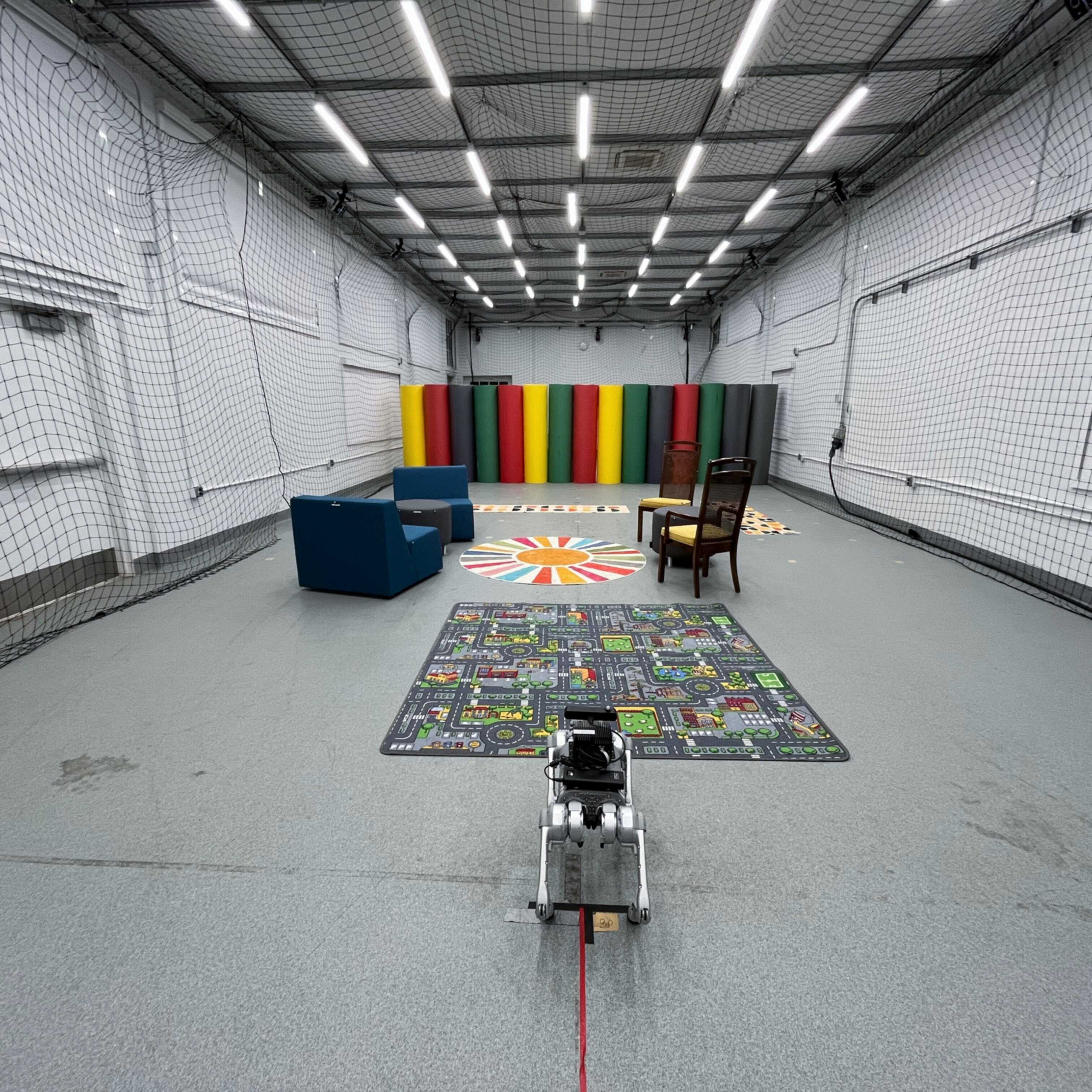}
    \end{minipage}
    \hfill
    \begin{minipage}{0.28\linewidth}
        \centering
        \includegraphics[width=\linewidth]{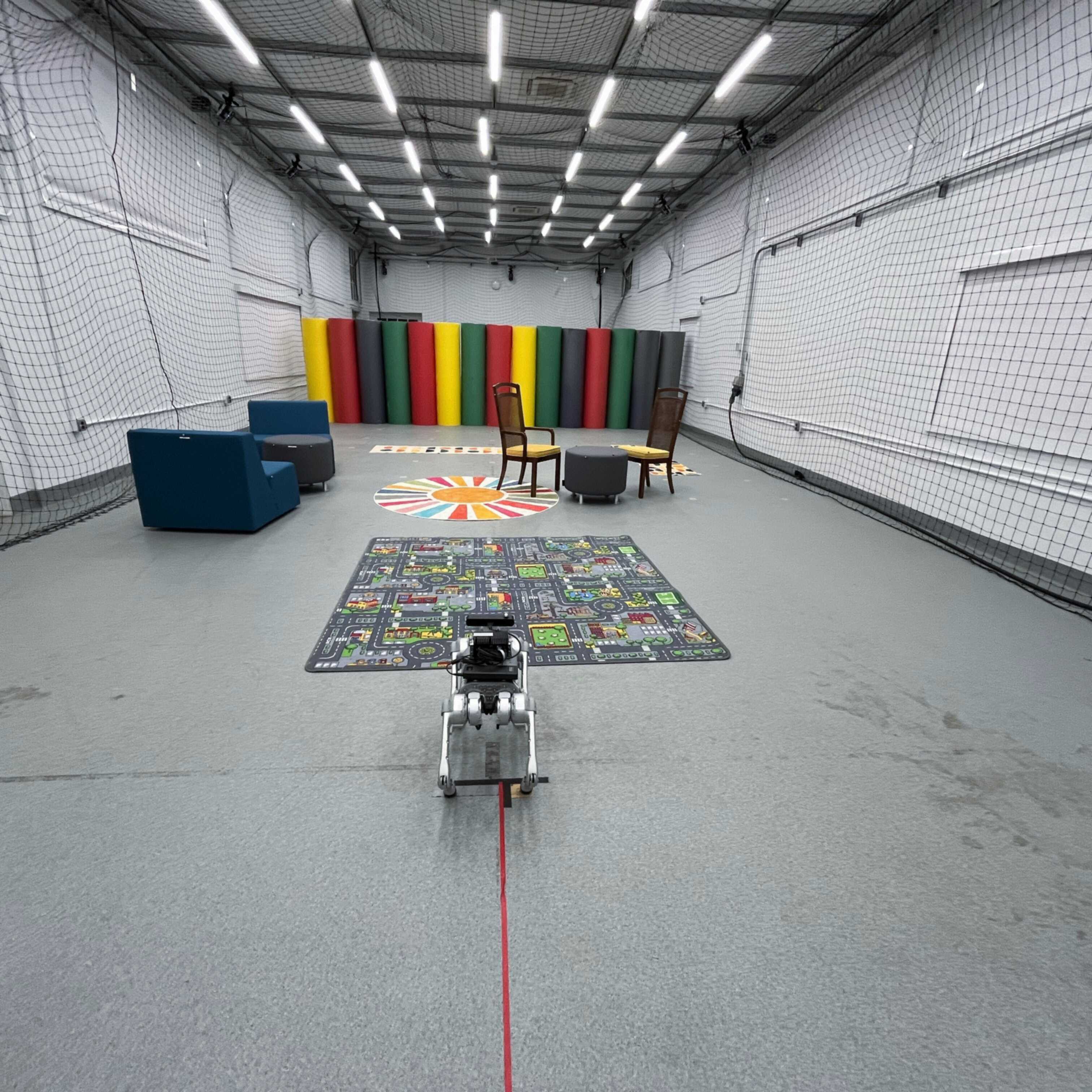}
    \end{minipage}
\end{minipage}
\vspace{12pt}
\begin{minipage}{\linewidth}
    \centering
    \begin{minipage}{0.28\linewidth}
        \centering
        \includegraphics[width=\linewidth]{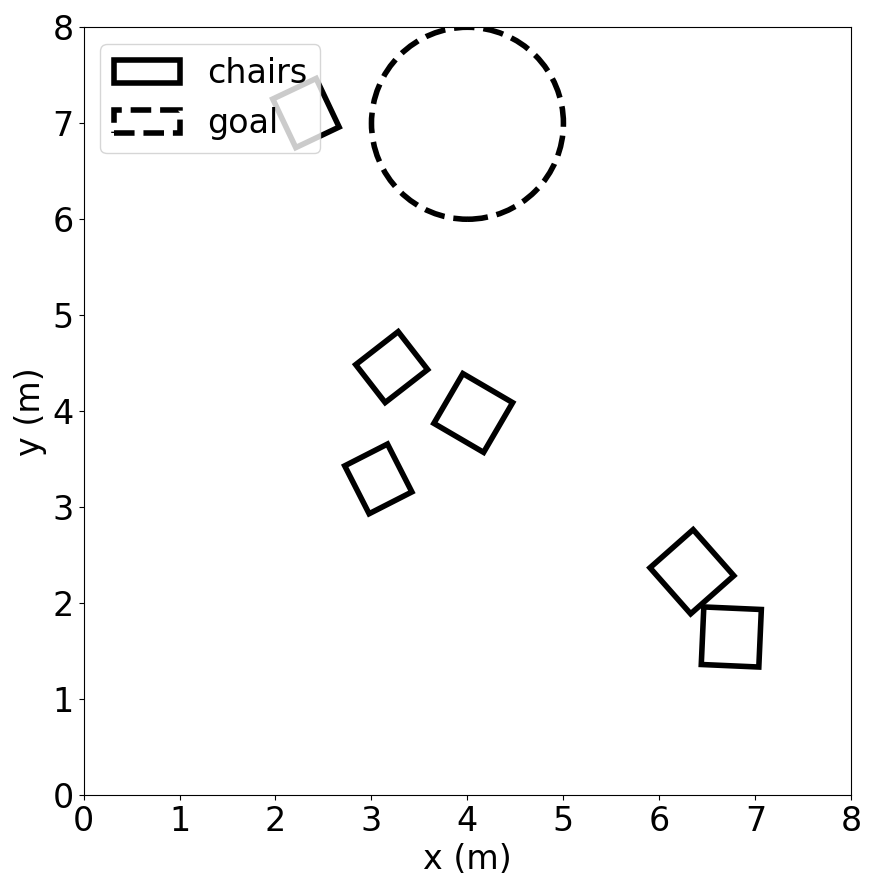}
        \caption*{(16) Environment 16}
    \end{minipage}
    \hfill
    \begin{minipage}{0.28\linewidth}
        \centering
        \includegraphics[width=\linewidth]{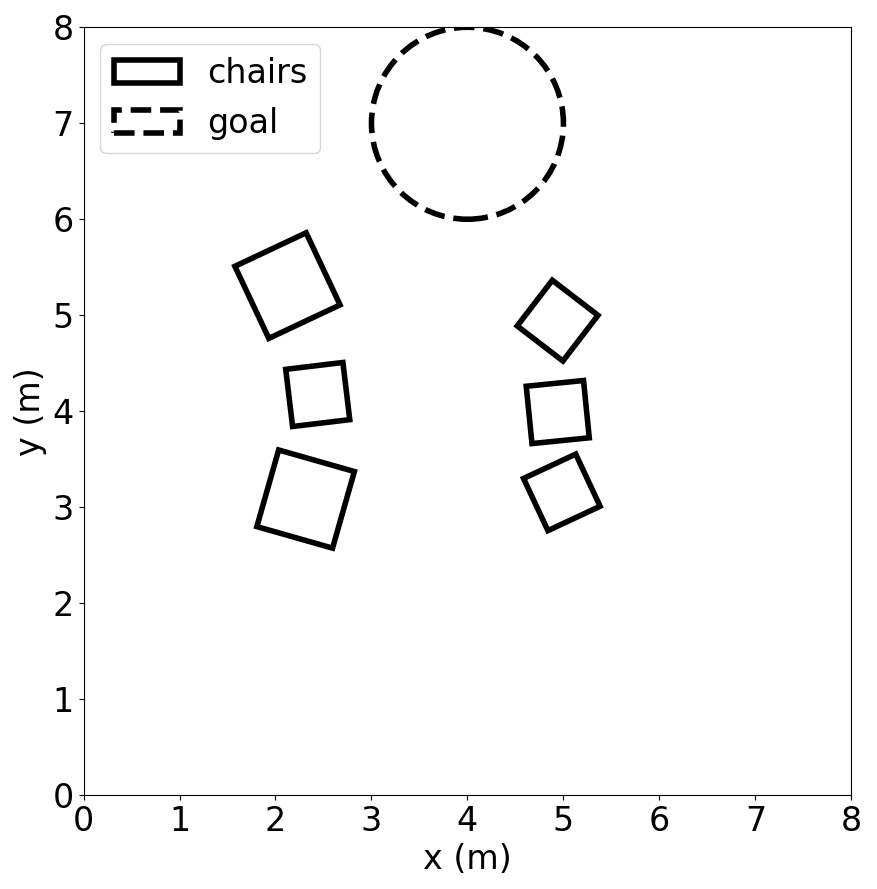}
        \caption*{(17) Environment 17}
    \end{minipage}
    \hfill
    \begin{minipage}{0.28\linewidth}
        \centering
        \includegraphics[width=\linewidth]{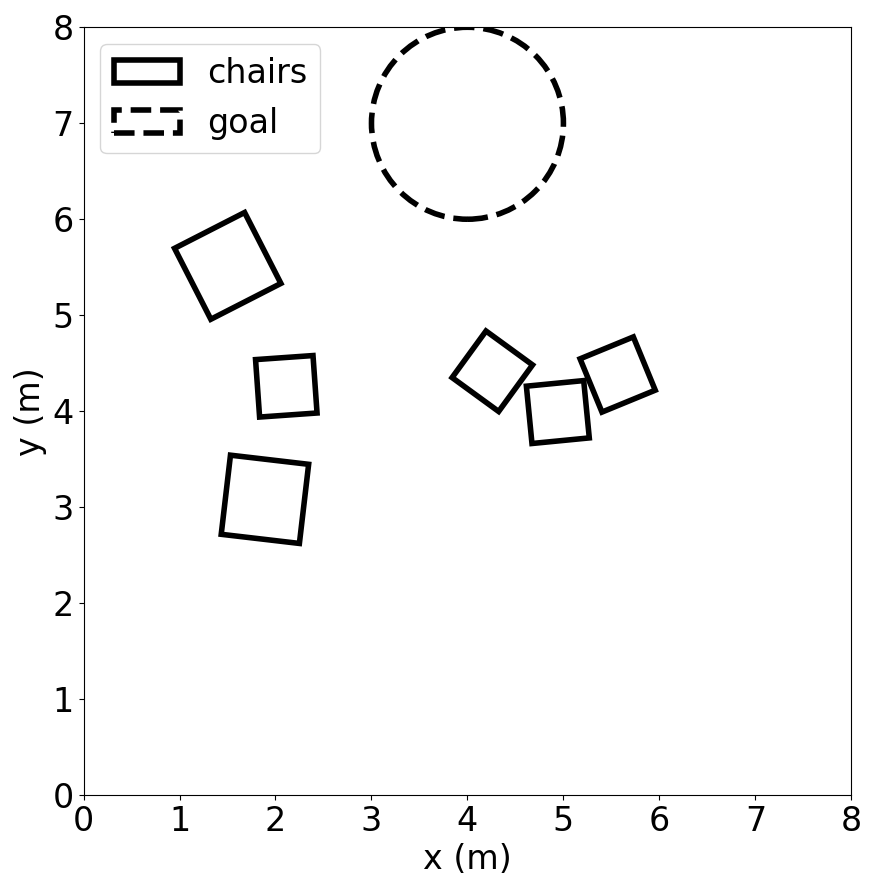}
        \caption*{(18) Environment 18}
    \end{minipage}
\end{minipage}
\end{figure*}

\begin{figure*}[h]
\centering
\begin{minipage}{\linewidth}
    \centering
    \begin{minipage}{0.28\linewidth}
        \centering
        \includegraphics[width=\linewidth]{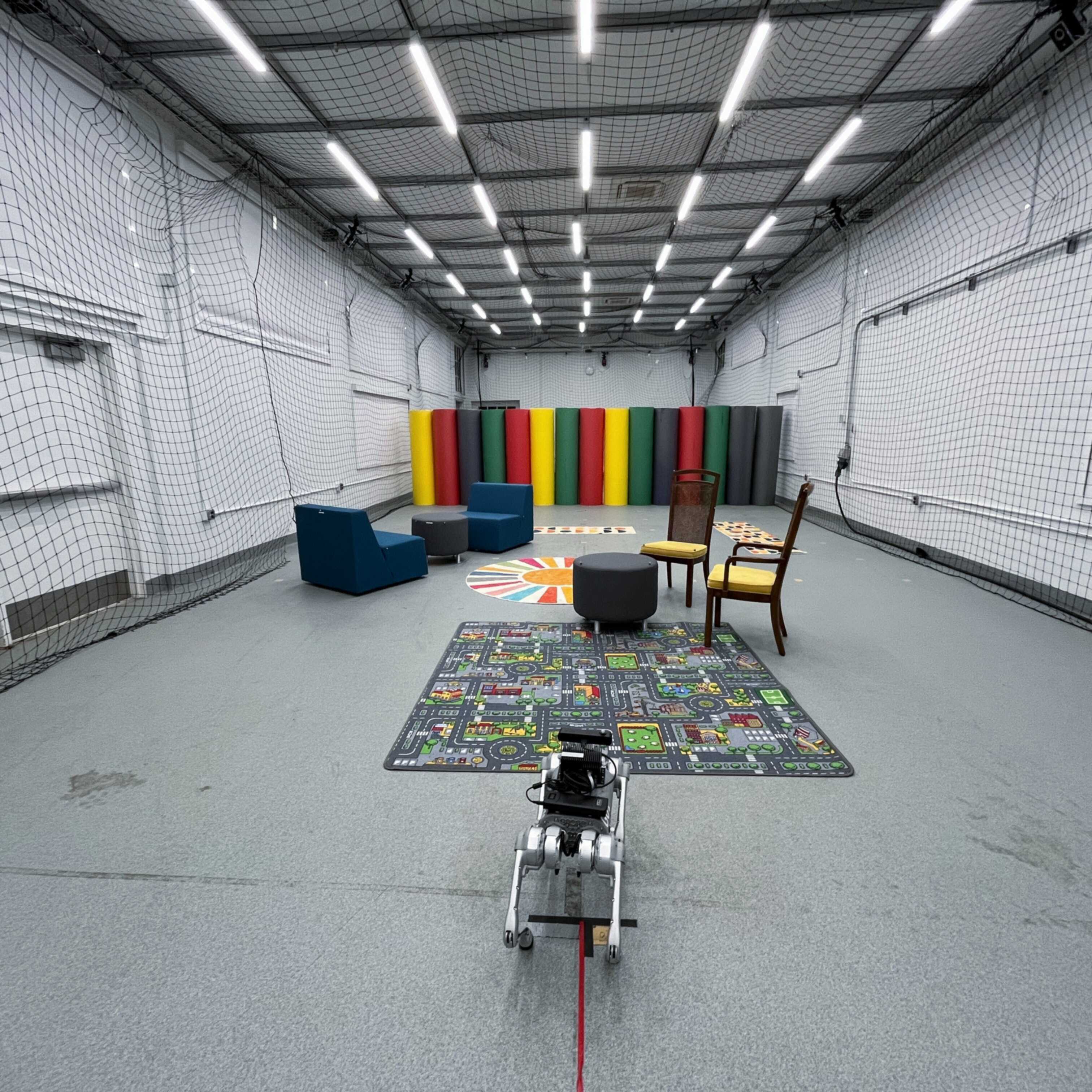}
    \end{minipage}
    \hfill
    \begin{minipage}{0.28\linewidth}
        \centering
        \includegraphics[width=\linewidth]{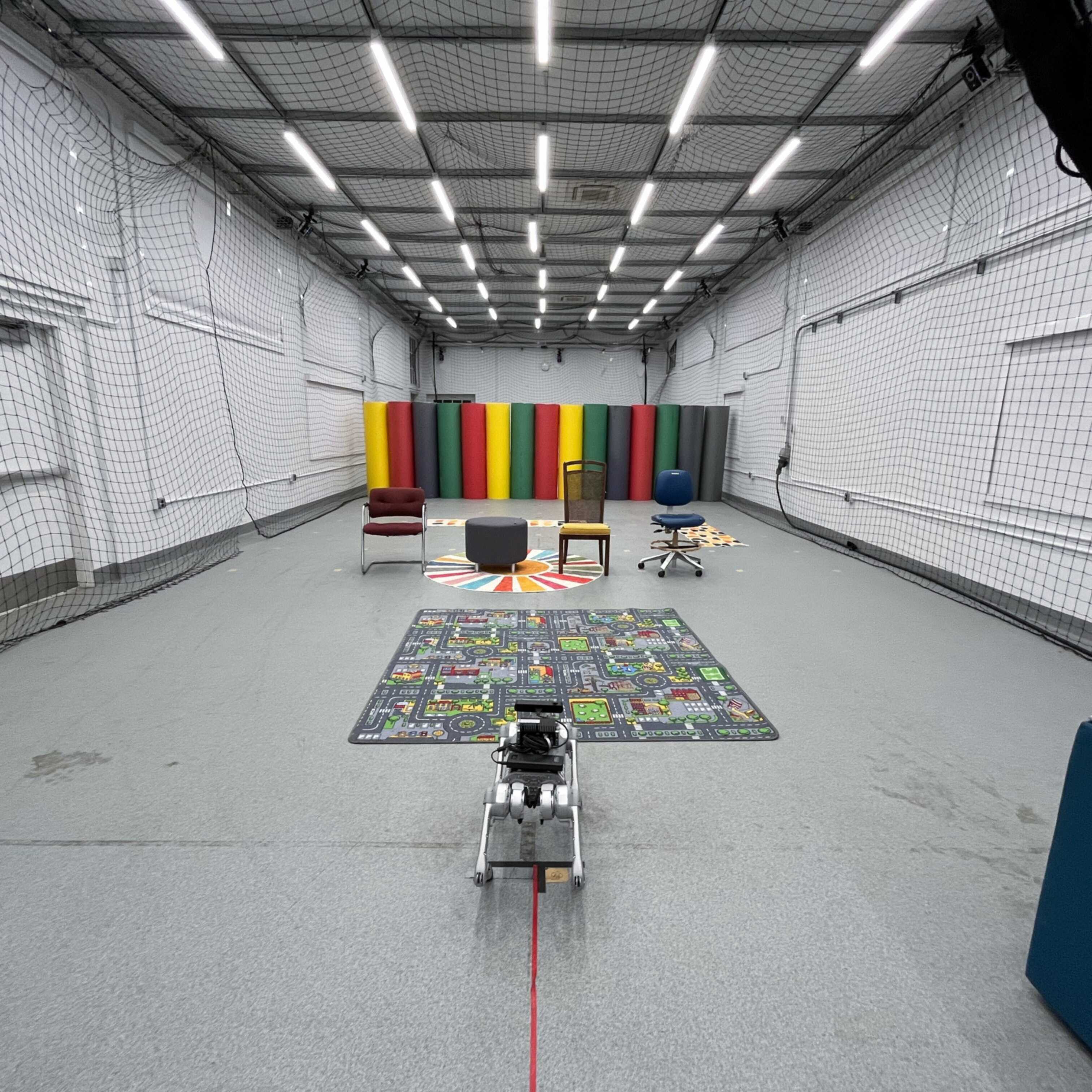}
    \end{minipage}
    \hfill
    \begin{minipage}{0.28\linewidth}
        \centering
        \includegraphics[width=\linewidth]{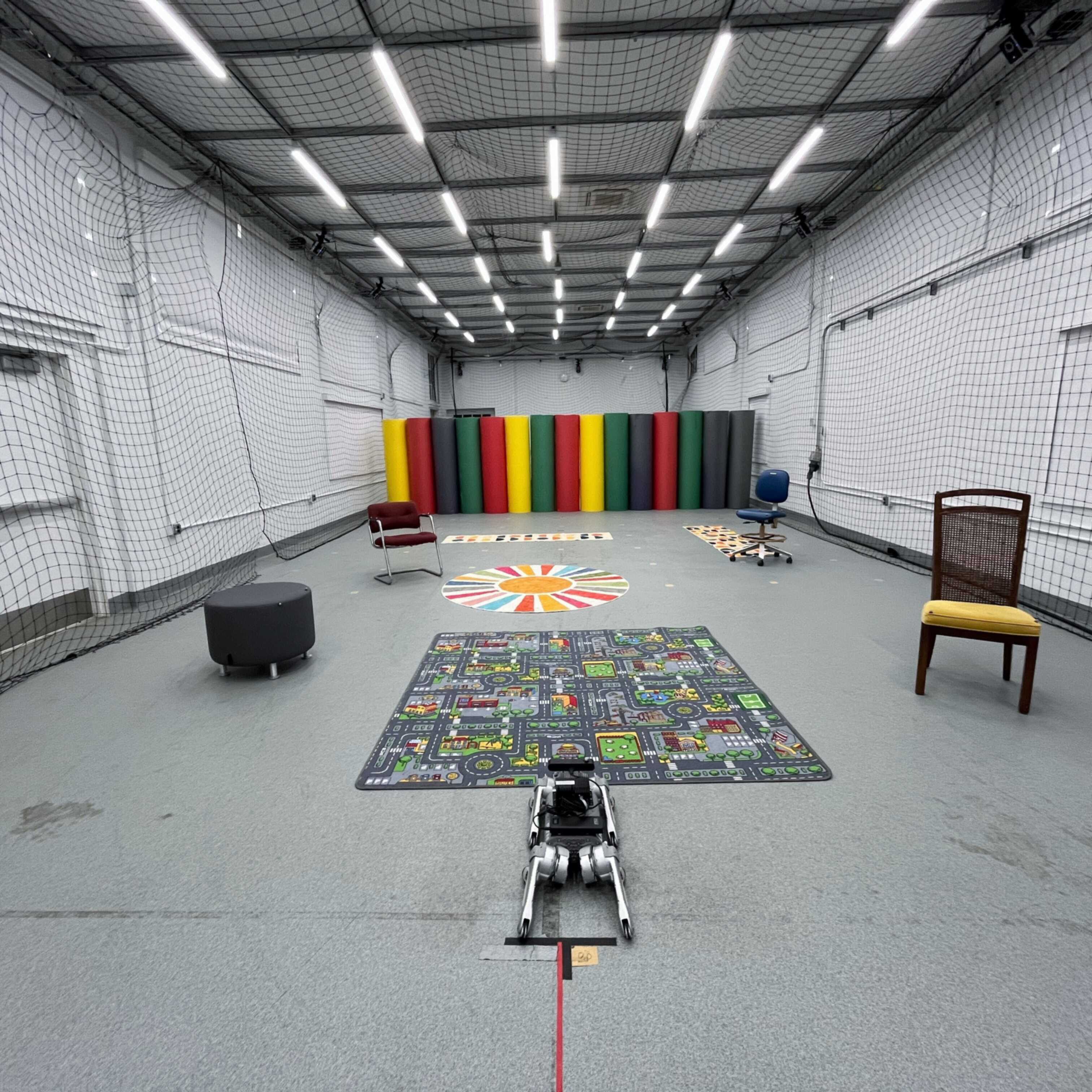}
    \end{minipage}
\end{minipage}
\vspace{12pt}
\begin{minipage}{\linewidth}
    \centering
    \begin{minipage}{0.28\linewidth}
        \centering
        \includegraphics[width=\linewidth]{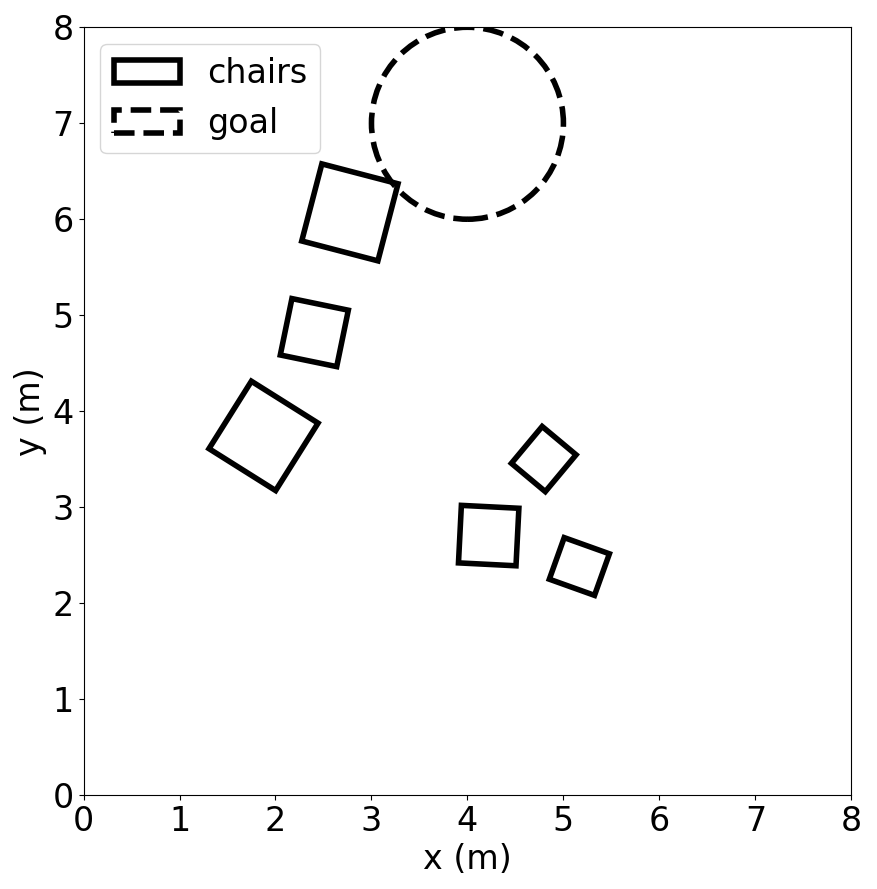}
        \caption*{(19) Environment 19}
    \end{minipage}
    \hfill
    \begin{minipage}{0.28\linewidth}
        \centering
        \includegraphics[width=\linewidth]{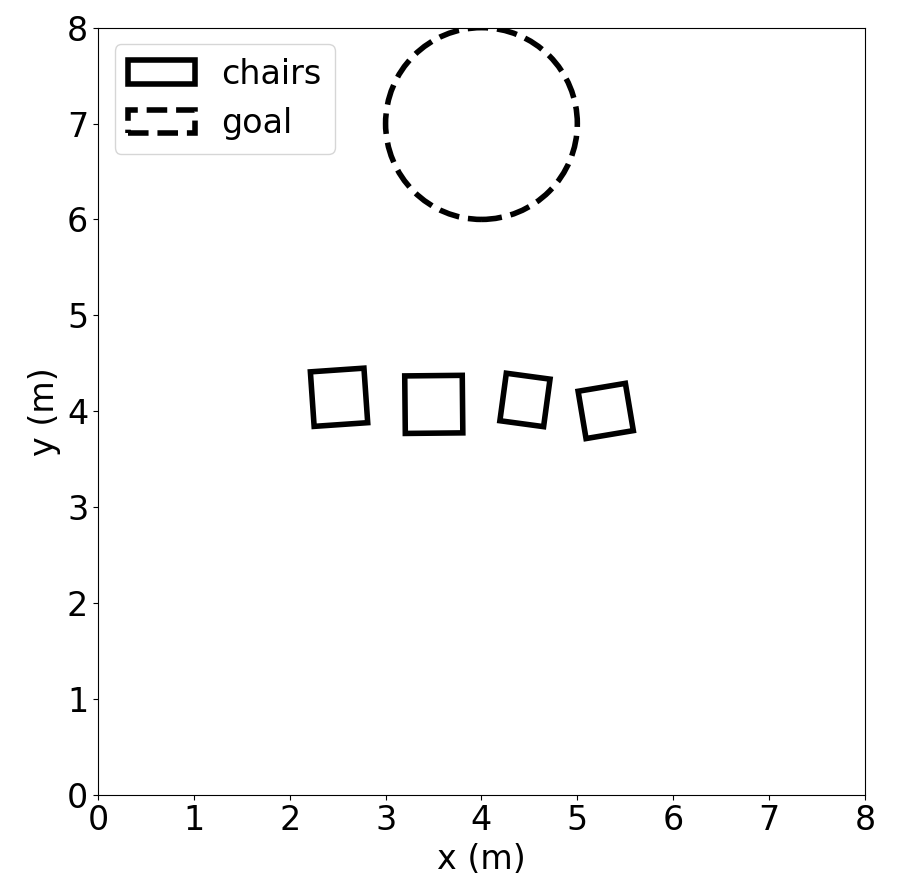}
        \caption*{(20) Environment 20}
    \end{minipage}
    \hfill
    \begin{minipage}{0.28\linewidth}
        \centering
        \includegraphics[width=\linewidth]{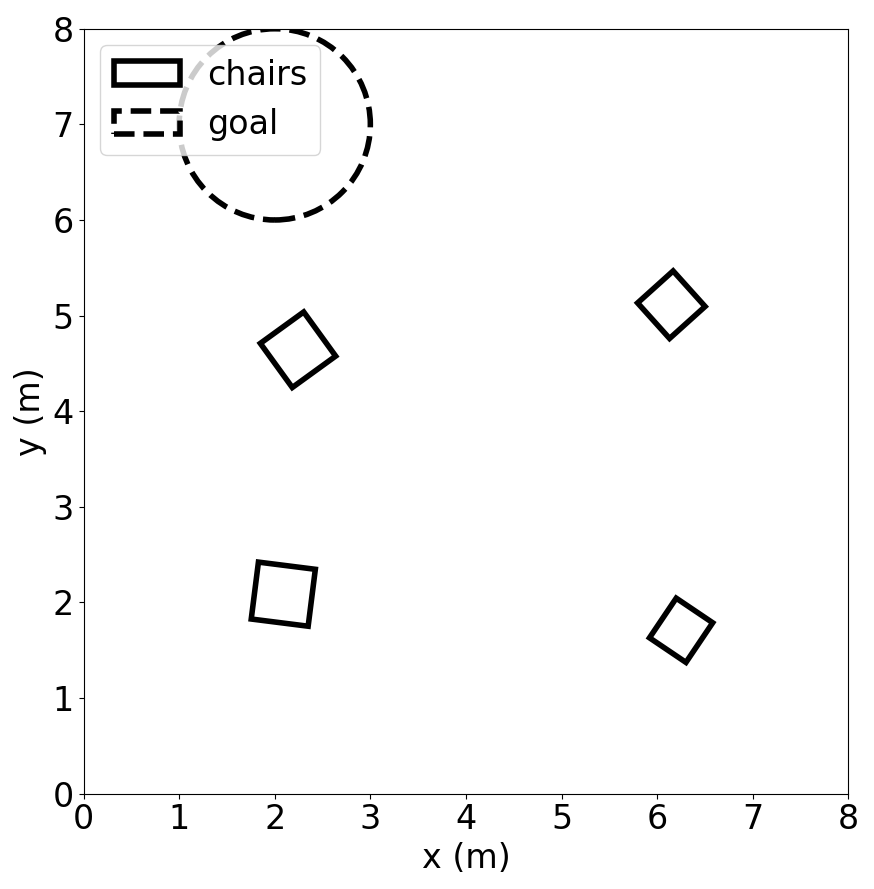}
        \caption*{(21) Environment 21}
    \end{minipage}
\end{minipage}
\end{figure*}

\begin{figure*}[h]
\centering
\begin{minipage}{\linewidth}
    \centering
    \begin{minipage}{0.28\linewidth}
        \centering
        \includegraphics[width=\linewidth]{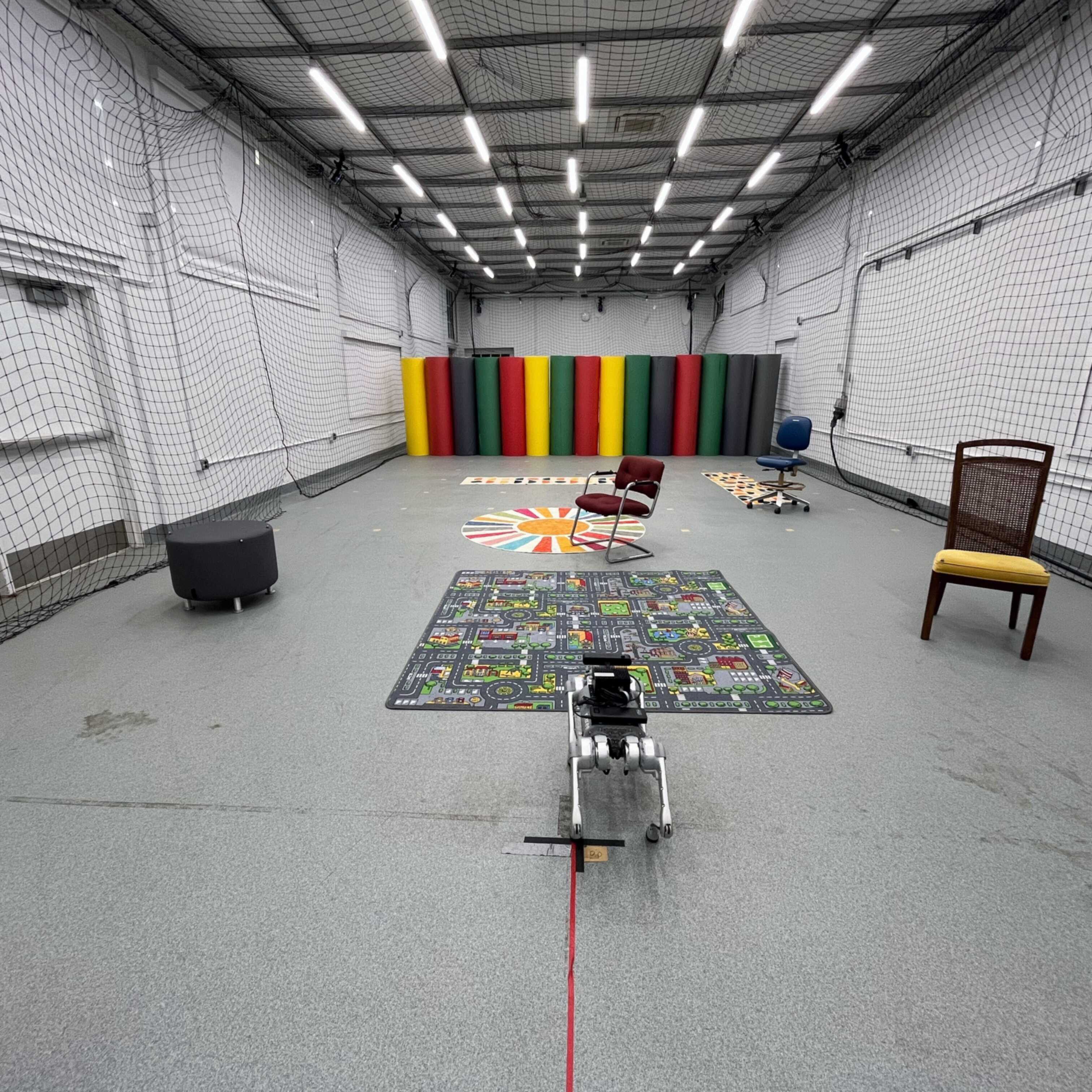}
    \end{minipage}
    \hfill
    \begin{minipage}{0.28\linewidth}
        \centering
        \includegraphics[width=\linewidth]{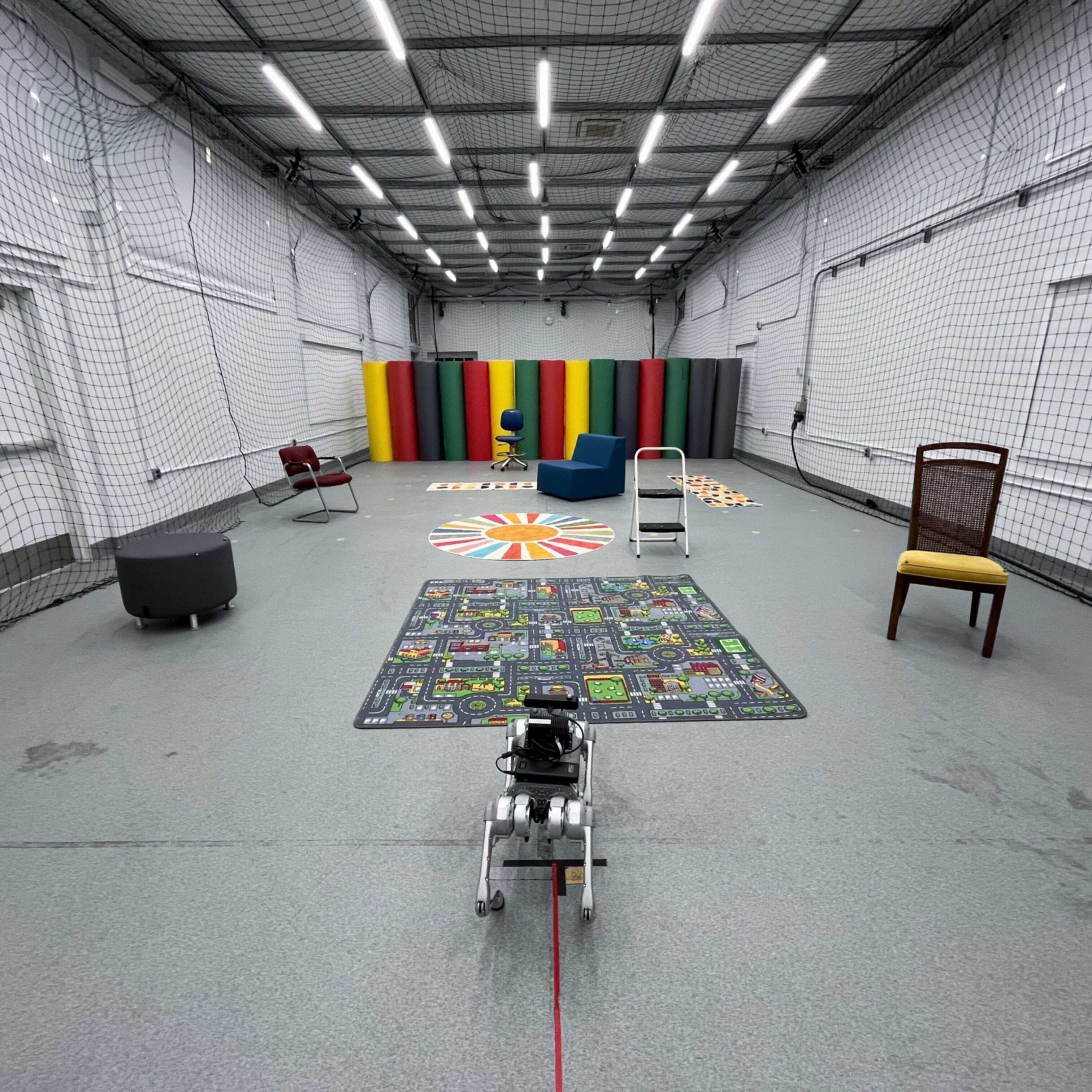}
    \end{minipage}
    \hfill
    \begin{minipage}{0.28\linewidth}
        \centering
        \includegraphics[width=\linewidth]{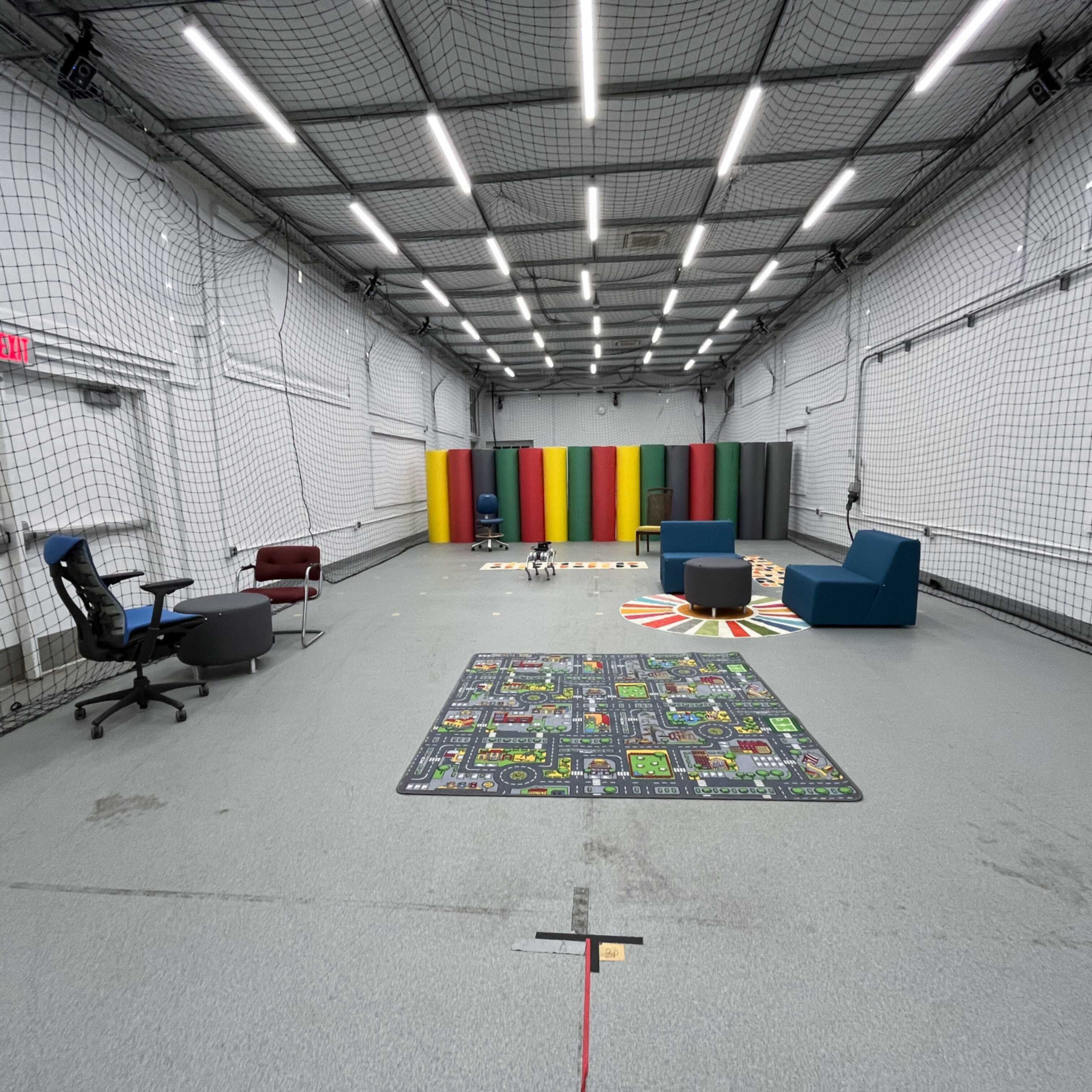}
    \end{minipage}
\end{minipage}
\vspace{12pt}
\begin{minipage}{\linewidth}
    \centering
    \begin{minipage}{0.28\linewidth}
        \centering
        \includegraphics[width=\linewidth]{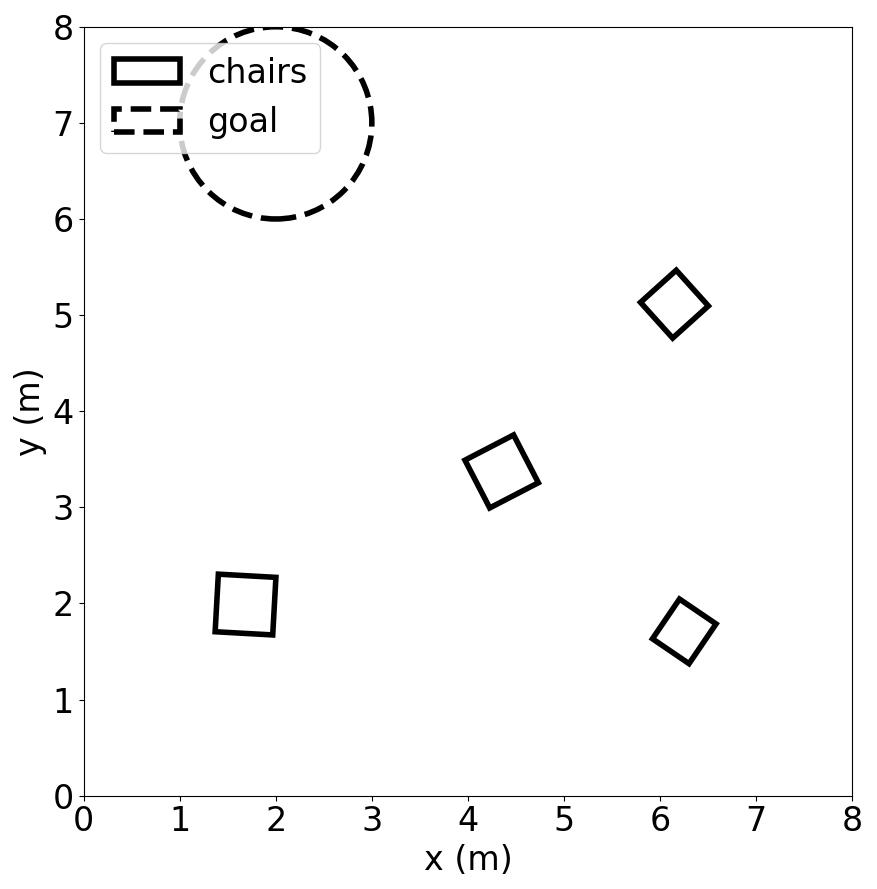}
        \caption*{(22) Environment 22}
    \end{minipage}
    \hfill
    \begin{minipage}{0.28\linewidth}
        \centering
        \includegraphics[width=\linewidth]{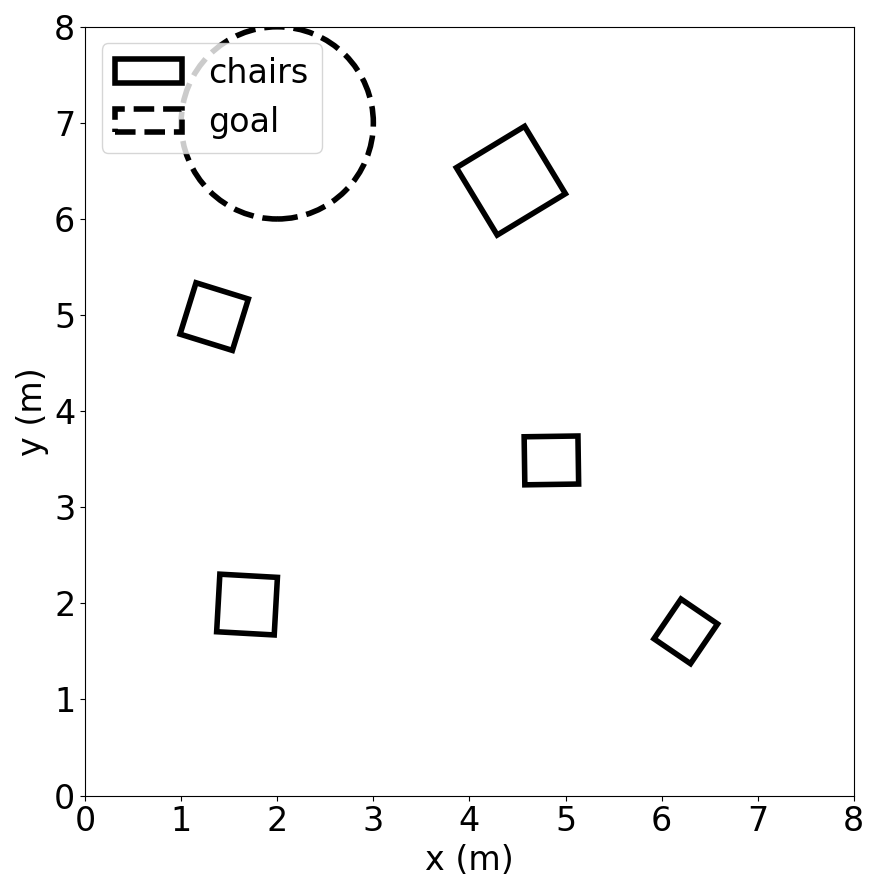}
        \caption*{(23) Environment 23}
    \end{minipage}
    \hfill
    \begin{minipage}{0.28\linewidth}
        \centering
        \includegraphics[width=\linewidth]{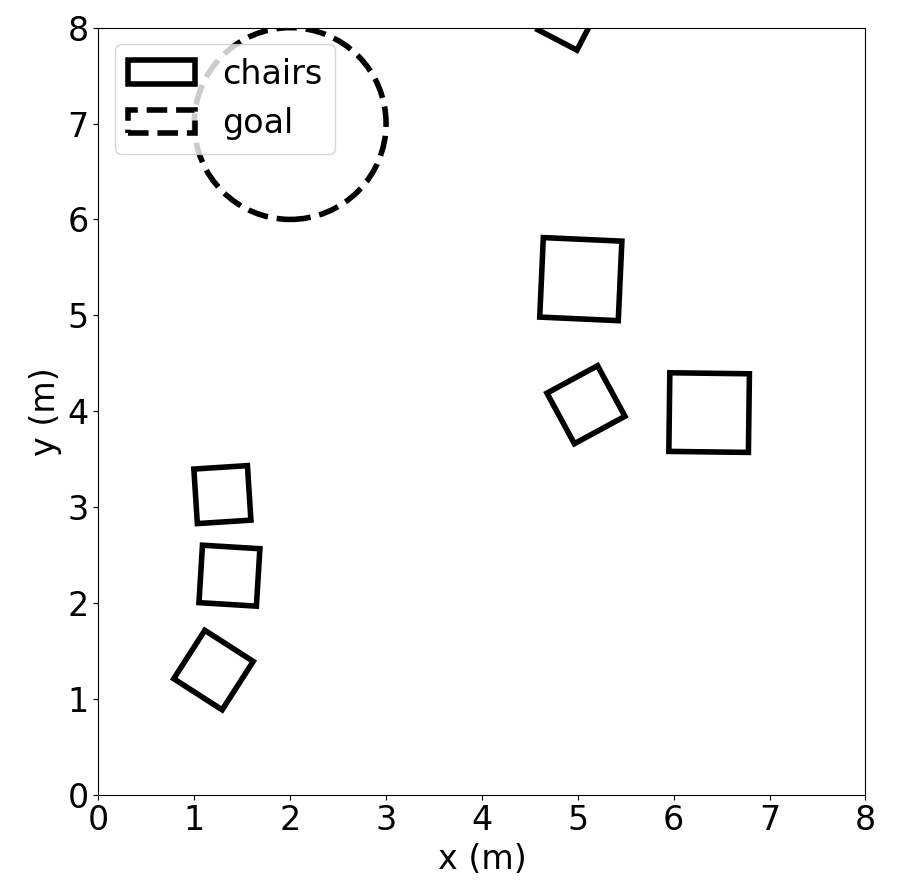}
        \caption*{(24) Environment 24}
    \end{minipage}
\end{minipage}
\end{figure*}

\begin{figure*}[h]
\centering
\begin{minipage}{\linewidth}
    \centering
    \begin{minipage}{0.28\linewidth}
        \centering
        \includegraphics[width=\linewidth]{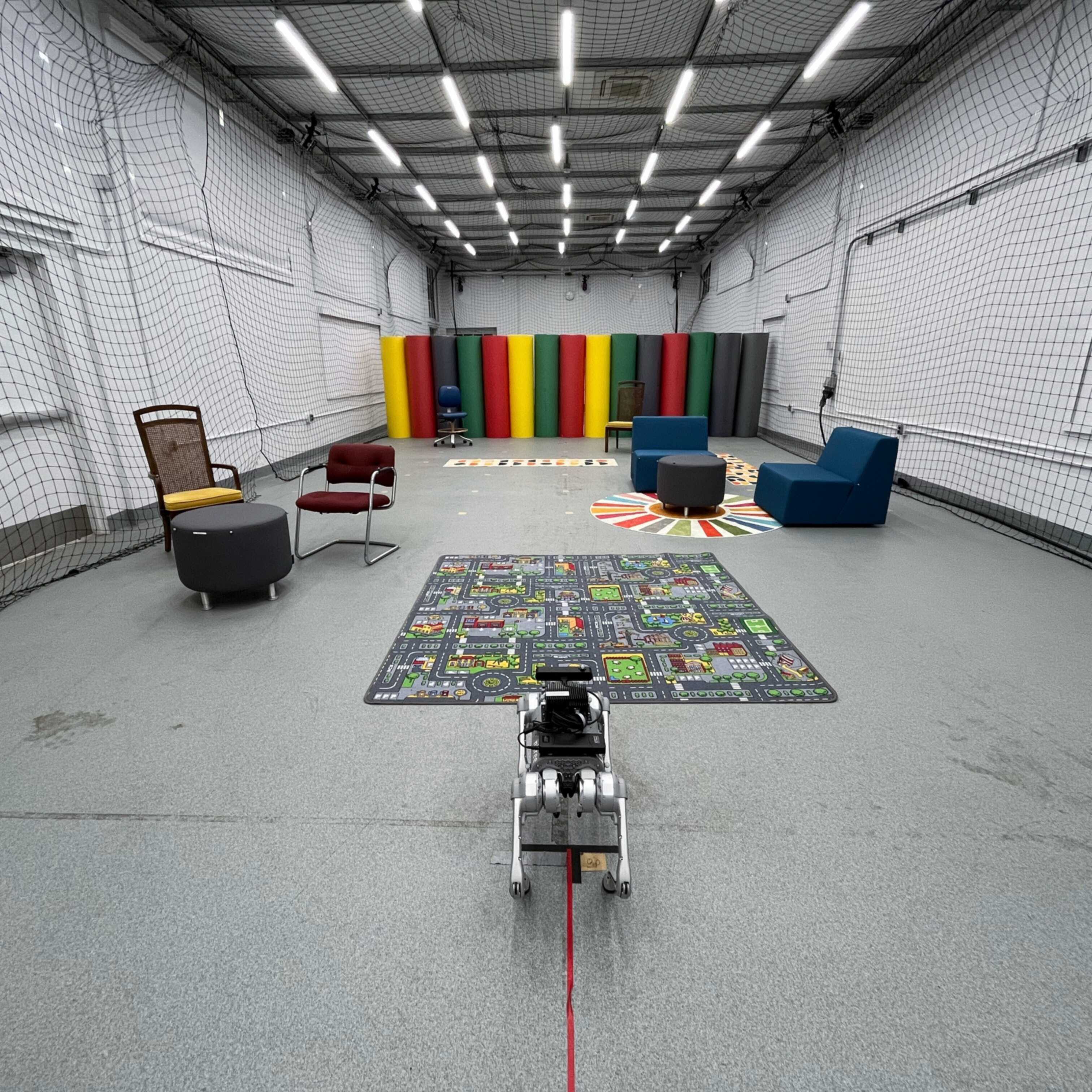}
    \end{minipage}
    \hfill
    \begin{minipage}{0.28\linewidth}
        \centering
        \includegraphics[width=\linewidth]{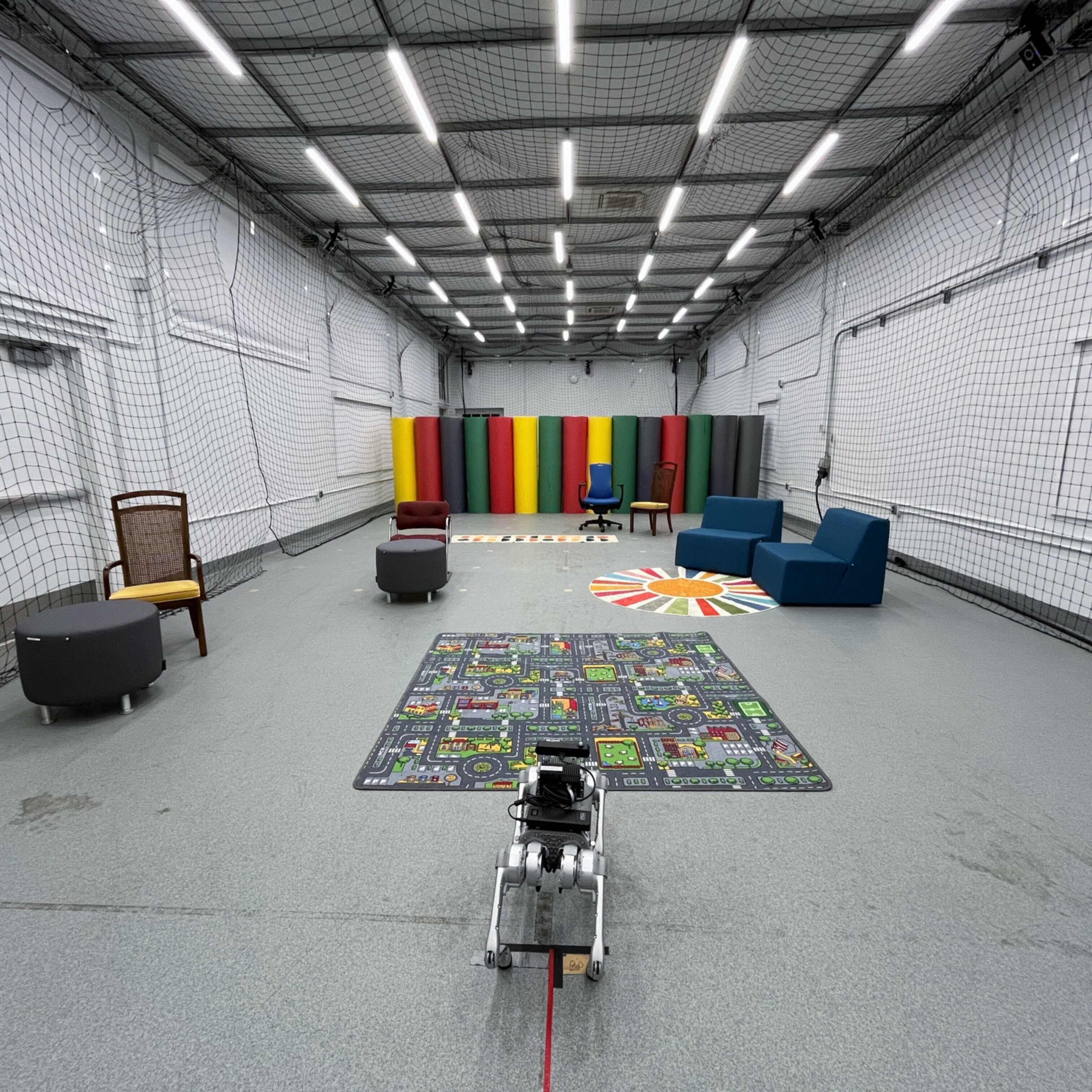}
    \end{minipage}
    \hfill
    \begin{minipage}{0.28\linewidth}
        \centering
        \includegraphics[width=\linewidth]{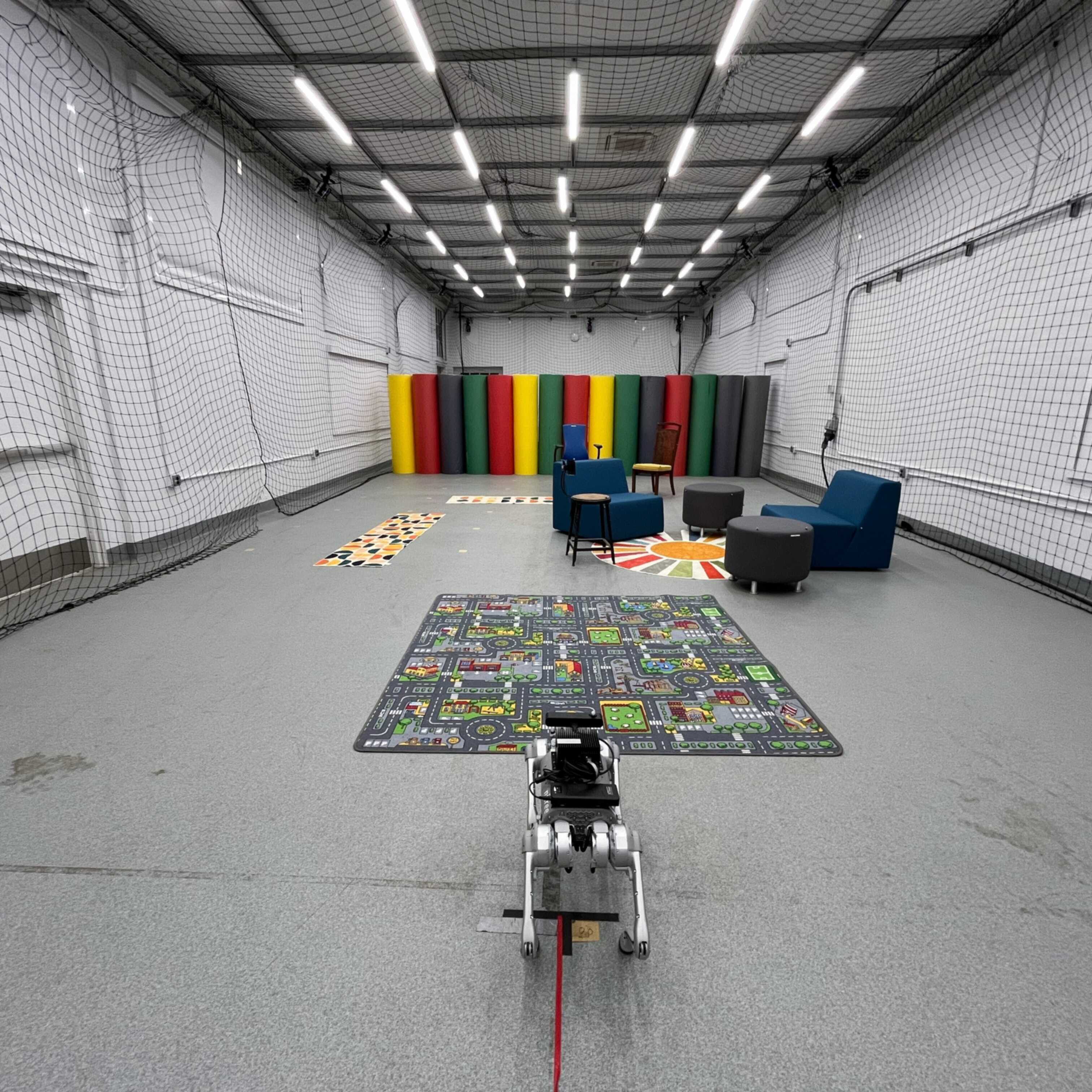}
    \end{minipage}
\end{minipage}
\vspace{12pt}
\begin{minipage}{\linewidth}
    \centering
    \begin{minipage}{0.28\linewidth}
        \centering
        \includegraphics[width=\linewidth]{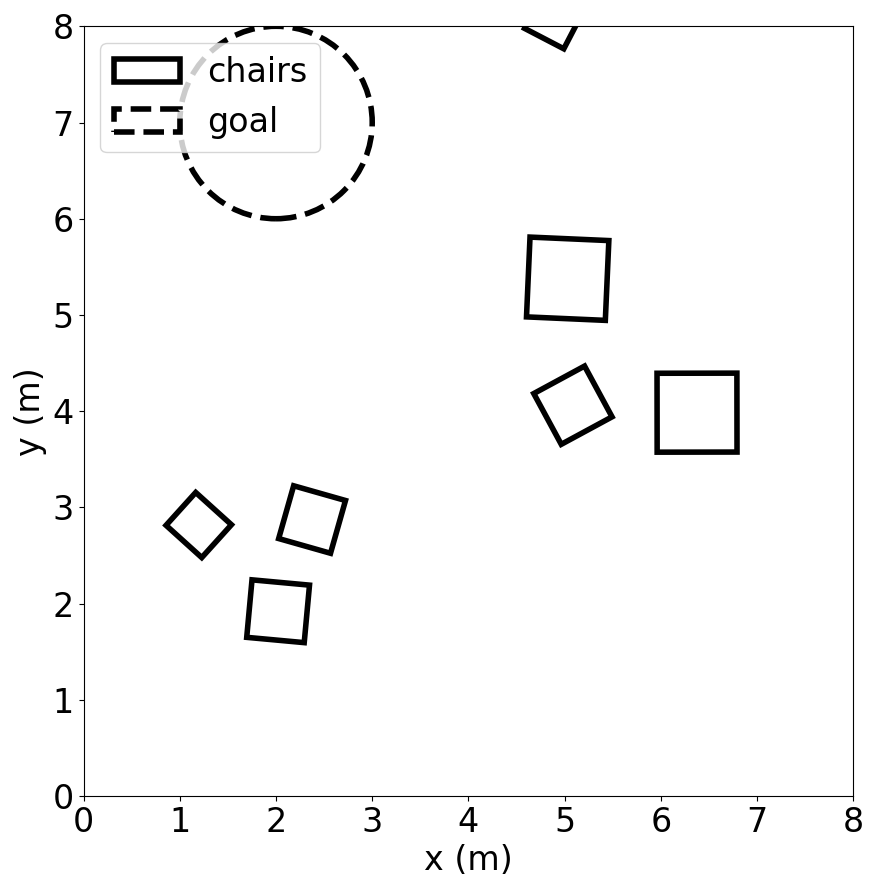}
        \caption*{(25) Environment 25}
    \end{minipage}
    \hfill
    \begin{minipage}{0.28\linewidth}
        \centering
        \includegraphics[width=\linewidth]{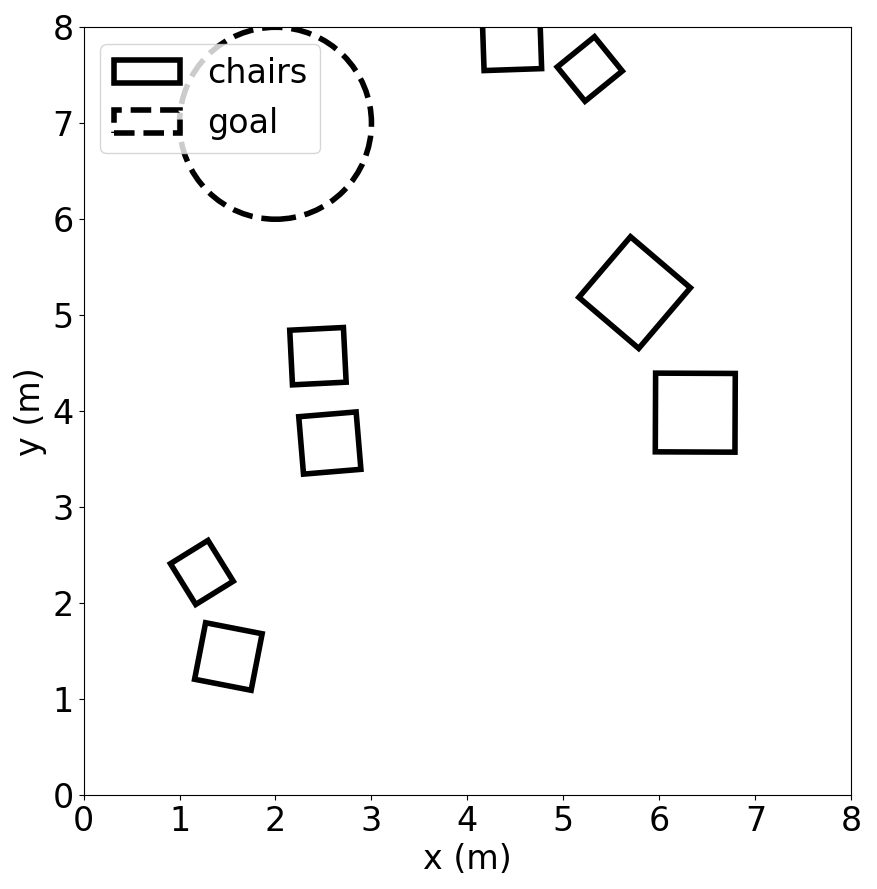}
        \caption*{(26) Environment 26}
    \end{minipage}
    \hfill
    \begin{minipage}{0.28\linewidth}
        \centering
        \includegraphics[width=\linewidth]{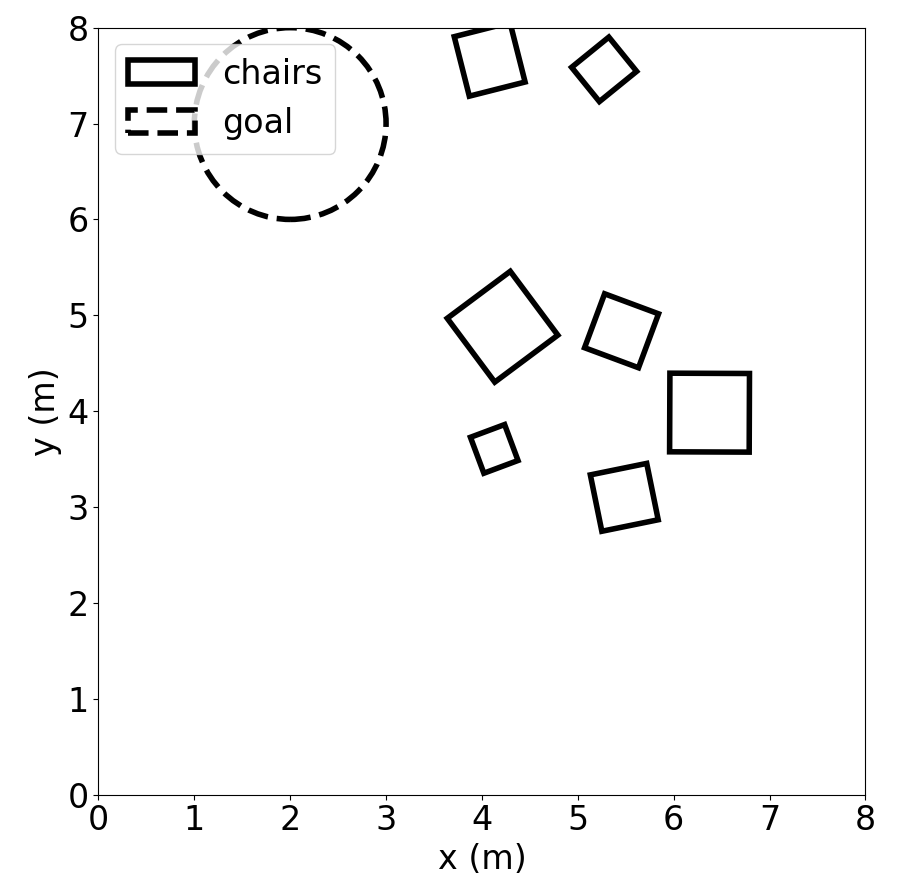}
        \caption*{(27) Environment 27}
    \end{minipage}
\end{minipage}
\end{figure*}

\begin{figure*}[h]
\centering
\begin{minipage}{\linewidth}
    \centering
    \begin{minipage}{0.28\linewidth}
        \centering
        \includegraphics[width=\linewidth]{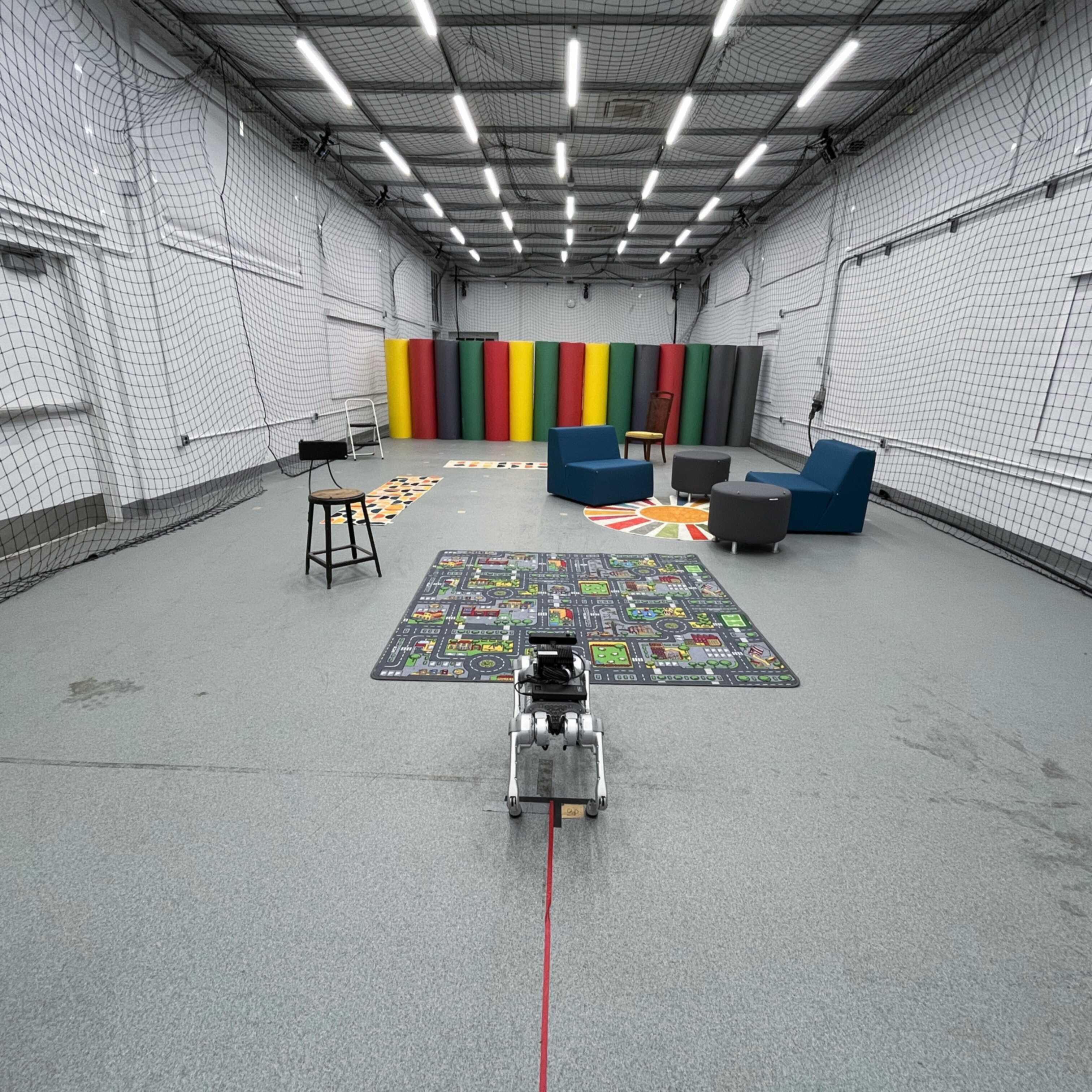}
    \end{minipage}
    \hfill
    \begin{minipage}{0.28\linewidth}
        \centering
        \includegraphics[width=\linewidth]{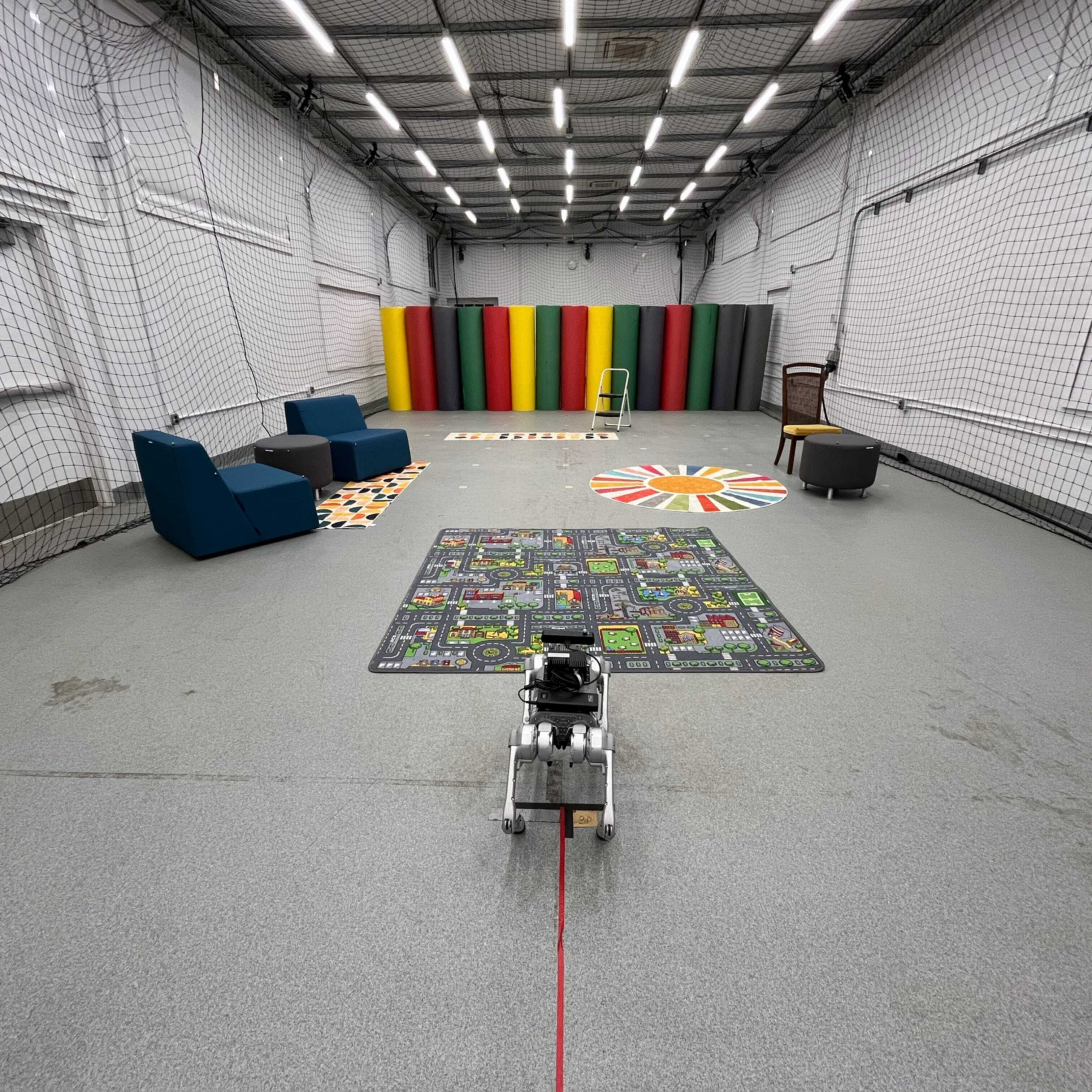}
    \end{minipage}
    \hfill
    \begin{minipage}{0.28\linewidth}
        \centering
        \includegraphics[width=\linewidth]{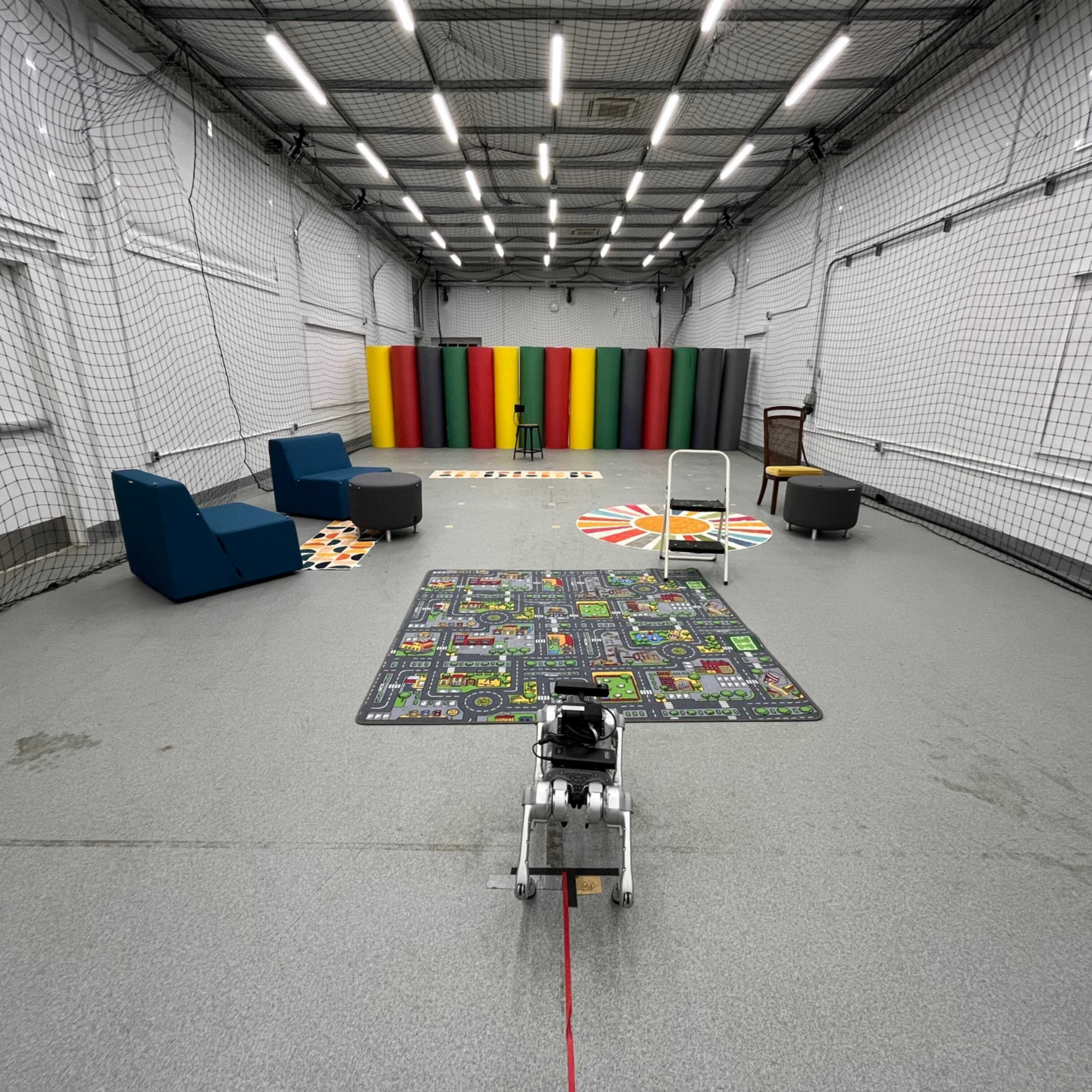}
    \end{minipage}
\end{minipage}
\vspace{12pt}
\begin{minipage}{\linewidth}
    \centering
    \begin{minipage}{0.28\linewidth}
        \centering
        \includegraphics[width=\linewidth]{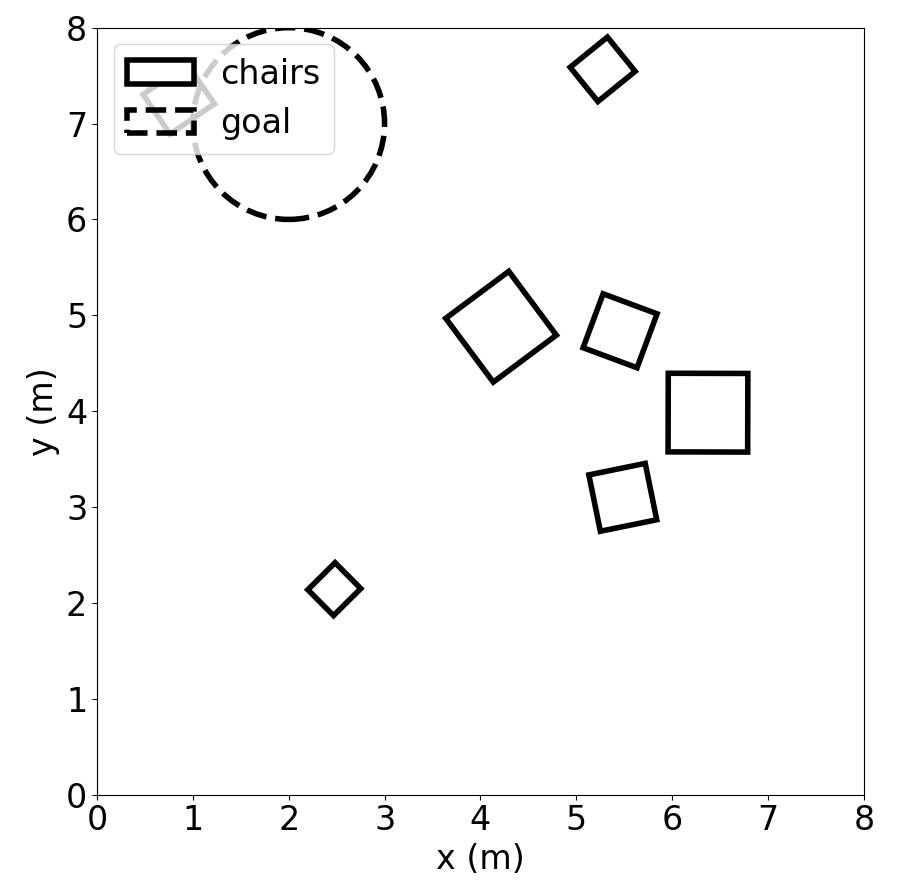}
        \caption*{(28) Environment 28}
    \end{minipage}
    \hfill
    \begin{minipage}{0.28\linewidth}
        \centering
        \includegraphics[width=\linewidth]{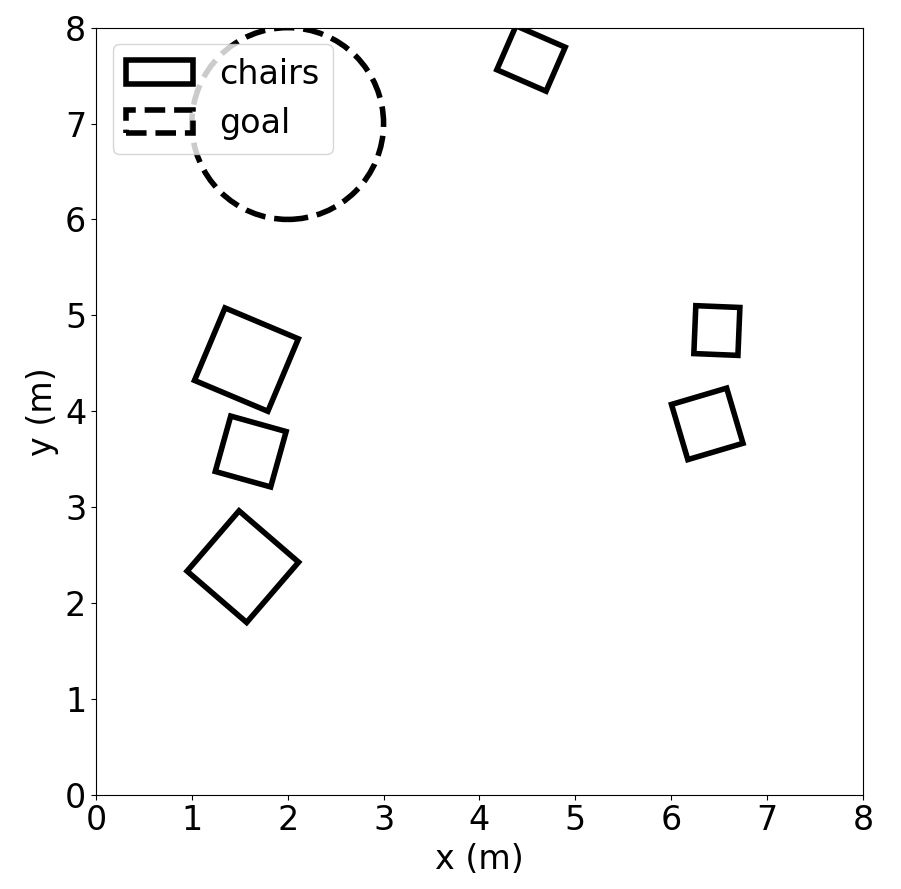}
        \caption*{(29) Environment 29}
    \end{minipage}
    \hfill
    \begin{minipage}{0.28\linewidth}
        \centering
        \includegraphics[width=\linewidth]{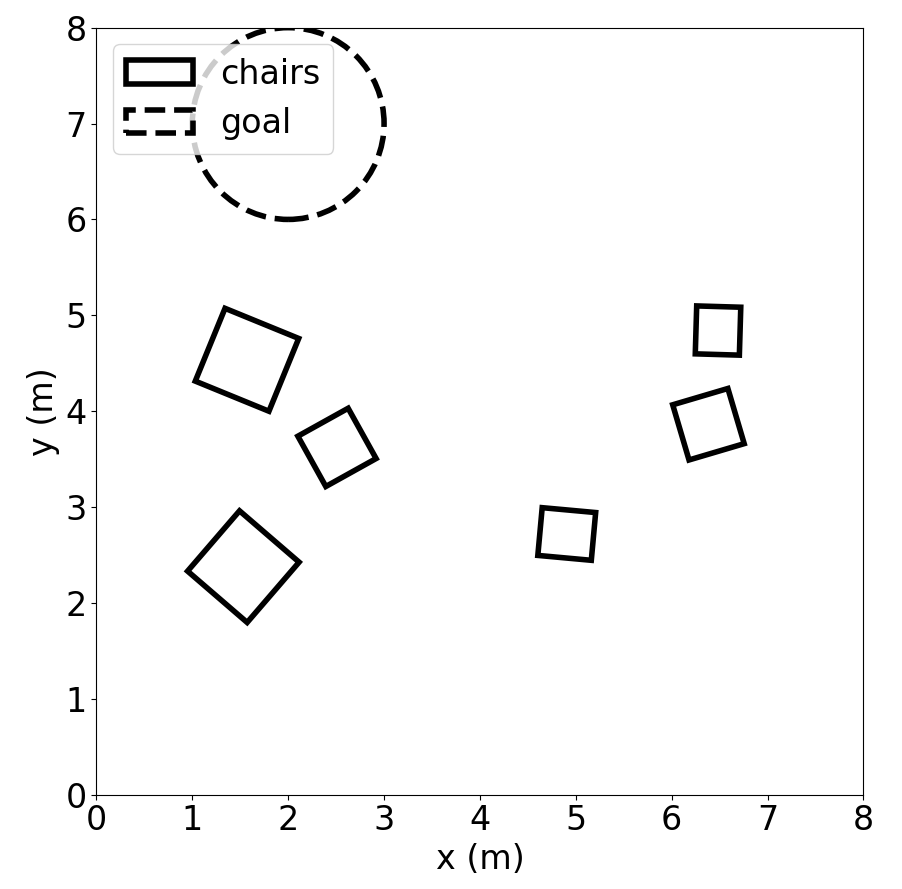}
        \caption*{(30) Environment 30}
    \end{minipage}
\end{minipage}
\end{figure*}

\end{appendices}

\end{document}